\documentclass[11pt]{article}

\usepackage[utf8]{inputenc}
\usepackage{amsmath,amsthm, amssymb, amsfonts}
\usepackage{todonotes}
\usepackage{tikz}
\usetikzlibrary{arrows}
\usepackage{wrapfig}
\usepackage{paralist}
\usepackage{verbatim}
\usepackage{fullpage}
\usepackage{relsize}
\usepackage{hyperref}
\usepackage{graphicx}
\usepackage{epstopdf}
\usepackage{wrapfig}
\usepackage{sistyle} % or siunitx

\newcommand{\Z}{\mathbb{Z}}
\newcommand{\N}{\mathbb{N}}

\newcommand{\R}{\mathbb{R}}

\newcommand{\Learn}{\mathrm{\mathbf{Learn}}}

\newcommand{\cC}{\mathcal{C}}

\newcommand{\cNN}{\mathcal{N}\hspace*{-0.035cm}\mathcal{N}}

\newcommand{\supp}{\mathrm{supp}}

\newcommand*{\pp}[1]{{\color{blue}#1}}

\newtheorem{theorem}{Theorem}[section]
\newtheorem{remark}[theorem]{Remark}
\newtheorem{definition}[theorem]{Definition}
\newtheorem{proposition}[theorem]{Proposition}
\newtheorem{lemma}[theorem]{Lemma}

\DeclareMathOperator{\suppp}{supp \,}

\title{Optimal Approximation with Sparsely Connected Deep Neural Networks}

\author{Helmut Bölcskei\footnotemark[1] \and  Philipp Grohs\footnotemark[2] \and  Gitta Kutyniok\footnotemark[3] \and Philipp Petersen\footnotemark[3]}

\begin{document}

\maketitle

\begin{abstract}
We derive fundamental lower bounds on the connectivity and the memory requirements of deep neural networks guaranteeing uniform approximation rates for arbitrary function classes in $L^2(\R^d)$. In other words, we establish a connection between the complexity of a function class and the complexity of deep neural networks approximating functions from this class to within a prescribed accuracy. Additionally, we prove that our lower bounds are achievable for a broad family of function classes. Specifically, all function classes that are optimally approximated by a general class of representation systems---so-called \emph{affine systems}---can be approximated by deep neural networks with minimal connectivity and memory requirements. Affine systems encompass a wealth of representation systems from applied harmonic analysis such as wavelets, ridgelets, curvelets, shearlets, $\alpha$-shearlets, and more generally $\alpha$-molecules. Our central result elucidates a remarkable universality property of neural networks 
and shows that they achieve the optimum approximation properties of all affine systems combined. As a specific example, we consider the class of $\alpha^{-1}$-cartoon-like functions, which is approximated optimally by $\alpha$-shearlets. We also explain how our results can be extended to the case of functions on low-dimensional immersed manifolds. Finally, we present numerical experiments demonstrating that the standard stochastic gradient descent algorithm generates deep neural networks providing close-to-optimal approximation rates.
Moreover, these results indicate that stochastic gradient descent can actually learn approximations that are sparse in the representation systems optimally sparsifying the function class the network is trained on.
\end{abstract}

\noindent {\bf Keywords.} Neural networks, function approximation, optimal sparse approximation, sparse
connectivity, wavelets, shearlets

\noindent {\bf AMS subject classification.} 41A25, 82C32, 42C40, 42C15, 41A46, 68T05, 94A34, 94A12

\renewcommand{\thefootnote}{\fnsymbol{footnote}}

\footnotetext[1]{Department of Information Technology and Electrical Engineering, ETH Zürich, 8092 Zürich,
Switzerland. \texttt{Email-Address: boelcskei@nari.ee.ethz.ch}}

\footnotetext[2]{Faculty of Mathematics, University of Vienna, 1090 Vienna, Austria, and Research Platform DataScience@UniVienna, University of Vienna, 1090 Vienna, Austria. \texttt{Email-Address: philipp.grohs@univie.ac.at}}

\footnotetext[3]{Institut f\"ur Mathematik, Technische Universit\"at Berlin, 10623 Berlin, Germany. \texttt{Email-Addresses: $\{$kutyniok,petersen$\}$@math.tu-berlin.de}}

%------------------------------------------------------------------------------------------------------------------------------
\section{Introduction}
%------------------------------------------------------------------------------------------------------------------------------

Neural networks arose from the seminal work by McCulloch and Pitts \cite{MP43} in 1943 which, inspired by the functionality of the human brain, introduced an algorithmic approach to learning with the aim of building a theory of artificial
intelligence. Roughly speaking, a neural network consists of neurons arranged in layers and connected by weighted edges; in
mathematical terms this boils down to a concatenation of affine linear functions and relatively simple non-linearities.

Despite significant theoretical progress in the 1990s \cite{Cybenko1989, Hornik1991251}, the area has seen practical progress only during the past decade, triggered by the drastic improvements in computing power and the availability of vast amounts of training data. Deep neural networks, i.e., networks with large numbers of layers, are now state-of-the-art technology for a wide variety of applications, such as
image classification \cite{Krizhevsky2012Imagenet}, speech recognition \cite{Hint2012acoustic}, or game intelligence
\cite{David2016Go}. For an in-depth overview, we refer to the survey paper by LeCun, Bengio, and Hinton
\cite{LeCun2015DeepLearning} and the recent book \cite{Goodfellow-et-al-2016}.

A neural network effectively implements a non-linear mapping and can be used to either perform classification directly or to extract features that are then fed into a classifier, such as a support vector machine \cite{steinwart2008support}. In the former case, the primary goal  is to approximate an unknown classification function based on a given set of input-output value pairs.
This is typically accomplished by learning the network's weights through, e.g., the stochastic gradient descent (via backpropagation) algorithm \cite{Rumelhart1988Backpropagation}. In a classification
task with, say, two classes, the function to be learned would take only two values, whereas in the
case of, e.g., the prediction of the temperature in a certain environment, it would be real-valued.
It is therefore clear that characterizing to what extent (deep) neural networks are capable of approximating general functions is a question of significant practical relevance.

Neural networks employed in practice often consist of hundreds of layers and may depend on billions of parameters, see for example the work \cite{he2016deep} on image classification. Training and operation of networks of this scale entail formidable computational challenges.
As a case in point, we mention speech recognition on a smartphone such as, e.g., Apple's SIRI-system, which operates in the cloud. Android's speech recognition system has meanwhile released an offline version based on a neural network with sparse connectivity.
%meaning that the number of edges with nonzero weights is small.
The desire to reduce the complexity of network training and operation naturally leads to the question of the fundamental limits on function approximation through neural networks with sparse connectivity. In addition, the network's memory requirements in terms of the number of bits needed to store its topology and weights are of concern in practice.

The purpose of this paper is to understand the connectivity and memory requirements of (deep) neural networks induced by demands on their
approximation-theoretic properties.
%properties of deep neural networks under connectivity and memory constraints.
Specifically, defining the complexity of a function class $\mathcal{C}$ as the rate of growth of the minimum number of bits needed to describe any element in $\mathcal{C}$ to within a maximum allowed error approaching zero, we shall be interested in the following question:
%connectivity and memory requirements on networks guaranteeing the same approximation accuracy.
%This will be accomplished by
%such that any element in $\mathcal{C}$ can be encoded and decoding up to a prescribed accuracy is possible. Now we can ask: 
Depending on the complexity of $\mathcal{C}$, what are the connectivity and memory requirements 
of a deep neural network approximating every element in $\mathcal{C}$ to within an error of $\varepsilon$? 
%prescribed accuracy %by a network satisfying these requirements?
We address this question by interpreting the network as an encoder in Donoho's min-max rate distortion theory \cite{DONOHO1993100}
%and establishing associated fundamental lower bounds on connectivity and memory requirements.
%for a decoder resulting in an approximation error of no more than $\varepsilon$ to exist.
% for the network to guarantee the same approximation accuracy. 
% uniform approximation rates for a given function class $\mathcal{C}$. 
%In other words, for given accuracy $\varepsilon>0$ we establish the required complexity such that for all $f\in \cC$ there exists a network of the given complexity approximating $f$ up to %an error of $\varepsilon$. 
and establishing rate-distortion optimality for 
%Moreover, we demonstrate that these bounds are saturated by 
a broad family of function classes $\mathcal{C}$, namely those classes for which so-called affine systems---a general class of representation systems---yield optimal approximation rates in the sense of non-linear approximation theory \cite{DeVore1998nonlinear}.  Affine systems encompass a wealth of representation systems from applied harmonic analysis such as wavelets \cite{Dau92}, ridgelets \cite{CandesDiss}, curvelets \cite{CD02}, shearlets \cite{GKL06}, $\alpha$-shearlets and more generally $\alpha$-molecules \cite{GroKKS2016alphaMolecules}. 
Our result therefore uncovers an interesting universality property of deep neural networks; they exhibit the optimal approximation properties of all affine systems combined. The technique we develop to prove our main statements is interesting in its own right as it constitutes a more general framework for transferring results on function
approximation through representation systems to results on approximation by deep neural networks.

%------------------------------------------------------------------------------------------------------------------------------
\subsection{Deep Neural Networks}
%------------------------------------------------------------------------------------------------------------------------------

While various network architectures exist in the literature, we focus on the following setup.

\begin{definition}\label{def:NN}
Let $L, d, N_1, \ldots, N_{L}\in \N$ with $L\,\ge\,2$. A map $\Phi: \R^d \to \R^{N_L}$ given by
\begin{equation}\label{eq:NNdef}
\Phi(x) = W_L\rho \, ( W_{L-1} \rho\, ( \dots \rho \, ( W_{1}(x)))), \quad \text{ for }x\in \R^d,
\end{equation}
with affine linear maps $W_{\ell}: \R^{N_{\ell-1}} \to \R^{N_\ell}$, $1 \leq \ell \leq L$, and the non-linear {\em activation function} $\rho$
%---often referred to as {\em activation function}---
acting component-wise, is called a \emph{neural network}. 
%This network is composed of affine linear maps $W_{\ell}: \R^{N_{\ell-1}} \to \R^{N_\ell}$, $1 \leq \ell \leq L$, and the non-linear function $\rho$---often referred to as {\em activation %function}---acting component-wise.
Here, $N_0 := d$ is the \emph{dimension of the $0$-th layer referred to as the input layer}, $L$ denotes the \emph{number of layers} (not counting the input layer), $N_1, \ldots, N_{L-1}$ stands for the \emph{dimensions of the $L-1$ hidden layers}, and $N_L$ is the \emph{dimension of the output layer}.  The affine linear map $W_{\ell}$ is defined via $W_{\ell}(x)=A_\ell x + b_\ell$ with $A_{\ell}\in \mathbb{R}^{N_{\ell}\times N_{\ell-1}}$ and the affine part $b_\ell\in \mathbb{R}^{N_\ell}$. $(A_\ell)_{i,j}$ represents the \emph{weight associated with the edge between
the $j$-th node in the $(\ell-1)$-th layer and the $i$-th node in the $\ell$-th layer}, while  $(b_\ell)_i$ is \emph{the weight associated with the $i$-th node in the $\ell$-th layer}. These assignments are schematized in Figure \ref{fig:Weights}.  The total number of nodes is given by $\mathcal{N}(\Phi) := d + \sum_{\ell=1}^L N_{\ell}$.
The real numbers $(A_\ell)_{i,j}$ and $(b_\ell)_i$ are said to be the network's edge weights and node weights, respectively, and the total number of nonzero edge weights, denoted by 
$\mathcal{M}(\Phi)$, is the network's connectivity. 
%As $L\,\ge\,2$, the connectivity necessarily satisfies $M\,\ge\,2$.
\end{definition}
The term ``network'' stems from the interpretation of the mapping $\Phi$ as a weighted acyclic directed graph with nodes arranged in $L+1$ hierarchical layers and
edges only between adjacent layers. If the network's connectivity $\mathcal{M}(\Phi)$ is small relative to the number of connections  possible (i.e., the number of edges in the graph that is fully connected between adjacent layers), we say that the network is {\it sparsely connected}\/.

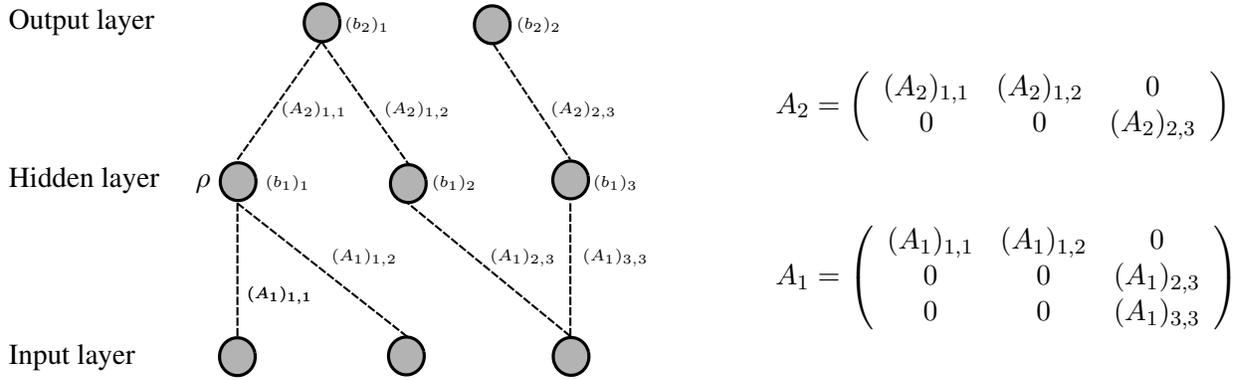
\begin{figure}[htb]
\flushleft
\hspace{3cm}
  \includegraphics[width = 0.3\textwidth]{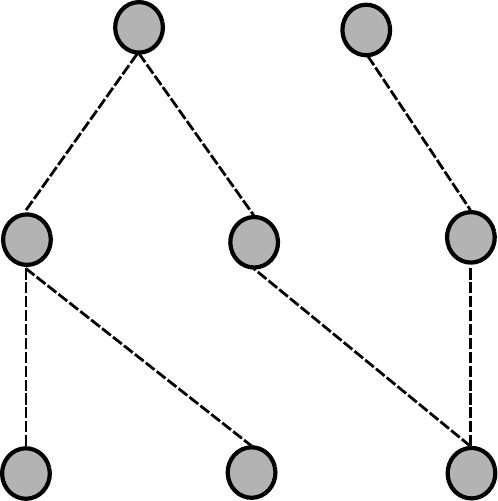}
 \tiny
 \put(-93,132){$(b_2)_1$}
 \put(-28,132){$(b_2)_2$}
 \put(-123,72){$(b_1)_1$}
 \put(-60,72){$(b_1)_2$}
 \put(1,72){$(b_1)_3$}
 \put(-117,100){$(A_2)_{1,1}$}
 \put(-78,100){$(A_2)_{1,2}$}
 \put(-15,100){$(A_2)_{2,3}$}
 \put(-130,30){$(A_1)_{1,1}$}
 \put(-3,44){$(A_1)_{3,3}$}
 \put(-38,44){$(A_1)_{2,3}$}
 \put(-98,44){$(A_1)_{1,2}$}
 \put(-130,30){$(A_1)_{1,1}$}
 \normalsize
  \put(70, 100){$A_2 = \left(\begin{array}{c c c}
                     (A_2)_{1,1} & (A_2)_{1,2} & 0\\
                     0 & 0 & (A_2)_{2,3}\\
                    \end{array}\right)
$}
 \put(70, 35){$A_1 = \left(\begin{array}{c c c}
                     (A_1)_{1,1} & (A_1)_{1,2} & 0\\
                     0 & 0 & (A_1)_{2,3}\\
                     0 & 0 & (A_1)_{3,3}\\
                    \end{array}\right)
$}
\put(-220, 132){Output layer}
\put(-220, 72){Hidden layer \quad $\rho$}
\put(-220, 5){Input layer}

 \caption{Assignment of the weights $(A_\ell)_{i,j}$ and $(b_\ell)_{i}$ of a two-layer network to the edges and nodes, respectively. }
 \label{fig:Weights}
\end{figure}

Throughout the paper, we consider the case $\Phi: \R^d \to \R$, i.e., $N_L = 1$, which includes situations
such as the classification and temperature prediction problem described above. We emphasize, however, that the general results of Sections \ref{sec:lowerbound}, \ref{sec:bestapprox}, and \ref{sec:optimalapprox} are readily generalized to $N_L >1$.

We denote the class of networks $\Phi: \R^d \to \R$ with exactly $L$ layers, connectivity no more than $M$, and activation function $\rho$ by $\cNN_{L, M, d, \rho}$ with the understanding that for $L=1$, the set $\cNN_{L, M, d, \rho}$ is empty.
Moreover, we let
\begin{align*}
 \cNN_{\infty, M, d, \rho} := \bigcup_{L\in \N} \cNN_{L, M, d, \rho}, \quad \cNN_{L, \infty, d, \rho} := \bigcup_{M\in \N} \cNN_{L, M, d, \rho}, \quad  \cNN_{\infty, \infty, d, \rho} := \bigcup_{L\in \N} \cNN_{L, \infty, d, \rho}.
\end{align*}

Now, given a function $f: \R^d \to \R$, we are interested in the theoretically best possible approximation of $f$ by a network $\Phi \in \cNN_{\infty, M, d, \rho}$.
Specifically, we will want to know how the approximation quality depends on the connectivity $M$ and what the associated
%what the associated minimum connectivity $M$ and minimum 
number of bits needed to store the network topology and the corresponding quantized weights is. Clearly, smaller $M$ entails lower computational complexity in terms of evaluating (\ref{eq:NNdef}) and a smaller number of bits translates to reduced memory requirements for storing the network. Such a result benchmarks all conceivable algorithms for learning the network topology and weights.
%employed to learn the network weights}.
%On the quantitative level, we shall be interested in the dependence of this lower bound on the connectivity $M$ and on the number of bits available to encode the network topology %and the quantized weights.

%------------------------------------------------------------------------------------------------------------------------------
\subsection{Quantifying Approximation Quality}\label{subsec:NtermApprox}
%------------------------------------------------------------------------------------------------------------------------------

We proceed to formalizing our problem statement and start with a brief review of a widely used framework in approximation theory \cite{DL93,DeVore1998nonlinear}.
%for the characterization of the approximation quality of functions under restricting conditions on the approximant.

Fix $\Omega\subset \mathbb{R}^d$. Let  $\cC$ be a compact set of functions in $L^2(\Omega)$, henceforth referred to as function class, and consider a corresponding 
system $\mathcal{D}:=(\varphi_i)_{i \in I} \subset L^2(\Omega)$ with $I$ countable, termed \emph{representation system}. We study the {\em error of best $M$-term approximation} of $f \in \cC$ in $\mathcal{D}$:

\begin{definition}\cite{DL93} \label{def:optimalApproximationRate}
Given $d\in \N$, $\Omega \subset \R^d$, a function class $\cC \subset L^2(\Omega)$, and a representation system $\mathcal{D} = (\varphi_i)_{i \in I} \subset L^2(\Omega)$, we define, for $f \in \cC$ and $M\in \N$,
\begin{align} \label{eq:GammaMDictDef}
\Gamma_M^\mathcal{D}(f) := \inf_{\substack{I_M \subseteq I,\\ \#I_M = M, (c_i)_{i \in I_M}}} \left\|f - \sum_{i \in I_M} c_i \varphi_i\right\|_{L^2(\Omega)}.
\end{align}
We call $\Gamma_M^\mathcal{D}(f)$ the {\em best $M$-term approximation error of $f$ in $\mathcal{D}$}.
Every $f_M = \sum_{i \in I_M} c_i \varphi_i$ attaining the infimum in \eqref{eq:GammaMDictDef} is referred to as a {\em best $M$-term approximation} of $f$ in $\mathcal{D}$.
The supremal $\gamma > 0$ such that 
%there exists $C>0$ with
\[
\sup_{f \in \cC}\Gamma_M^\mathcal{D}(f) \in \mathcal{O}(M^{-\gamma}), \,\, M \rightarrow \infty,
%\qquad \mbox{ for all } M \in \N,
\]
will be denoted by $\gamma^\ast(\mathcal{C},\mathcal{D})$. We say that the  {\em best $M$-term approximation rate of $\cC$ in the representation system $\mathcal{D}$} is $\gamma^\ast(\mathcal{C},\mathcal{D})$.
\end{definition}

Function classes $\mathcal{C}$ widely studied in the approximation theory literature include unit balls in Lebesgue, Sobolev, or Besov spaces \cite{DeVore1998nonlinear}, as well as $\alpha$-cartoon-like functions \cite{GroKKS2016alphaMolecules}. A wealth of structured representation systems $\mathcal{D}$ is provided by the area of applied harmonic analysis, starting with wavelets \cite{Dau92}, followed by ridgelets \cite{CandesDiss}, curvelets \cite{CD02}, shearlets \cite{GKL06}, parabolic molecules \cite{GK14},
and most generally $\alpha$-molecules \cite{GroKKS2016alphaMolecules}, which include all previously named systems as special cases. Further examples are Gabor frames \cite{grochenig2013foundations} and wave atoms \cite{demanet2007wave}. 

%------------------------------------------------------------------------------------------------------------------------------
\subsection{Approximation by Deep Neural Networks}\label{subsec:NNapproxintro}
%------------------------------------------------------------------------------------------------------------------------------

The main conceptual contribution of this paper is the development of an approximation-theoretic framework for deep neural networks in the spirit of \cite{DL93}.
Specifically, we shall substitute the concept of best $M$-term approximation with representation systems by best $M$-edge approximation through neural networks.
In other words, parsimony in terms of the number of participating elements of a representation system is replaced by parsimony in terms of connectivity. More formally, we consider the following setup.

\begin{definition}\label{def:optimalApproximationRateNN}
Given $d\in \N$, $\Omega \subset \R^d$, a function class $\cC \subset L^2(\Omega)$, and an activation function $\rho: \R \to \R$, we define, for $f \in \cC$ and $M\in \N$,
\begin{align} \label{eq:GammaMDef}
\Gamma_M^{\cNN}(f) := \inf_{\Phi \in \cNN_{\infty, M, d, \rho}} \|f - \Phi \|_{L^2(\Omega)}.
\end{align}
We call $\Gamma_{M}^{\cNN}(f)$ the \emph{best $M$-edge approximation error of $f$}.
The supremal $\gamma > 0$ such that 
%there exists $C>0$ with
\[
\sup_{f \in \cC} \Gamma_{M}^{\cNN}(f) \in \mathcal{O}(M^{-\gamma}), \,\, M \rightarrow \infty,
%\qquad \mbox{for all } M \in \N,
\]
 will be denoted by $\gamma_{\cNN}^\ast(\cC, \rho)$. We say that the {\em best $M$-edge approximation rate of $\cC$ by neural networks with activation function $\rho$} is 
 $\gamma_{\cNN}^\ast(\cC, \rho)$.
\end{definition}

We emphasize that the infimum in \eqref{eq:GammaMDef} is taken over all networks with fixed activation function $\rho$, fixed input dimension $d$, no more than $M$ edges of nonzero weight, and arbitrary number of layers $L$. In particular, this means that the infimum is taken over all possible network topologies.
The resulting best $M$-edge approximation rate is fundamental as it benchmarks all learning algorithms, i.e., all algorithms that  
map an input function $f$ and an $\varepsilon>0$ to a neural network that approximates $f$ with error no more than $\varepsilon$.
Our framework hence provides a means for assessing the performance of a given learning algorithm in the sense of allowing to measure how close the $M$-edge approximation rate induced by the algorithm is to the best $M$-edge approximation rate $\gamma_{\cNN}^\ast(\cC, \rho)$.

%------------------------------------------------------------------------------------------------------------------------------
\subsection{Previous Work}
%------------------------------------------------------------------------------------------------------------------------------

The best-known results on approximation by neural networks are the universal approximation theorems of Hornik \cite{Hornik1991251} and Cybenko \cite{Cybenko1989}, stating that every measurable function $f$ can be approximated arbitrarily well by a single-hidden-layer ($L=2$ in our terminology) neural network. The literature on approximation-theoretic properties of networks with a single hidden layer continuing this line of work is abundant. Without any claim to completeness, we mention work on approximation error bounds in terms of the number of neurons for functions with bounded first moments
\cite{Barron1993}, \cite{Barron1994}, the non-existence of localized approximations
\cite{ChuXM1994networksforlocApprox}, a fundamental lower bound on approximation rates \cite{DeVore1997approxfeedforward, CandesDiss}, and the approximation of smooth or analytic functions \cite{Mhaskar1996NNapprox,Mhaskar1995151}.

Approximation-theoretic results for networks with multiple hidden layers were obtained in \cite{Hornik1989universalApprox, Mhaskar1993} for general functions, in \cite{Funahashi1989183} for continuous functions, and for functions together with their derivatives in \cite{NguyenThien1999687}.
In \cite{ChuXM1994networksforlocApprox} it was shown that for certain approximation tasks deep networks can perform fundamentally better than single-hidden-layer networks.
We also highlight two recent papers, which investigate the benefit---from
an approximation-theoretic perspective---of multiple hidden layers.
Specifically, in \cite{Eldan2016PowerofDepth} it was shown that there exists a function which, although expressible through a small three-layer network,
can only be represented through a very large two-layer network; here size is measured in terms of the total number of neurons in the network. In the setting
of deep convolutional neural networks, first results of a nature similar to those in \cite{Eldan2016PowerofDepth} were reported in \cite{Mhaskar2016DeepVSShallow}. Additionally, by linking the expressivity properties of neural networks to tensor decompositions, \cite{cohen2016expressive, cohen2016convolutional} established the existence of functions that can be realized by relatively small deep convolutional networks but require exponentially larger shallow networks. For survey articles on approximation-theoretic aspects of neural networks, we refer the interested reader to \cite{ellacott1994aspects,pinkus1999approximation}.

Most closely related to our work is that by Shaham, Cloninger, and Coifman \cite{ShaCC2015provableAppDNN},
which shows that for functions that are sparse in specific wavelet frames, the best $M$-edge approximation rate of three-layer neural networks is at least as high as the best $M$-term approximation rate in piecewise linear wavelet frames.

%------------------------------------------------------------------------------------------------------------------------------
\subsection{Contributions}\label{subsec:contrext}
%------------------------------------------------------------------------------------------------------------------------------

Our contributions can be grouped into four threads.
\begin{itemize}
\item {\em Fundamental lower bound on connectivity.} 
%Let $d\in \N$, $\Omega\subset \R^d$,  $\rho :\R \to \R$, and $\cC \subset L^2(\Omega)$. We establish a lower bound on the best $M$-edge approximation rate of $\cC$ by neural %networks in terms of the description complexity of $\cC$. 
%{\bf PG: We do not give an explicit proof of the following result but i think it is ok like this.}
%Denoting the \emph{optimal exponent} with respect to the minimax code length of $\cC$ as defined in \cite{DONOHO1993100,grohs2015optimally} (see Definition~\ref{def:optexp} %below) by $\gamma^*(\cC)$, we demonstrate in Theorem~\ref{thm:optimality} and Corollary \ref{thm:EffRepNN} that under minor regularity assumptions on $\rho$, for all $\gamma > %\gamma^*(\cC)$, there exists a $C_\gamma>0$ such that
%\begin{align} \label{eq:fundbound}
% \sup_{f \in \cC}  \ \inf_{\Phi \in \widetilde{\cNN}_{\infty, M, d, \rho}} \|f - \Phi\|_{L^2(\Omega)} \geq C_\gamma M^{-\gamma}, \text{ for all }M \in \N,
%\end{align}
%where $\widetilde{\cNN}_{\infty, M, d, \rho}$ denotes all networks $\Phi \in \cNN_{\infty, M, d, \rho}$ whose weights can be encoded $O(\log_2(M))$ bits, each. {\bf PG: this statement %is confusing. every single number can be encoded with 1 bit.}
We quantify the minimum network connectivity needed to allow approximation of \emph{all} elements of a given function class $\cC$ to within a maximum allowed error. On a conceptual level, this result establishes a universal link between the complexity of a given function class and the connectivity required by corresponding approximating
neural networks.

\vspace*{1mm}
\item {\em Transfer from $M$-term to $M$-edge approximation.}
We develop a general framework for transferring best $M$-term approximation results in representation systems to best $M$-edge approximation results for neural networks.
%These transfer results hold for representation systems $\mathcal{D}$ that are \emph{representable by  neural networks} in the following sense: There exist $L, R\in \mathbb{N}$ such %that for every element $\varphi_i\in \mathcal{D}$ of the representation system $\mathcal{D}$ and every $\varepsilon>0$, there is a neural network $\Phi_{i,\varepsilon}$ with $L$ layers and %connectivity at most $R$ such that $\|\varphi_i-\Phi_i\|\le\varepsilon$. 
%If a representation system $\mathcal{D}$ is representable by neural networks, we will demonstrate that for all $\gamma < \gamma^*(\cC, \mathcal{D})$  there exists a 
%$C_\gamma>0$ such that
%\begin{align} \label{eq:upperboundNEW}
% \sup_{f \in \cC}  \ \inf_{\Phi \in {\cNN}_{\infty, M, d, \rho}} \|f - \Phi\|_{L^2(\Omega)} \leq C_\gamma M^{-\gamma}, \text{ for all }M \in \N.
%\end{align}
%Consequently, we have that for all representation systems $\mathcal{D}$ and function classes $\cC$ such that $\gamma^*(\cC, \mathcal{D}) = \gamma^*(\cC)$ the upper bound %\eqref{eq:upperboundNEW} matches the lower bound of \eqref{eq:fundbound}, however, by dropping the assumption of encodeable weights.

\vspace*{1mm}
\item {\em Memory requirements}.
We characterize the memory requirements needed to store the topology and the quantized weights of optimally-approximating neural networks.

\vspace*{1mm}
\item {\em Realizability of optimal approximation rates.} An important practical question is how neural networks trained by stochastic gradient descent (via backpropagation) \cite{Rumelhart1988Backpropagation} perform relative to the fundamental bounds established in the paper. 
Interestingly, our numerical experiments indicate that stochastic gradient descent can achieve 
$M$-edge approximation rates quite close to the fundamental limit.
%Interestingly, our numerical experiments indicate that, given a fixed network topology with sparse connectivity motivated by the constructions of affine systems, the stochastic gradient %descent algorithm yields close-to-optimal approximation rates. Moreover, training a network to approximate $\alpha^{-1}$-cartoon-like functions, we observe that stochastic gradient %descent generates neural networks which mimic the classical best $M$-term approximation of such functions in a representation system of $\alpha$-molecules.
\end{itemize}

%------------------------------------------------------------------------------------------------------------------------------
\subsection{Outline of the Paper}
%------------------------------------------------------------------------------------------------------------------------------
Section \ref{sec:effectiveApprox} introduces the novel concept of effective best $M$-edge approximation.
The fundamental lower bound on connectivity is developed in Section \ref{sec:lowerbound}. Section
\ref{sec:bestapprox} describes a general framework for transferring best $M$-term approximation results in representation systems to best $M$-edge approximation results for neural networks. In Section \ref{sec:optimalapprox}, we apply this transfer framework to the broad class of affine representation systems, and
Section \ref{sec:alphacart} shows that this leads to optimal $M$-edge approximation rates for cartoon functions. In Section \ref{sec:manifold}, we briefly outline 
the extension of our main findings to the approximation of functions defined on manifolds. Finally, numerical results assessing the performance of stochastic gradient descent (via backpropagation) relative to our lower bound on connectivity are reported in Section \ref{sec:numerics}.

%------------------------------------------------------------------------------------------------------------------------------
\section{Effective Best $M$-term and Best $M$-edge Approximation}\label{sec:effectiveApprox}
%------------------------------------------------------------------------------------------------------------------------------

We proceed by introducing $M$-term approximation via dictionaries and $M$-edge approximation via neural networks. These concepts do, however, not allow for a meaningful notion of optimality in practice. A remedy is provided by effective best $M$-term approximation according to \cite{DONOHO1993100,grohs2015optimally} and the new concept of effective best $M$-edge approximation introduced below.

\subsection{Effective Best $M$-term Approximation}

The best $M$-term approximation rate $\gamma^\ast(\mathcal{C},\mathcal{D})$ according to Definition~\ref{def:optimalApproximationRate} measures the hardness of approximation of a given function class $\mathcal{C}$ by a fixed representation system $\mathcal{D}$. It is sensible to ask
%{\bf PG: I deleted one paragraph which i found confusing}
whether for a given function class $\cC$, there is a fundamental limit on $\gamma^\ast(\mathcal{C},\mathcal{D})$ when one is allowed to vary over $\mathcal{D}$.
%its best $M$-term approximability. Specifically, this limit would be obtained by identifying the largest 
%$\gamma^\ast(\mathcal{C},\mathcal{D})$.
%In this regard, it is conceivable that the optimal approximation rate for $\cC$ in any representation system reflects specific properties of $\cC$. Finally, if an optimal rate exists, then %one can assess the suitability of given representation systems for approximating $\cC$ by comparing the provided approximation rate with the optimal rate.
As shown in \cite{DONOHO1993100,grohs2015optimally}, every dense (and countable) $\mathcal{D} \subset L^2(\Omega)$, $\Omega \subset \R^d$, results in $\gamma^\ast(\mathcal{C},\mathcal{D}) = \infty$ for all function classes $\mathcal{C} \subset L^2(\Omega)$.
%, simply because every $f\in L^2(\mathbb{R}^d)$ can be represented exactly as a $1$-term approximation in $\mathcal{C}$. 
However, identifying the elements in $\mathcal{D}$ participating in the best $M$-term approximation is infeasible as it entails searching through the infinite set $\mathcal{D}$ and requires, in general, an infinite number of bits to describe the indices of the participating elements.
%In summary, we can conclude that
This insight leads to the concept of ``best $M$-term approximation subject to polynomial-depth search''  as introduced by Donoho in \cite{DONOHO1993100}.
%and further developed by Grohs in \cite{grohs2015optimally}. 

\begin{definition}\label{def:polydepth}

Given $d\in \N$, $\Omega \subset \R^d$, a function class $\cC \subset L^2(\Omega)$, and a representation system $\mathcal{D} = (\varphi_i)_{i \in I} \subset L^2(\Omega)$, 
the supremal $\gamma > 0$ so that there exist a polynomial $\pi$ and a constant $D>0$ such that
%
%we define, for $f \in \cC$ and $M\in \N$,
\begin{align} \label{eq:GammaMDictDefeff}
%\Gamma_M^\mathcal{D,\text{eff}}(f) := 
\sup_{f \in \mathcal{C}} \inf_{\substack{I_M\subset \{1,\dots, \pi(M)\}, \\ \#I_M = M, \, (c_i)_{i \in I_M}, \, \max_{i\in I_M} \! |c_i | \, \le \, D}} \left\|f - \sum_{i \in I_M} c_i \varphi_i\right\|_{L^2(\Omega)} \in \mathcal{O}(M^{-\gamma}), \,\, M \rightarrow \infty,
\end{align}
will be denoted by $\gamma^{\ast,\text{eff}}(\mathcal{C},\mathcal{D})$ and
referred to as 
%We call $\Gamma_M^\mathcal{D,\text{eff}}(f)$ the {\em best effective $M$-term approximation error of $f$ in $\mathcal{D}$}.
%Every $f_M = \sum_{i \in I_M} c_i \varphi_i$ attaining the infimum in \eqref{eq:GammaMDictDef} is referred to as a {\em best $M$-term approximation} of $f$ in $\mathcal{D}$.
%The supremal $\gamma > 0$ such that 
%there exists $C>0$ with
%\[
%\sup_{f \in \cC}\Gamma_M^\mathcal{D,\text{eff}}(f) =\mathcal{O}(M^{-\gamma}), 
%\qquad \mbox{ for all } M \in \N,
%\]
%will be referred to as 
{\em effective best $M$-term approximation rate of $\cC$ in the representation system $\mathcal{D}$}.
% and denoted by
%$\gamma^{\ast,\text{eff}}(\mathcal{C},\mathcal{D})$. 
%We say that the {\em best effective $M$-term approximation rate of $\cC$ in the representation system $\mathcal{D}$} is $\gamma^{\ast,\text{eff}}(\mathcal{C},\mathcal{D})$.
\end{definition}

%    Let $d\in \N$, $\Omega\subset \mathbb{R}^d$. Consider the function class $\mathcal{C}\subset L^2(\Omega)$ and the representation system $\mathcal{D}=(\varphi_i)_{i\in 
%\mathbb{N}}\subset L^2(\Omega)$. The supremal $\gamma > 0$ so that there exist
%a constant $D>0$, index sets $I_M\subset \{1,\dots, \pi(M)\}$  with $\# I_M=M$, and coefficients $(c_i)_{i\in I_M}$
%satisfying $\max_{i\in I_M}|c_i|\le D$, such that
%    
 %    a univariate polynomial $\pi$ and constants $C,D>0$ such that for all $f\in \mathcal{C}$,  
   % \begin{equation} \label{eq:polydepth}
    %\left\|f    -\sum_{i\in I_M}c_i\varphi_i\right\|_{L^2(\Omega)}  = \mathcal{O}(M^{-\gamma})
    %, \qquad \mbox{ for all } M \in \N,
    % \end{equation}
%    for some index set $I_M\subset \{1,\dots, \pi(M)\}$  with $\# I_M=M$ and coefficients $(c_i)_{i\in I_M}$
%    satisfying $\max_{i\in I_M}|c_i|\le D$.
   % will be referred to as ${\gamma^{\ast,\text{eff}}(\cC,\mathcal{D})}$. We say that the {\em effective best $M$-term approximation rate of}\/ $\mathcal{C}$ in $\mathcal{D}$ is ${\gamma^{\ast,\text{eff}}(\cC,\mathcal{D})}$.
%    
%    $\gamma^{\ast,\text{eff}}(\cC,\mathcal{D})$ and the \emph{optimal effective best $M$-term approximation rate of $\cC$ in $\mathcal{D}$} is ${\gamma^{\ast,\text{eff}}(\cC,
%\mathcal{D})}$.
%\end{definition}

We will demonstrate in Section \ref{subsec:repcode} that $\sup_{\mathcal{D} \subset L^2(\Omega)}\gamma^{\ast,\text{eff}}(\cC,\mathcal{D})$ is, indeed, finite under quite general conditions on $\cC$ and, in particular, depends on the ``description complexity" of $\cC$. This will allow us to assess the approximation 
capacity of a given representation system $\mathcal{D}$ for $\cC$ by comparing $\gamma^{\ast,\text{eff}}(\cC,\mathcal{D})$ to the ultimate limit $\sup_{\mathcal{D} \subset L^2(\Omega)}\gamma^{\ast,\text{eff}}(\cC,\mathcal{D})$.

%Hence, $\sup_{\mathcal{D} \subset L^2(\R^d)}\gamma^{\ast,\text{eff}}(\cC,\mathcal{D})$ provides a measure for the complexity of $\cC$. Moreover, we can now assess the practical 

\subsection{Effective Best $M$-edge Approximation}

We next aim at establishing a relationship in the spirit of effective best $M$-term approximation for approximation through deep neural networks.
%Specifically, we shall 
%understanding if there exists a relationship between a function class $\mathcal{C}$ and $\gamma_{\cNN}^\ast(\cC, \rho)$. 
To this end, we first note that Definition~\ref{def:optimalApproximationRateNN} encounters problems similar to those identified for 
approximation by representation systems, namely the quantity $\sup_{\rho:\R \to \R} \gamma^\ast_{\cNN} (\cC, \rho)$ does not reveal anything 
tangible about the approximation complexity of $\cC$ in deep neural networks, unless further constraints are imposed 
on the approximating network. To make this point, we first review the following remarkable result:

\begin{theorem}{\cite[Theorem 4]{maiorov1999lower}}\label{pinkusthm}
	There exists a function $\rho:\R\to \R$ that is $C^\infty$, strictly increasing, and satisfies $\lim_{x\to \infty}\rho(x)=1$ and $\lim_{x\to -\infty}\rho(x)=0$, such that for any $d\in \mathbb{N}$, any $f\in C([0,1]^d)$, and any $\varepsilon >0$, there is a neural network $\Phi$ with activation function $\rho$ and three layers, of dimensions $N_1=3d, N_2=6d+3$, and $N_3 = 1$,
	satisfying
	\begin{equation}
	\label{eq:medge}
		\sup_{x\in [0,1]^d}\left|f(x)-\Phi(x)\right|\le \varepsilon. 
	\end{equation}
\end{theorem}

We observe that the number of nodes and the number of layers of the approximating network in Theorem~\ref{pinkusthm} do not depend on the approximation error $\varepsilon$. 
In particular, $\varepsilon$ can be chosen arbitrarily small while having $\mathcal{M}(\Phi)$ bounded. By density of $C([0,1]^d)$ in $L^2([0,1]^d)$ and hence in all compact subsets 
of $L^2([0,1]^d)$, this implies the existence of an activation function $\rho: \R\to \R$ such that $\gamma^\ast_{\cNN} (\cC, \rho)=\infty$ for all compact $\mathcal{C} \subset L^2([0,1]^d), d\in \N$.
However, the networks underlying Theorem~\ref{pinkusthm} necessarily lead to weights that are (in absolute value) not bounded by $|\pi(\varepsilon^{-1})|$ for a polynomial $\pi$, a requirement we will
have to impose to get rate-distortion-optimal approximation through neural networks (see Section \ref{sec:lowerbound}).
To see that the weights, indeed, do not obey a polynomial growth bound in $\varepsilon^{-1}$, we note that thanks to Theorem~\ref{pinkusthm}, there exist $C>0$ and $\gamma>0$ such that
%fix $\gamma > 0$ and set $\varepsilon=CM^{-\gamma}$ in Theorem~\ref{pinkusthm}. This yields
\begin{align}\label{eq:TheContradictionOfPinkus}
	\sup_{f \in \mathcal{C}}\, \inf_{\Phi_{M} \in \cNN_{3, M,d,\rho}} \|f    - \Phi_{M} \|_{L^2(\Omega)}\le C M^{-\gamma}, \,\, \text {for all}\,\, M \in \N.
\end{align}
%where $M_0 \le 21d^2+15d+3$ is the number of nonzero edge weights of the network in Theorem~\ref{pinkusthm}. 
Now, as $\varepsilon$ in Theorem~\ref{pinkusthm} can be made arbitrarily small while the connectivity of the corresponding networks remains upper-bounded by $21d^2+15d+3$,
(\ref{eq:TheContradictionOfPinkus}) would have to hold for arbitrarily large $\gamma$, in particular also for $\gamma > \gamma_{\cNN}^{\ast, \text{eff}}(\mathcal{C}, \rho)$, where $\gamma_{\cNN}^{\ast, \text{eff}}(\mathcal{C}, \rho)$ is the effective best $M$-edge approximation rate according to Definition~\ref{def:repoptiNN}. By Theorem~\ref{thm:EffRepNN} below, however, $\gamma_{\cNN}^{\ast, \text{eff}}(\mathcal{C}, \rho) \leq {\gamma^\ast(\mathcal{C})}$, where ${\gamma^\ast(\mathcal{C})}$ is the optimal exponent according to Definition~\ref{def:optexp}.
Owing to Definition~\ref{def:repoptiNN}, we can therefore conclude that the weights of the network achieving the infimum in (\ref{eq:TheContradictionOfPinkus})
can not be bounded by a polynomial in $M \sim \varepsilon^{-1}$ whenever $\gamma^\ast(\mathcal{C})<\infty$. Here and in the sequel, we write $a \sim b$ if the variables $a$ and $b$ are proportional, i.e., there exist uniform constants $c_1,c_2 >0$ such that $c_1 a \leq b \leq c_2 a$.

The observation just made resembles the problem in best $M$-term approximation which eventually led to the concept of {\em effective} best $M$-term approximation,
where we restricted the search depth in the representation system $\mathcal{D}$ to be bounded by a given polynomial in $M$ and
the coefficients $c_i$ to be bounded according to $\max_{i\in I_M} \! |c_i | \, \le \, D$. 
Interpreting the weights in the network as the counterpart of the coefficients $c_i$ in best $M$-term approximation, we see that the restriction on the search depth corresponds to restricting the size of the indices enumerating the participating weights.
The need
for such a restriction is obviated by the tree structure of deep neural networks as exposed in detail in the proof of Proposition~\ref{prop:optimalitynoquant}. The second restriction will lead us to a growth condition on the weights, which is more generous than the corresponding requirement of the $c_i$ in effective best $M$-term approximation being bounded.

In summary, this leads to the novel concept of ``best $M$-edge approximation subject to polynomial weight growth'' as formalized next.

\begin{definition}\label{def:repoptiNN}

Given $d\in \N$, $\Omega \subset \R^d$, a function class $\cC \subset L^2(\Omega)$, and an activation function $\rho: \R \to \R$,
the supremal $\gamma > 0$ so that there exist an $L \in \N$ and a polynomial $\pi$ such that
%
%we define, for $f \in \cC$ and $M\in \N$,
\begin{align} \label{eq:GammaMDictDef-nn}
%\Gamma_M^\mathcal{D,\text{eff}}(f) := 
\sup_{f \in \mathcal{C}}\, \inf_{\Phi_{M} \in {\cNN}_{L, M, d, \rho}^\pi}  \|f    - \Phi_{M}\|_{L^2(\Omega)} \in  
\mathcal{O}(M^{-\gamma}), \, M \rightarrow \infty,
%
% \text{all weights of } \Phi_{M} \text{bounded in absolute value by } ,
\end{align}
where ${\cNN}_{L, M, d, \rho}^\pi$ denotes the class of networks in ${\cNN}_{L, M, d, \rho}$ that have all their weights bounded in absolute value by $\pi(M)$,
%\inf_{\substack{I_M\subset \{1,\dots, \pi(M)\},\\ \#I_M = M, (c_i)_{i \in I_M}\!, \max_{i\in I_M}|c_i|\le D}} \left\|f - \sum_{i \in I_M} c_i \varphi_i\right\|_{L^2(\Omega)} = \mathcal{O}(M^{-
%\gamma})
will be referred to as 
%We call $\Gamma_M^\mathcal{D,\text{eff}}(f)$ the {\em best effective $M$-term approximation error of $f$ in $\mathcal{D}$}.
%Every $f_M = \sum_{i \in I_M} c_i \varphi_i$ attaining the infimum in \eqref{eq:GammaMDictDef} is referred to as a {\em best $M$-term approximation} of $f$ in $\mathcal{D}$.
%The supremal $\gamma > 0$ such that 
%there exists $C>0$ with
%\[
%\sup_{f \in \cC}\Gamma_M^\mathcal{D,\text{eff}}(f) =\mathcal{O}(M^{-\gamma}), 
%\qquad \mbox{ for all } M \in \N,
%\]
%will be referred to as 
{\em effective best $M$-edge approximation rate of $\cC$ by neural networks}\/ and 
denoted by $\gamma^{\ast,\text{eff}}_{\cNN}(\cC, \rho)$.
%We say that the {\em best effective $M$-term approximation rate of $\cC$ in the representation system $\mathcal{D}$} is $\gamma^{\ast,\text{eff}}(\mathcal{C},\mathcal{D})$.
\end{definition}
We will show in Corollary \ref{thm:EffRepNN} that $\sup_{\rho:\R \to \R} \gamma^{\ast,\text{eff}}_{\cNN} (\cC, \rho)$
is bounded and depends on the ``description complexity'' of the function class $\cC$. 

\section{Fundamental Bounds on Effective $M$-Term and $M$-Edge Approximation Rate}\label{sec:lowerbound}

The purpose of this section is to establish fundamental bounds on effective best $M$-term and effective best $M$-edge approximation rates by evaluating
the corresponding approximation strategies in the
%establish a lower bound on the worst-case connectivity of neural networks approximating elements from a given function class $\cC \subset L^2(\R^d)$ to within a prescribed %accuracy $of $\varepsilon>0$. A key ingredient is the concept of
min-max rate distortion theory framework as developed in \cite{DONOHO1993100,grohs2015optimally}.
%This theory provides fundamental bounds on the length of lossy compression by encoder-decoder pairs. 
%Based on this general theory, we will then analyze best $M$-term approximation by representation systems and approximation through neural networks from an encoder-decoder %perspective.

%------------------------------------------------------------------------------------------------------------------------------
\subsection{Min-Max Rate Distortion Theory} \label{subsec:ratedistorsion}
%------------------------------------------------------------------------------------------------------------------------------
Min-max rate distortion theory provides a theoretical foundation for deterministic lossy data compression. We recall
the following notions and concepts from \cite{DONOHO1993100,grohs2015optimally}.

Let $d\in \N$, $\Omega \subset \R^d$, and consider the function class $\cC\subset L^2(\Omega)$. Then, for each $\ell\in \N$, we denote by
\begin{align*}
    \mathfrak{E}^\ell:= \left\{E: \cC \to \{0,1\}^{\ell}\right\}
\end{align*}
the set of \emph{binary encoders of $\cC$ of length $\ell$}, and we let
\begin{align*}
    \mathfrak{D}^\ell:= \left\{D:\{0,1\}^{\ell} \to  L^2(\Omega)\right\}
\end{align*}
be the set of \emph{binary decoders of length $\ell$}. An encoder-decoder pair  $(E, D) \in \mathfrak{E}^\ell \times \mathfrak{D}^\ell$
is said to {\em achieve uniform error $\varepsilon$ over the function class $\cC$}, if
\begin{align*}
    \sup_{f\in \cC} \|D(E(f)) - f \|_{L^2(\Omega)} \leq \varepsilon.
\end{align*}
%This means that the worst-case error incurred by applying the encoder-decoder pair  $(E, D) \in \mathfrak{E}^\ell \times \mathfrak{D}^\ell$ to the elements of $\cC$ is upper-bounded %by $\varepsilon$, often also expressed as the uniform error over $\cC$ being no more than $\varepsilon$.

A quantity of central interest is the minimal length $\ell\in \N$ for which there exists an encoder-decoder pair $(E, D) \in \mathfrak{E}^\ell \times \mathfrak{D}^\ell$ that achieves uniform error $\varepsilon$ over the function class $\cC$, along with its asymptotic behavior as made precise in the following definition.

\begin{definition}\label{def:optexp}
Let $d\in \N$, $\Omega \subset \R^d$, and $\cC\subset L^2(\Omega)$. Then, for $\varepsilon >0$, the \emph{minimax code length} $L(\varepsilon, \cC)$ is
\[
L(\varepsilon, \cC) := \min\left\{\ell\in \N: \exists (E,D) \in  \mathfrak{E}^\ell \times \mathfrak{D}^\ell:  \sup_{f\in \cC} \|D(E(f)) - f \|_{L^2(\Omega)} \leq \varepsilon\right\}.
\]
Moreover, the \emph{optimal exponent} $\gamma^*(\cC)$ is defined as
\[
\gamma^*(\cC): = \sup \left \{\gamma \in \R: L(\varepsilon, \cC) \in \mathcal{O} \! \left(\varepsilon ^{-1/\gamma}\right), \, \varepsilon \rightarrow 0 \right\}.
%\exists C>0:\ \forall \varepsilon>0:\ \varepsilon \le C\cdot L(\varepsilon, \cC)^{-\gamma} \}.
\]
\end{definition}
The optimal exponent $\gamma^*(\cC)$ quantifies the minimum growth rate of $L(\varepsilon, \cC)$ as the error $\varepsilon$ tends to zero and
can hence be seen as quantifying the ``description complexity'' of the function class $\cC$. Larger $\gamma^*(\cC)$ results in smaller growth 
rate and hence smaller memory requirements for storing signals $f\in \cC$ such that reconstruction
with uniformly bounded error is possible. The quantity $\gamma^*(\cC)$ is closely related to the concept of Kolmogorov entropy \cite{Off2002MetricEntropy}. Remark 5.10 in \cite{grohs2015optimally} makes this connection explicit.

The optimal exponent
is known for several function classes, such as subsets of Besov spaces $B_{p,q}^s(\R^d)$ with $1\leq p,q<\infty, s>0$, and $q>(s+1/2)^{-1}$, namely all functions in $B_{p,q}^s(\R^d)$ of bounded norm, see e.g. \cite{cohen2001tree}. If $\cC$ is a bounded subset of $B_{p,q}^s(\R^d)$, then we have $\gamma^*(\cC) ={s}/{d}$.
In the present paper, we shall be particularly interested in so-called $\beta$-cartoon-like functions,
for which the optimal exponent is given by ${\beta}/{2}$ (see \cite{Don2001Sparse,Grohs2016} and Theorem~\ref{thm:ExponentOfCartoons}).
%------------------------------------------------------
\subsection{Fundamental Bound on Effective Best $M$-Term Approximation Rate}\label{subsec:repcode}

We next recall a result from \cite{DONOHO1993100,grohs2015optimally}, which says that,
for a given function class $\cC$, the optimal exponent $\gamma^*(\cC)$ constitutes a fundamental bound on the effective best $M$-term 
approximation rate of $\cC$ in any representation system. This gives operational meaning to $\gamma^*(\cC)$.
% characterizing the description 
%complexity of $\cC$.
\begin{theorem}[\cite{DONOHO1993100,grohs2015optimally}]\label{thm:optDictApproxLwrBd}
    Let $d\in \N$, $\Omega\subset \mathbb{R}^d$, $\mathcal{C}\subset L^2(\Omega)$, and assume that the effective best $M$-term approximation rate of $\cC$ in $\mathcal{D} \subset L^2(\Omega)$ is $\gamma^{\ast,\text{eff}}(\cC,\mathcal{D})$. Then, we have
    $$
	    \gamma^{\ast,\text{eff}}(\cC,\mathcal{D}) \leq {\gamma^\ast(\cC)}.
    $$
\end{theorem}
In light of this result the following definition is natural (see also \cite{grohs2015optimally}).
\begin{definition}\label{def:repopti}
Let $d\in \N$, $\Omega\subset \mathbb{R}^d$, and assume that the effective best $M$-term approximation rate of $\cC \subset L^2(\Omega)$ in $\mathcal{D} \subset L^2(\Omega)$ is $\gamma^{\ast,\text{eff}}(\cC,\mathcal{D})$. If
$$
\gamma^{\ast,\text{eff}}(\cC,\mathcal{D}) = {\gamma^\ast(\cC)},
$$
then the function class $\cC$ is said to be \emph{optimally representable}\/ by $\mathcal{D}$.
% is \emph{optimally represented}\/ by the representation system $\mathcal{D}$.
% is said to be \emph{optimal for the function class $\cC$}.
\end{definition}
%
%------------------------------------------------------------------------------------------------------------------------------
\subsection{Fundamental Bound on Effective Best $M$-Edge Approximation Rate} \label{subsec:NNencoding}
%-----------------------------------------------------------------------------------------------------------------------------

We now state the first main result of the paper, namely the equivalent of Theorem~\ref{thm:optDictApproxLwrBd} for approximation by deep neural networks.
Specifically, we establish that the optimal exponent $\gamma^*(\cC)$ also constitutes a fundamental bound on the effective best $M$-edge
approximation rate of $\cC$. We say below that a function $f: \R \to \R$ is \emph{dominated} by a function $g: \R \to \R$ if $|f(x)| \leq |g(x)|$, for all $x\in \R$.
\begin{theorem}\label{thm:EffRepNN}
	Let $d\in \N$, $\Omega\subset \mathbb{R}^d$ be bounded, and $\cC \subset L^2(\Omega)$. Then, for all $\rho: \R \to \R$ that are Lipschitz-continuous or differentiable with $\rho'$ dominated by an arbitrary polynomial, we have 
	$$
	\gamma_{\cNN}^{\ast, \text{eff}}(\cC, \rho) \leq {\gamma^\ast(\cC)}.
	$$
\end{theorem}

The key ingredients of the proof of Theorem~\ref{thm:EffRepNN} are developed throughout this section and the formal proof will be stated at the end of the section. Before embarking on this, we note that, in analogy to Definition~\ref{def:repopti}, what we just found suggests the following.
%The lower bound of Corollary \ref{thm:EffRepNN}, similar to Definition~\ref{def:repopti}, suggests the following definition.  
%
\begin{definition}\label{def:NNopti}
	For $d\in \N$, $\Omega \subset \R^d$ bounded, we say that the function class $\mathcal{C}\subset L^2(\Omega)$ is \emph{optimally representable by neural networks with activation function $\rho: \R \to \R$}, if 
	$$\gamma_{\cNN}^{\ast, \text{eff}}(\cC, \rho) = {\gamma^*(\cC)}.$$
\end{definition}

It is remarkable that the fundamental limits of approximation through representation systems and approximation through deep neural networks are determined by the same quantity, although the approximants in the two cases
are vastly different, namely linear combinations of elements of a representation system with the participating functions identified subject to a polynomial-depth search constraint in the former,
and concatenations of affine functions followed by non-linearities under growth constraints on the weights in the network in the latter case.

A key ingredient of the proof of Theorem~\ref{thm:EffRepNN} is the following result, which establishes a fundamental lower bound on
the connectivity of networks with quantized weights achieving uniform error $\varepsilon$ over a given function class.
\begin{proposition}\label{prop:optimalitynoquant}
Let $d\in \N$, $\Omega\subset \mathbb{R}^d$, $\rho :\R \to \R$, $c>0$, and $\cC \subset L^2(\Omega)$. Further, let
\[
\Learn: \left(\!0,\frac{1}{2}\right) \times \cC \to \cNN_{\infty, \infty, d, \rho}
\]
be a map such that, for each pair $(\varepsilon,f)\in (0, 1/2)\times\cC$, every weight of the neural network  $\Learn(\varepsilon, f)$
is represented by no more than $\lceil c\log_2(\varepsilon^{-1}) \rceil$ bits while guaranteeing that
\begin{equation}
\label{eq:learnalgest}
    \sup_{f \in \cC}\|f - \Learn(\varepsilon, f)\|_{L^2(\Omega)} \leq \varepsilon.
\end{equation}

Then, 
%\footnote{$f = \omega (g)$ stands for $\lim |f/g| = \infty$.}
\begin{equation}\label{eq:NotNNoptiNotQuant}
    \sup_{f\in \cC}\mathcal{M}(\Learn(\varepsilon, f)) \notin \mathcal{O}\! \left(\varepsilon^{-1/\gamma}\right), \, \varepsilon \rightarrow 0, \quad \mbox{for all }\gamma> {\gamma^\ast(\cC)}.
    \end{equation}
%    \sup_{\varepsilon\in (0,\frac12)}\varepsilon^{\frac{1}{\gamma}}\cdot\sup_{f\in \cC}\mathcal{M}(\Learn(\varepsilon, f)) =\infty,\quad \mbox{for all }\gamma> {\gamma^\ast(\cC)}.
%\end{equation}
\end{proposition}

\begin{proof}
The proof will be effected by identifying $\Learn(\varepsilon, f)=D(E(f))$, where
$(E,D) \in \mathfrak{E}^{\ell(\varepsilon)} \times \mathfrak{D}^{\ell(\varepsilon)}$ are encoder-decoder pairs achieving
uniform error $\varepsilon$ over $\cC$ with
\begin{align}\label{eq:CodeLength}
\ell(\varepsilon) \leq C_0 \cdot \sup_{f\in \cC}\left[ \mathcal{M}(\Learn(\varepsilon, f))\log_2(\mathcal{M}(\Learn(\varepsilon, f))) + 1\right]\log_2(\varepsilon^{-1}),
\end{align}
where $C_0>0$ is a constant, and such that the weights in $\Learn(\varepsilon, f)$ are represented by no more than $\lceil c\log_2(\varepsilon^{-1}) \rceil$ bits each.
Before presenting the construction of these encoder-decoder pairs, we establish that this, indeed, implies the statement of the theorem. 
To this end, let $\gamma > \gamma^*(\cC)$ and, towards a contradiction to \eqref{eq:NotNNoptiNotQuant}, assume that $\sup_{f\in \cC}\mathcal{M}(\Learn(\varepsilon, f)) \in \mathcal{O}(\varepsilon^{-1/\gamma}), \varepsilon \rightarrow 0$.
Then, there would exist a $\nu$ with $\gamma > \nu> \gamma^*(\cC)$ such that there are
encoder-decoder pairs $(E,D) \in \mathfrak{E}^{\ell(\varepsilon)} \times \mathfrak{D}^{\ell(\varepsilon)}$ achieving uniform error $\varepsilon$ over $\cC$ with codelength
$\ell(\varepsilon) \in \mathcal{O}( \varepsilon^{-1/\nu}), \varepsilon \rightarrow 0$, which stands in contradiction to the optimality of 
$\gamma^\ast(\cC)$ according to Definition~\ref{def:optexp}.

We proceed to the construction of the encoder-decoder pairs, which will be accomplished
by encoding the network topology and quantized weights in bitstrings of length $\ell(\varepsilon)$ satisfying \eqref{eq:CodeLength} while
guaranteeing unique reconstruction.
Fix $f\in \cC$. We enumerate the nodes in $\Learn(\varepsilon,f)$ by assigning natural numbers, henceforth called \emph{indices},
increasing from left to right in every layer as schematized in Figure \ref{fig:Numbering}. For the sake of notational simplicity, we also set
$\Phi:=\Learn(\varepsilon, f)$ and $M := \mathcal{M}(\Phi)$. %\log_2(\mathcal{M}(\Phi)).
Without loss of generality, we assume throughout that $M$ is a power of $2$ and greater than $1$. For all $M$ that are not powers of $2$, we make use of the fact that $\cNN_{L,M,d,\rho} \subset \cNN_{L,M',d,\rho}$, where $M'$ is the smallest power of $2$ larger than $M$, and we encode the network like an $M'$-edge network. Since $M < M' < 2M$ this affects $\ell(\varepsilon)$ by a multiplicative constant only. The case $M=0$ will de dealt with in Step 1 below.

We recall that the number of layers of $\Phi$ is denoted by $L$, the number of nodes in these layers is $N_1,\dots , N_L$ (see Definition~\ref{def:NN}), and $d$ stands for the dimension of the input layer. 

Denoting the number of nodes in layer $\ell=1,...,L-1$ associated with edges of nonzero weight in the following layer by 
$\widetilde{N}_\ell$ and setting $\widetilde{N}_L=N_L$, it follows that
\begin{align}
    \label{eq:edgeboundsnodes1}
    d + \sum_{\ell=1}^L\widetilde{N}_\ell\le 2 \widetilde{M},
\end{align}
where we let $\widetilde{M} := M + d$.
All other nodes do not contribute to the mapping $\Phi(x)$ and can hence be ignored. 

Moreover, we can assume that 
\begin{equation}\label{eq:edgeboundsnodes2}
    L\le \widetilde{M}
\end{equation}
as otherwise there would be at least one layer $\ell > 1$ such that 
$$
	 A_\ell = 0.
$$
As a consequence, the reduced network
$$
	x \mapsto W_L\rho(W_{L-1}\dots W_{\ell+1}\rho(0 \cdot x + b_{\ell})),
$$
realizes the same function as the original network $\Phi$ but has less than $L$ layers. This reduction can be repeated inductively until the resulting reduced network
satisfies (\ref{eq:edgeboundsnodes2}). 
\begin{figure}
\centering
\begin{tikzpicture}[->,>=stealth',shorten >=1pt,auto,node distance=2cm,
                    thick,main node/.style={circle,draw,font=\sffamily\Large\bfseries}]

  \node[main node] (1) {7};
  \node[main node] (2) [below left of=1] {5};
  \node[main node] (3) [below right of = 1] {6};
  \node[main node] (4) [below left of = 2] {2};
  \node[main node] (5) [below left of = 3] {3};
  \node[main node] (6) [below right of = 3] {4};
  \node[main node] (7) [below  of = 5] {1};

  \path[every node/.style={font=\sffamily}]
    (2) edge node [left] { \ } (1)
    (3) edge node [left] { \ } (1)
    (4) edge node [left] { \ } (2)
    (5) edge node [left] { \ } (2)
        %edge node [left] { \ } (6)
    (6) edge node [right] {\ } (3)
	(4) edge node [right] {  } (3)
    (7) edge node [left] {  } (4)
    (7) edge node [left] { } (5)
    (7) edge node [left] { } (6);
\end{tikzpicture}
  %  \put(-265,15){\bf Input layer}
  %  \put(-265,145){\bf Output layer}
  %  \put(-265,65){\bf Hidden layer}
  %  \put(-265,105){\bf Hidden layer}
%    \put(-85,160){2}
%    \put(-193,110){3}
%    \put(-155,110){4}
%    \put(-105,110){5}
%    \put( -38,110){6}
    \caption{Enumeration of nodes as employed in the proof of the theorem.}\label{fig:Numbering}
\end{figure}
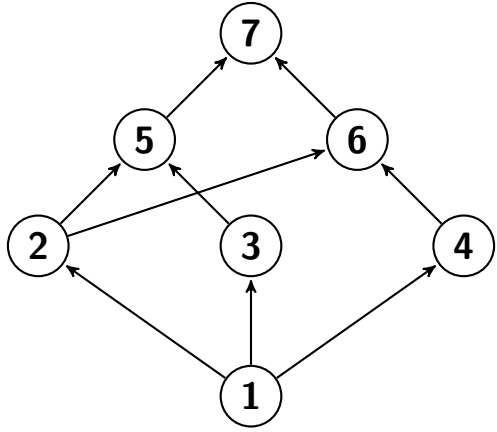

The bitstring representing $\Phi$ is constructed according to the following steps.

{\it Step 1:}  If $M=0$, we encode the network by a leading $0$ followed by the bitstring representing the node weight in the last layer. Upon defining $0 \log_2(0)=0$, we then note that (\ref{eq:CodeLength}) holds trivially and we terminate the encoding procedure. Else, we encode the number of nonzero edge weights, $M$, by starting the overall bitstring with $M$ $1$'s followed by a single $0$. The length of this bitstring is therefore bounded by $\widetilde{M}$.

{\it Step 2:} We continue by encoding the number of layers in the network. Thanks to \eqref{eq:edgeboundsnodes2} this requires no more than $\log_2(\widetilde{M})$ bits. We thus reserve the next $\log_2(\widetilde{M})$ bits for the binary representation of $L$. %\footnote{Here and in the remainder of the proof, we follow the convention that bit strings are zero-padded from the left, that is, for example, if the number of layers $L$ is equal to $2$ and $M=8$, we encode $L$ by the bit string $010$.}

{\it Step 3:} Next, we store the dimension $d$ of the input layer and the numbers of nodes $\widetilde{N}_\ell, \ell = 1, \dots, L$, associated with edges of nonzero weight.
As by \eqref{eq:edgeboundsnodes1} $d \le \widetilde{M}$ and $\widetilde{N}_\ell\le 2\widetilde{M}$, for all $\ell$, we can encode (generously) $d$ and each $\widetilde{N}_\ell$ using
$\log_2(\widetilde{M})+1$ bits. For the sake of concreteness, we first encode $d$ followed by $\widetilde{N}_1, \dots, \widetilde{N}_L$ in that order.
In total, Step 3 requires a bitstring of length
$$
 ((L+1)\cdot (\log_2(\widetilde{M})+1)) \leq  (\widetilde{M}+1)\log_2(\widetilde{M}) + \widetilde{M} + 1.
$$
In combination with Steps 1 and 2 this yields an overall bitstring of length at most
\begin{equation}
\label{sizecode}
\widetilde{M} \log_2(\widetilde{M}) + 2 \log_2(\widetilde{M}) + 2\widetilde{M}+1.
\end{equation}

{\it Step 4:} We encode the topology of the graph associated with $\Phi$ and consider only nodes that contribute to the mapping $\Phi(x)$. 
Recall that we assigned a unique index $i$, ranging from $1$ to $\widetilde{N}:=d+\sum_{\ell=1}^L \widetilde{N}_\ell$,
to each of these nodes.
%These indices range from $1$ to \pp{$\widetilde{N}:=d+\sum_{\ell=1}^L \widetilde{N}_\ell$} and, 
By (\ref{eq:edgeboundsnodes1}) each of these indices can be encoded by a bitstring of length $\log_2(\widetilde{M})+1$. We denote the bitstring corresponding to index $i$ by $b(i)\in \{0,1\}^{\log_2(\widetilde{M})+1}$ and let $n(i)$ be the number of children of the node with index $i$, i.e., the number of nodes in the next layer connected to the node with index $i$ via an edge. 
For each node $i=1,\dots , \widetilde{N}$, we form a bitstring of length $n(i)\cdot (\log_2(\widetilde{M})+1)$ by concatenating the bitstrings $b(j)$ for all $j$ such that there is an edge between $i$ and $j$. We follow this string with an all-zeros bitstring of length $\log_2(\widetilde{M})+1$ to signal the transition to the node with index $i+1$. 
The enumeration is concluded with an all-zeros bitstring of length $\log_2(\widetilde{M})+1$ signaling that the last node has been reached.
%transition to the last node, which does not have children.
Overall, this yields a bitstring of length
\begin{equation}
\label{topcode}
    \sum_{i=1}^{\widetilde{N}} (n(i)+1)\cdot (\log_2(\widetilde{M})+1) \leq 3\widetilde{M}\cdot (\log_2(\widetilde{M})+1),
\end{equation}
where we used $\sum_{i=1}^{\widetilde{N}} n(i) = M < \widetilde{M}$ and (\ref{eq:edgeboundsnodes1}). 
Combining (\ref{sizecode}) and (\ref{topcode}) it follows that we have encoded the overall topology of the network $\Phi$ using at most
\begin{equation}
\label{eq:graphcode}
    5\widetilde{M} + 4\widetilde{M} \log_2(\widetilde{M})+2\log_2(\widetilde{M})+1
\end{equation}
bits.

{\it Step 5:} We encode the weights of $\Phi$.
% i.e., those associated to the neurons and to the edges of $\Phi$. 
By assumption, each weight can be represented
by a bitstring of length $\lceil c\log_2(\varepsilon^{-1})\rceil$.
For each node $i=1,\dots, \widetilde{N}$, we reserve the first $\lceil c\log_2(\varepsilon^{-1})\rceil$ bits to encode its associated node weight and, for each of its children a bitstring of length $\lceil c\log_2(\varepsilon^{-1})\rceil$ to encode
the weight corresponding to the edge between that child and its parent node. Concatenating the results in ascending order of child node indices, we get
%the order of according to the associated indices in
a bitstring of length $(n(i)+1)\cdot (\lceil c \log_2(\varepsilon^{-1})\rceil)$ for node $i$, and an overall bitstring of length
\begin{equation}\label{weightcode}
    \sum_{i = 1}^{\widetilde{N}} (n(i)+1)\cdot \left(\lceil c \log_2 \! \left(\varepsilon^{-1}\right)\rceil\right) \leq 3\widetilde{M}\cdot  \lceil c\log_2 \! \left (\varepsilon^{-1}\right)\rceil
\end{equation}
representing the weights of the graph associated with the network $\Phi$. 
%Note that the encoder needs to know $\varepsilon$.

With (\ref{eq:graphcode}) this shows that the overall number of bits needed to encode the network topology and weights is no more than
%for the number of bits needed to encode the topology of $\Phi$ and (\ref{weightcode}) for the number of bits needed to encode the corresponding weights, we conclude that a bit %string of length at most
%
\begin{align}\label{eq:HelmutCalledThisAlpha}
    5\widetilde{M} + 4\widetilde{M} \log_2(\widetilde{M})+2\log_2(\widetilde{M})+ 1 + 3\widetilde{M}\cdot  \lceil c\log_2 \left(\varepsilon^{-1}\right)\rceil.
\end{align}
%
%is needed to encode $\Phi$. 
The network can be recovered by sequentially reading out $M,L,d$, the $\widetilde{N}_{\ell}$, the topology, and the quantized weights from the overall bitstring.
It is not difficult to verify that the individual steps in the encoding procedure were crafted such that this yields unique 
recovery. 
As \eqref{eq:HelmutCalledThisAlpha} can be upper-bounded by
\begin{equation} \label{totalcodelength}
    C_0 M \log_2(M) \log_2 \! \left(\varepsilon^{-1}\right)
\end{equation}
for a constant $C_0>0$ depending on $c$ and $d$ only, we have constructed an encoder-decoder pair $(E,D)\in \mathfrak{E}^{\ell(\varepsilon)} \times \mathfrak{D}^{\ell(\varepsilon)}$ 
with  $\ell(\varepsilon)$ satisfying (\ref{eq:CodeLength}).
% \le C_0 M \log_2(M) \! \log_2(\varepsilon^{-1})$ achieving uniform error $\varepsilon>0$ over the function class $\cC$.
This concludes the proof.
\end{proof}
Proposition \ref{prop:optimalitynoquant} applies to networks that have each weight represented by a finite number of bits scaling according to $\log_2(\varepsilon^{-1})$ 
while guaranteeing that the underlying encoder-decoder pair achieves uniform error $\varepsilon$ over $\cC$. 
We next show that such a compatibility is possible for networks with activation functions that are either Lipschitz or differentiable such that $\rho'$ is dominated by an arbitrary polynomial. 
We can now demonstrate that for sufficiently regular activation functions, faithful quantization of the weights of a network is possible.

\begin{lemma}\label{lem:PolynomiallyboundedImpliesFiniteBitLength}
Let $d, L, k, M \in \N, \eta\in (0,1/2), \Omega\subset \mathbb{R}^d$ be bounded, and $\rho : \R \to \R$ be either Lipschitz-continuous or differentiable such that $\rho'$ is dominated by an arbitrary polynomial.
Let $\Phi \in \cNN_{L,M,d,\rho}$ with $M \leq \eta^{-k}$ and all its weights bounded (in absolute value) by $\eta^{-k}$. Then, there exist $m\in \N$, depending on $k, L$, and $\rho$ only, and  $\widetilde{\Phi} \in \cNN_{L,M,d,\rho}$ such that
$$
     \| \widetilde{\Phi} - \Phi\|_{L^\infty(\Omega)} \leq \eta
$$
and all weights of $ \widetilde{\Phi}$ are elements of $\eta^{m} \Z \cap [-\eta^{-k}, \eta^{-k}]$.
\end{lemma}

\begin{proof}
We prove the statement for Lipschitz-continuous $\rho$ only. The argument for differentiable activation functions with first derivative not growing faster than every polynomial is along similar lines.

Let $m\in \N$, to be specified later, and denote by $\widetilde{\Phi}$ the network that results by replacing all weights of $\Phi$ by a closest element in $\eta^{m} \Z \cap [-\eta^{-k}, \eta^{-k}]$. Set $C_{\mathrm{max}}:=\eta^{-k}$ and denote the maximum of $1$ and the total number of edge weights plus node weights that contribute to the mapping $\Phi(x)$ by $C_W$.
Note that $C_W \leq 3M \leq 3\eta^{-k}$, where the latter inequality is by assumption.
For $\ell=1,\dots, L-1$, define $\Phi^\ell: \Omega \to \R^{N_\ell}$ as
$$
\Phi^\ell(x) := \rho \, ( W_{\ell} \rho\, ( \dots \rho \, ( W_{1}(x)))), \quad \text{ for }x\in \Omega,
$$
and $\widetilde{\Phi}^\ell$ accordingly, and let, for $\ell=1,\dots, L-1$,
\begin{align*}
 e_\ell : = \left \|\Phi^\ell - \widetilde{\Phi}^\ell \right \|_{L^\infty(\Omega, \R^{N_\ell})}, \quad \ e_L : = \left \|\Phi - \widetilde{\Phi} \right \|_{L^\infty(\Omega)}.
\end{align*}
Denote the maximum of 1 and the Lipschitz constant of $\rho$ by $C_\rho$, set $C_0 := \max\{1, \sup \{|x|: x\in \Omega\}\}$, and let 
$$
C_\ell := \max\left\{\left\|\Phi^\ell\right\|_{L^\infty(\Omega, \R^{N_\ell})}, \left\|\widetilde{\Phi}^\ell\right\|_{L^\infty(\Omega, \R^{N_\ell})}\right\},\quad \text{ for }\ell = 1, \dots, L-1.
$$
Then, it is not difficult to see that
\begin{align}
e_1 \leq C_0\, C_\rho\, C_W\, \eta^m, \text{ and } e_\ell  \leq  C_\rho\,  C_W\,  C_{\ell-1}\,  \eta^m + C_\rho\,  C_W\,  C_{\mathrm{max}}\,  e_{\ell-1}, \text{ for all } \ell = 2,\dots, L-1.\label{eq:EstimateOfEEll}
\end{align}
Additionally, we observe that 
\begin{align}\label{eq:EL}
 e_L \leq C_W\, C_{L-1}\,  \eta^m + C_W\,  C_{\mathrm{max}}\,  e_{L-1}.
\end{align}
We now bound the quantity $C_\ell$ for $\ell = 1,\dots , L-1$. A simple computation, exploiting the Lipschitz-continuity of $\rho$, yields
\begin{align*}
C_\ell \leq (|\rho(0)| + C_\rho\,  C_W\,  C_{\mathrm{max}}\,  C_{\ell-1}),\quad \mbox{ for all }\ell = 1,\dots, L-1.
\end{align*}
Since $\rho$ is continuous on $\R$ we have $|\rho(0)|<\infty$ and thus, by $C_\rho, C_W,  C_{\mathrm{max}} \ge 1$, there exists $C'>0$ such that
$$
C_\ell \leq C'\,  C_{0}\,  (C_\rho\,  C_W\,  C_{\mathrm{max}})^\ell, \quad \mbox{for all  }\ell=1,\dots, L-1.
$$
%%%P5: Warum ^{-k-2}? und nicht einfach ^{-k}? Wir haben oben, dass C_W \leq 3 \eta^{-k}. Da 3 \leq 2^2 \leq \eta^{-2} kommt dann diese Abschaetzung zustande.
As $C_W$ and $C_{\mathrm{max}}$ are both bounded by $\eta^{-k-2}$, it follows that $C_\ell$ is bounded by $\eta^{-p}$ for a $p\in \N$.
We can therefore find $n\in \N$ such that 
\begin{equation}
\max\{C_{0}C_{\rho}C_W, C_W\,  C_{\mathrm{max}}, C_W\,  C_{L-1}, C_\rho\,  C_W\,  C_{\ell-1}, C_\rho\,  C_W\,  C_{\mathrm{max}}\} \leq \frac{\eta^{-n}}{2}, \quad \mbox{for all } \ell=1,\dots, L-1.
\label{eq:cboundspoly}
\end{equation}
Invoking \eqref{eq:EstimateOfEEll}, we conclude that 
\begin{align}\label{eq:inductionStep}
e_{\ell} \leq \frac{\eta^{-n}}{2}(\eta^m + e_{\ell-1}), \quad \mbox{for all } \ell = 1,\dots, L-1,
\end{align}
where we set $e_{0}=0$. We proceed by induction to prove that there exists $r\in \N$ such that for all $\ell = 1,\dots, L-1$,
\begin{align} \label{eq:ThisFollowsByInduction}
e_{\ell} \leq \eta^{m-(\ell-1) n - r}.
\end{align}
Clearly there exists $r \in \N$ such that $e_{1} \leq \eta^{m-r}$. 
Moreover, one easily verifies that the existence of an $r \in \N$ such that \eqref{eq:ThisFollowsByInduction} is satisfied for an $\ell\in \{1,\dots, L-2\}$, thanks to \eqref{eq:inductionStep}, implies the existence of an $r \in \N$ such that \eqref{eq:ThisFollowsByInduction} is satisfied for $\ell$ replaced by $\ell+1$. This concludes the induction argument. 

Using \eqref{eq:cboundspoly} and \eqref{eq:ThisFollowsByInduction} in  \eqref{eq:EL}, we finally obtain
$$
e_L \leq  \frac{\eta^{m-n}}{2}  + \frac{\eta^{m-(L-1)n - r}}{2},
$$
which yields $e_L \leq \eta$ for sufficiently large $m$.
\end{proof}

\begin{remark}
Note that the weights of the network being elements of $\eta^{m} \Z \cap [-\eta^{-k}, \eta^{-k}]$ implies that each weight can be represented by no more than
$\lceil c \log_2(\eta^{-1})\rceil$ bits, for some constant $c>0$.
\end{remark}

Proposition \ref{prop:optimalitynoquant} not only says that the connectivity growth rate can not exceed
$\mathcal{O}\! \left(\varepsilon^{-1/\gamma^\ast(\cC)}\right), \, \varepsilon \rightarrow 0$,
but its proof, by virtue of constructing an encoder-decoder pair that achieves this growth rate also provides an achievability result.
We next establish a matching strong converse in the sense of showing that for $\gamma> {\gamma^\ast(\cC)}$, the uniform approximation error
remains bounded away from zero for infinitely many $M \in \N$. To simplify terminology in the sequel, we introduce the notion of a polynomially bounded variable.
%P5: I scheint hier in 2 verschiedenen contexts verwendet worden zu sein.
\begin{definition}\label{def:polybd}
A real variable $X$ depending on the variables $z_i \in D_i \subset \R$, $i = 1, \dots, N$, is said to be \emph{polynomially bounded in $z_1, \dots, z_N$}, if there exists an $N$-dimensional polynomial $\pi$ such that $|X| \leq |\pi( z_1, \dots, z_N)|$, for all $z_i \in D_i, i=1,\dots,N$. A set of real variables $(X_j)_{j\in J}$, each depending on $z_i \in D_i \subset \R$, $i = 1, \dots, N$, is \emph{uniformly polynomially bounded in $z_1, \dots, z_N$}, if there exists an $N$-dimensional polynomial $\pi$ such that $|X_j| \leq |\pi( z_1, \dots, z_N)|$, for all $j\in J$ and all $z_i \in D_i$, $i = 1, \dots, N$. 
\end{definition}
We will refrain from explicitly specifying the $D_i$ in Definition \ref{def:polybd} whenever they are clear from the context.

\begin{remark}\label{rem:PolyBoundAndExpoBoundareTheSame}
If $D_i = \R \setminus [-B_i,B_i]$ for some $B_i \ge 1, i=1,\dots,N$, then a variable $X$ depending on $z_i \in D_i, i=1,\dots,N,$ is polynomially bounded in $z_1, \dots, z_N$ if and only if there exists a $k\in \N$ such that $|X| \leq |z_1^k \cdot z_2^k \cdot . \ . \ . \cdot  z_N^k|$, for all $z_i \in D_i$.
\end{remark}

\begin{proposition}\label{prop:optimality}
Let $d, L\in \N$, $\Omega\subset \mathbb{R}^d$ be bounded, $\pi$ be a polynomial, $\cC \subset L^2(\Omega)$, $\rho : \R \to \R$ either Lipschitz-continuous or differentiable such that $\rho'$ is dominated by an arbitrary polynomial. Then, for all $C > 0$ and $\gamma > \gamma^*(\cC)$ we have that
\begin{align} \label{eq:fundbound2}
 \sup_{f \in \cC}  \ \inf_{\Phi \in {\cNN}_{L, M, d, \rho}^\pi} \|f - \Phi\|_{L^2(\Omega)} \geq C M^{-\gamma}, \text{ for infinitely many }M \in \N.
\end{align}
%where ${\cNN}_{L, M, d, \rho}^\pi$ denotes the class of networks $\Phi \in \cNN_{L, M, d, \rho}$ with weights bounded by $|\pi(M)|$.
\end{proposition}
\begin{proof}
Let $\gamma > \gamma^*(\cC)$. Assume, towards a contradiction, that \eqref{eq:fundbound2} holds only for finitely many $M \in \N$. Then, there exists a constant $C$ such that
\eqref{eq:fundbound2} holds for no $M \in \N$ and hence there exists a constant $C$ so that
$$
	\sup_{f \in \cC}  \ \inf_{\Phi \in {\cNN}_{L, M, d, \rho}^\pi} \|f - \Phi\|_{L^2(\Omega)} \leq CM^{-\gamma}, \quad \text{ for all } M \in \N.
$$
Setting $M_\varepsilon := \lceil (\varepsilon/(3C))^{-1/\gamma}\rceil$, it follows that, for every $f \in \mathcal{C}$ and every $\varepsilon \in (0,1/2)$, there exists a neural network $\Phi_{\varepsilon,f} \in {\cNN}_{L, M_\varepsilon, d, \rho}^\pi$ such that 
$$
     \|f- \Phi_{\varepsilon,f} \|_{L^2(\Omega)} \leq 2 \sup_{f \in \cC}  \ \inf_{\Phi \in {\cNN}_{L, M_\varepsilon, d, \rho}^\pi} \|f - \Phi\|_{L^2(\Omega)} \leq 2 C M_\varepsilon^{-\gamma} \leq \frac{2 \varepsilon}{3}.
$$
As the weights of $\Phi_{\varepsilon,f}$ are polynomially bounded in $M_{\varepsilon}$, they are polynomially bounded in $\varepsilon^{-1}$. By Lemma \ref{lem:PolynomiallyboundedImpliesFiniteBitLength} and Remark \ref{rem:PolyBoundAndExpoBoundareTheSame}, there hence exists a network $\widetilde{\Phi}_{\varepsilon,f}$ 
whose weights are represented by no more than $\lceil c\log_2(\varepsilon^{-1}) \rceil$ bits, for some constant $c>0$, satisfying 
$$
    \left\|\Phi_{\varepsilon,f} - \widetilde{\Phi}_{\varepsilon,f} \right\|_{L^2(\Omega)} \leq \frac{\varepsilon}{3}.
$$
Defining
\begin{equation*}
    \Learn: \left(\!0,\frac{1}{2}\right) \times \cC  \to \cNN_{\infty, \infty, d, \rho}, \quad (\varepsilon, f) \mapsto \widetilde{\Phi}_{\varepsilon, f},
\end{equation*}
it follows that
$$
    \sup_{f \in \cC}\|f - \Learn(\varepsilon, f)\|_{L^2(\Omega)} \leq \varepsilon \quad \text{ with } \quad \mathcal{M}(\Learn(\varepsilon, f)) \leq M_\varepsilon  \in \mathcal{O}(\varepsilon^{-1/\gamma}), \, \, \varepsilon \to 0.
$$
The proof is concluded by noting that $\textbf{Learn}$ violates Proposition \ref{prop:optimalitynoquant}.
\end{proof}
We can now proceed to the proof of Theorem~\ref{thm:EffRepNN}.
\begin{proof}[Proof of Theorem~\ref{thm:EffRepNN}]
		Suppose towards a contradiction that 
		$
		\gamma_{\cNN}^{\ast, \text{eff}}(\cC, \rho) > {\gamma^\ast(\cC)}
		$. Let $\gamma \in \left(\gamma^\ast(\cC),\gamma_{\cNN}^{\ast, \text{eff}}(\cC, \rho)\right)$.
		Then, Definition~\ref{def:repoptiNN} implies that 
		there exist a polynomial $\pi, L\in \mathbb{N}$, and $C>0$ such that
		\begin{align*} 
		\sup_{f \in \mathcal{C}}\, \inf_{\Phi_{M} \in \cNN_{L, M,d,\rho}^\pi}  \|f    - \Phi_{M}\|_{L^2(\Omega)} \leq
		C M^{-\gamma}, \,\, \text{for all}\,\, M \in \N.
		\end{align*}
		This, however, constitutes a contradiction to Proposition \ref{prop:optimality}.
\end{proof}

We conclude this section with a discussion of the conceptual implications of the results established above. Proposition \ref{prop:optimalitynoquant} combined with Lemma \ref{lem:PolynomiallyboundedImpliesFiniteBitLength} establishes that
neural networks with weights polynomially bounded in $\varepsilon^{-1}$ and achieving uniform approximation error $\varepsilon$ over $\mathcal{C}$ 
cannot exhibit edge growth rate smaller than $\mathcal{O}(\varepsilon^{-1/\gamma^{*}(\mathcal{C})}), \varepsilon \rightarrow 0$; in other words, a decay of the uniform approximation error, as a function of $M$, faster than $\mathcal{O}(M^{-\gamma^{\ast}(\mathcal{C})}), M \rightarrow \infty$, is not possible. 

Note that requiring uniform approximation error $\varepsilon$ only (without imposing the constraint of the network's weights being polynomially bounded in $\varepsilon^{-1}$)
can lead to arbitrarily large rate $\gamma$ as exemplified by Theorem~\ref{pinkusthm}, which proves the existence of networks realizing an arbitrarily small approximation error over $L^{2}([0,1]^{d})$ with a finite number of nodes; in particular, the number of nodes remains constant as $\varepsilon \rightarrow 0$. However, as argued right after Theorem~\ref{pinkusthm}, these networks necessarily lead to weights that are not polynomially bounded in $\varepsilon^{-1}$.

%------------------------------------------------------------------------------------------------------------------------------
\section{Transitioning from Representation Systems to Neural Networks}\label{sec:bestapprox}
%------------------------------------------------------------------------------------------------------------------------------
The remainder of this paper is devoted to identifying function classes that are optimally representable---according to Definition~\ref{def:NNopti}---by neural networks. The mathematical technique we develop in the process is interesting in its own right as it constitutes a general framework for transferring results on function approximation through representation systems to results on approximation by neural networks.
In particular, we prove that for a given function class $\cC$ and an associated representation system $\mathcal{D}$ which satisfies certain technical conditions, there exists
a neural network with $\mathcal{O}(M)$ nonzero edge weights that achieves (up to a multiplicative constant) the same uniform error over $\cC$ as a best $M$-term approximation
in $\mathcal{D}$. This will finally lead to a characterization of function classes $\mathcal{C}$ that are
optimally representable by neural networks in the sense of Definition~\ref{def:NNopti}.

We start by stating technical conditions on representation systems for the transference principle outlined above to apply.
\begin{definition}\label{def:wellrep} Let $d \in \N$, $\Omega\subset \mathbb{R}^d$, $\rho : \R \to \R$, and $\mathcal{D} = (\varphi_i)_{i\in I}\subset L^2(\Omega)$ be a representation system. Then,
$\mathcal{D}$ is said to be \emph{representable by neural networks (with activation function $\rho$)}, if
there exist $L, R \in \N$ such that for all $\eta >0$ and every $i\in I$, there is a neural network $\Phi_{i,\eta}\in \cNN_{L, R, d, \rho}$ with
    $$
        \|\varphi_i - \Phi_{i,\eta}\|_{L^2(\Omega)}\le \eta.
    $$
   If, in addition, the weights of $\Phi_{i,\eta}\in \cNN_{L, R, d, \rho}$ are polynomially bounded in $i,\eta^{-1}$, 
   and if $\rho$ is either Lipschitz-continuous or differentiable such that $\rho'$ is dominated by an arbitrary polynomial, then we say that $\mathcal{D}$ is \em{effectively representable by neural networks (with activation function $\rho$)}.
\end{definition}
%
%
%The next result formalizes our transference principle for networks with \pp{bounded} weights in $\R$. Stimmt! bounded war ueberfluessig.
The next result formalizes our transference principle for networks with weights in $\R$.
%relates best $M$-term approximation properties of representation systems
%to $M$-edge approximation properties for neural networks.
%
\begin{theorem}\label{theo:ApproxOfNeuralNetworks}
Let $d \in \N$, $\Omega \subset \R^d$, and $\rho : \R \to \R$. Suppose that  $\mathcal{D} = (\varphi_i)_{i\in I} \subset L^2(\Omega)$ is representable by neural networks. Let $f\in L^2(\Omega)$
and, for $M\in \mathbb{N}$, let $f_M = \sum_{i\in I_M} c_i \varphi_i$, $I_M\subset I$, $\#I_M = M$,
satisfy
$$
\| f - f_M\|_{L^2(\Omega)} \leq \varepsilon,
$$
where $\varepsilon \in (0,1/2)$.
Then, there exist $L \in \N$ (depending on $\mathcal{D}$ only) and a neural network $\Phi(f,M)\in \cNN_{L, M', d, \rho}$ with $M' \in \mathcal{O}(M)$, satisfying
\begin{align}\label{eq:ApproxErrorOfDictNetwork}
    \|f - \Phi(f,M)\|_{L^2(\Omega)} \le 2\varepsilon.
\end{align}
In particular, for all function classes $\cC\subset L^2(\Omega)$ it holds that
\begin{align} \label{eq:theSameStatementAgainButWithSymbols}
    \gamma^\ast_{\cNN}(\cC, \rho)\geq \gamma^\ast(\cC,\mathcal{D}).
\end{align}%
\end{theorem}
\begin{proof}
By representability of $\mathcal{D}$ according to Definition~\ref{def:wellrep}, it follows that there exist $L,R\in\mathbb{N}$, 
%depending on $\mathcal{D}$ only, 
such that for each $i \in I_M$ and for $\eta:=\varepsilon / \max\{1,\sum_{i\in I_M}|c_i|\}$, there exists a neural network $\Phi_{i,\eta} \in \cNN_{L, R, d, \rho}$ with
\begin{align} \label{eq:2}
\|\varphi_i - \Phi_{i,\eta}\|_{L^2(\Omega)} \le \eta.
\end{align}
Let then $\Phi(f,M)$ be the neural network consisting of the networks $(\Phi_{i,\eta})_{i \in I_M}$ operating in parallel, all with the same
input, and summing their one-dimensional outputs (see Figure \ref{fig:network} below for an illustration) with weights $(c_i)_{i\in I_M}$ according to
\begin{equation}
\label{eq:NNadd}
    \Phi(f,M)(x):= \sum_{i\in I_M}c_i\Phi_{i,\eta}(x), \quad \text{ for }x \in \Omega.
\end{equation}
This construction is legitimate as all networks $\Phi_{i,\eta}$ have the same number of layers and the last layer of a neural network according to Definition~\ref{def:NN} implements an affine function only (without subsequent application of the activation function $\rho$). Then, $\Phi(f,M) \in \cNN_{L, RM, d, \rho}$, and application of the triangle inequality together with \eqref{eq:2} 
%and \pp{$\eta=\varepsilon / \max\{ 1, \sum_{i\in I_M}|c_i|\}$} 
yields $\left\|f_M - \Phi(f,M)\right\|_{L^2(\Omega)} \le \varepsilon$.
Another application of the triangle inequality according to 
\begin{align*}
    \|f - \Phi(f,M)\|_{L^2(\Omega)} \leq \| f - f_M\|_{L^2(\Omega)}  + \| f_M -  \Phi(f,M)\|_{L^2(\Omega)}  \leq 2\varepsilon 
\end{align*}
finalizes the proof of \eqref{eq:ApproxErrorOfDictNetwork} which by Definitions \ref{def:optimalApproximationRate} and \ref{def:optimalApproximationRateNN} implies \eqref{eq:theSameStatementAgainButWithSymbols}.
\end{proof}

Theorem~\ref{theo:ApproxOfNeuralNetworks} shows that we can restrict ourselves to the approximation of the individual elements of a representation
system by neural networks with the only constraint being that the number of nonzero edge weights in the individual networks must admit a uniform upper bound.
Theorem~\ref{theo:ApproxOfNeuralNetworks} does, however, not guarantee that the weights of the network $\Phi(f,M)$ can be represented with no more than $\lceil c\log_2(\varepsilon^{-1}) \rceil$ bits when the overall approximation error is proportional to $\varepsilon$. This will again be accomplished through a transfer argument, applied to representation systems $\mathcal{D}$ satisfying slightly more stringent technical conditions.

\begin{theorem}\label{theo:EncodeOfNeuralNetworks}Let $d\in \N$, $\Omega\subset \mathbb{R}^d$ be bounded, and $\mathcal{C}\subset L^2(\Omega)$. Suppose that
  the representation system $\mathcal{D}=(\varphi_i)_{i\in \mathbb{N}}\subset L^2(\Omega)$ is effectively representable by neural networks.
    Then, for all $\gamma <\gamma^{\ast, \text{eff}}(\mathcal{C},\mathcal{D})$, there exist a polynomial $\pi$, constants $c > 0, L \in \N$, and a map
    $$
        \Learn: \left(\!0,\frac{1}{2}\right) \times L^2(\Omega) \to \cNN_{L, \infty, d, \rho}^{\pi} ,
    $$
    such that for every $f\in \cC$ the weights in $\Learn (\varepsilon,f)$ can be represented by no more than $\lceil c \log_2(\varepsilon^{-1}) \rceil$ bits while
    $\|f - \Learn (\varepsilon,f)\|_{L^2(\Omega)}\le \varepsilon$ and $\mathcal{M}(\Learn (\varepsilon,f)) \in \mathcal {O}(\varepsilon^{-1/\gamma}), \varepsilon \rightarrow 0$.
    
    %following statements hold:
    %\begin{itemize}
     %   \item[(i)] there exist $m,k \in \N$ such that each weight of the network $\Learn (\varepsilon,f)$ is an element of $\varepsilon^m \Z \cap [-\varepsilon^{-k},\varepsilon^{-k}]$,
      %  \item[(ii)] the error bound $\|f - \Learn (\varepsilon,f)\|_{L^2(\Omega)}\le \varepsilon
      %  $ holds true, and
      %  \item[(iii)] the neural network $\Learn (\varepsilon,f)$ has at most $c\varepsilon^{-1/\gamma}$ edges.
   % \end{itemize}
\end{theorem}
\begin{remark}
Theorem~\ref{theo:EncodeOfNeuralNetworks} implies that if $\mathcal{D}$ optimally represents the function class $\cC$ in the sense of Definition~\ref{def:repopti} and at the same time is effectively representable by neural networks, then $\cC$ is optimally representable by neural networks in the sense of Definition~\ref{def:NNopti}.
\end{remark}

\begin{proof}[Proof of Theorem~\ref{theo:EncodeOfNeuralNetworks}]

    Let $M\in \mathbb{N}$ and $\gamma <\gamma^{\ast, \text{eff}}(\cC,\mathcal{D})$. 
    According to Definition~\ref{def:polydepth}, there exist constants $C,D>0$ and a polynomial $\pi$ such that for every $f\in \cC$, there is a subset $I_M\subset \{1,\dots , \pi(M)\}$, 
    and coefficients $(c_i)_{i\in I_M}$ with $\max_{i\in I_M} \! |c_i|\le D$ so that 
    \begin{equation}
         \left\|f - \sum_{i \in I_M} c_i \varphi_i\right\|_{L^2(\Omega)} \le \frac{C M^{-\gamma}}{2} =: \frac{\delta_{M}}{2}.
    \end{equation}
    We only need to consider the case $\delta_{M} \leq 1/2$ as will become clear below.
    %, by (\ref{eq:ErrorBoundedByEpsM}), we only need to deal with that case.
By effective representability according to Definition~\ref{def:wellrep}, there are $L, R\in\mathbb{N}$ such that for each $i \in I_M$ and with $\eta:= \delta_{M} / \max\{1, 4 \sum_{i\in I_M}|c_i| \}$, there exists a neural network $\Phi_{i,\eta} \in \cNN_{L, R, d, \rho}$ (with $\rho$ either Lipschitz-continuous or differentiable such that $\rho'$ is dominated by an arbitrary polynomial) satisfying
        \begin{align*} 
        \left \|\varphi_i - \Phi_{i,\eta}\right\|_{L^2(\Omega)} \le \eta.
        \end{align*}
        In addition, the weights of $\Phi_{i,\eta}$ are polynomially bounded in $i, \eta^{-1}$.
    Let then $\Phi(f,M) \in \cNN_{L,RM,d,\rho}$ be the neural network consisting of the networks $(\Phi_{i,\eta})_{i \in I_M}$ operating in parallel, according to \eqref{eq:NNadd}. 
    We conclude that 
$$
   \left\|\sum_{i \in I_M} c_i \varphi_i -   \Phi(f,M)\right\|_{L^2(\Omega)} \leq \frac{\delta_M}{4}.
$$
%        
    %Next, note that by effective representability of $\cC$ in $\mathcal{D}$, the weights $(c_i)_{i\in I_M}$ are uniformly bounded, independently of $f\in \cC$. 
    As the weights of the networks $\Phi_{i,\eta}$ are polynomially bounded in $i, \eta^{-1}$ and $i\le \pi(M), \delta_{M} \sim M^{-\gamma}$, it follows that the weights of $\Phi(f,M)$ are polynomially bounded in $\delta_M^{-1}$.
    $$
        \left\|\Phi(f,M) - \widetilde{\Phi}(f,M)\right\|_{L^2(\Omega)} \leq \frac{\delta_M}{4},
    $$
    and all weights of $\widetilde{\Phi}(f,M)$ can be represented with no more than $\lceil c \log_{2}(\delta_M^{-1})\rceil$ bits, for some $c>0$.
    %    are elements of $(\delta/4)^m \Z \cap [-(\delta/4)^{-k},(\delta/4)^{-k}]$.
    Moreover, we have
    \begin{align} \nonumber
            \left\|f - \widetilde{\Phi}(f,M)\right\|_{L^2(\Omega)} \leq \left\|f - \sum_{i \in I_M} c_i \varphi_i \right\|_{L^2(\Omega)} &+ \left\|\sum_{i \in I_M} c_i \varphi_i - \Phi(f,M) \right\|_{L^2(\Omega)}\\
            & \quad +  \left\|\Phi(f,M) - \widetilde{\Phi}(f,M)\right\|_{L^2(\Omega)} \leq \delta_M  = CM^{-\gamma}. \label{eq:ErrorBoundedByEpsM}
    \end{align}
    
    For $\varepsilon \in (0, 1/2)$, we now set
    $$
        \Learn (\varepsilon,f):= \widetilde{\Phi}(f,M_\varepsilon),
    $$
    where 
    \begin{align} \label{eq:DefOfLearnForSmallEps}
        M_\varepsilon := \left\lceil \left(\frac{C}{\varepsilon}\right)^{\frac{1}{\gamma}} \right \rceil.
    \end{align}
    With this choice of $M_\varepsilon$, we have $CM_\varepsilon^{-\gamma} \leq \varepsilon$, which, when used in \eqref{eq:ErrorBoundedByEpsM}, yields
    \begin{align}\label{eq:TheConstructionOfLearnForSmallEpsilon}
        \left\|f - \Learn (\varepsilon,f)\right\|_{L^2(\Omega)} \leq  \varepsilon.
    \end{align}
    Since, by construction, $\Learn (\varepsilon,f)$ has $R M_\varepsilon$ edges and $M_\varepsilon\leq C^{1/\gamma} \varepsilon^{-1/\gamma}+1 \leq 2 C^{1/{\gamma}} \varepsilon^{-1/{\gamma}}$, it follows that $\Learn (\varepsilon,f)$ has at most $2RC^{1/{\gamma}} \varepsilon^{-1/{\gamma}}$ edges. 
    Moreover, as all weights of $\Learn (\varepsilon,f)$ can be represented by no more than $\lceil c\log_2(\delta^{-1}_{M_\varepsilon}) \rceil$ bits, it follows from 
    $\delta_{M_{\varepsilon}} \sim M_{\varepsilon}^{-\gamma}\sim \varepsilon$ that they can be represented by no more than     
    $\lceil c'\log_2(\varepsilon^{-1})\rceil$ bits, for some $c'>0$. This concludes the proof.
\end{proof}

%
%------------------------------------------------------------------------------------------------------------------------------
\section{All Affine Representation Systems are Effectively Representable by Neural Networks}\label{sec:optimalapprox}
%------------------------------------------------------------------------------------------------------------------------------
%
This section shows that a large class of representation systems, namely \emph{affine systems}, defined below, are effectively representable by neural networks. Affine systems include as special cases wavelets, ridgelets, curvelets, shearlets, $\alpha$-shearlets, and more generally $\alpha$-molecules. Combined with Theorem~\ref{theo:EncodeOfNeuralNetworks} the results in this section establish that any function class that is optimally represented by an arbitrary affine system is optimally represented by neural networks in the sense of Definition~\ref{def:NNopti}.

Clearly, such strong statements are possible only under restrictions on the choice of the activation function for the approximating neural networks.

\subsection{Choice of Activation Function}\label{subsec:rects}

We consider two classes of activation functions, namely sigmoidal functions and smooth approximations of rectified linear units. We start with the formal definition of sigmoidal activation functions as considered in \cite{Cybenko1989, Mhaskar1993, Mhaskar-Micchelli,ChuXM1994networksforlocApprox}.

\begin{definition}\label{def:sigmoidal}
A continuous function $\rho: \R \to \R$ is called a {\em sigmoidal function
of order $k \in \N$, $k\geq 2$}, if there exists $C>0$ such that
\[
\lim_{x \to - \infty} \frac{1}{x^k}\rho(x) = 0, \quad \lim_{x \to \infty} \frac{1}{x^k}\rho(x) = 1, \quad \mbox{and} \quad
|\rho(x)| \leq C(1+ |x|)^k, \text{ for } x \in \R.
\]
A differentiable function $\rho$ is called {\em strongly sigmoidal of order $k$}, if there exist constants $a,b,C>0$ such that
\begin{align*}
\left| \frac{1}{x^k}\rho(x)\right| \le C|x|^{-a}, \,\, \mbox{ for } x<0,\quad  \left|\frac{1}{x^k}\rho(x) - 1\right|\le Cx^{-a},\,\, \mbox{ for }x\geq 0, \quad \text{ and }\\
|\rho(x)| \leq C(1+ |x|)^k , \,\,  \left|\frac{d}{dx}\rho(x)\right|\le C |x|^b, \quad \mbox{for }x\in \R.
\end{align*}
\end{definition}
One of the most widely used activation functions is the so-called rectified linear unit (ReLU) given by $x\mapsto \max\{0, x\}$.
The second class of activation functions we consider here are smooth versions of the ReLU.
\begin{definition}\label{def:admissible}
Let $\rho: \R \to \R^+$, $\rho \in C^\infty(\R)$ satisfy
\[
 \rho(x) =
 \left\{ \begin{array}{rcl} 0, &  \text{ for }x \leq 0,\\ x, &  \text{ for } x \geq K, \end{array} \right.
\]
for some constant $K > 0$. Then, we call $\rho$ an \emph{admissible smooth activation function}.
\end{definition}
The reason for considering these two specific classes of activation functions resides in the fact that neural networks based thereon allow economical representations of multivariate bump functions, which, in turn, leads to effective representation of all affine systems (built from bump functions) by neural networks. Approximation of multivariate bump functions using sparsely connected neural networks is a classical topic in neural network theory \cite{yann1987modeles}. What is new here is the aspect of quantized weights and rate-distortion optimality.

A class of bump functions of particular importance in wavelet theory are $B$-splines.
In \cite{ChuXM1994networksforlocApprox} it was shown that $B$-splines can be parsimoniously approximated by neural networks with sigmoidal activation functions. It is instructive to recall this result. To this end, for $m\in \N$, we denote the univariate cardinal $B$-spline
of order $m \in \N$ by $N_m$, i.e., $N_1 = \chi_{[0,1]}$, where $\chi_{[0,1]}$ denotes the characteristic function of the interval ${[0,1]}$, and $N_{m+1} = N_{m} * \chi_{[0,1]}$, for all $m \ge 1$. Multivariate $B$-splines are simply tensor products of univariate $B$-splines. Specifically, we denote, for $d\in \mathbb{N}$, the $d$-dimensional cardinal $B$-spline of order $m$ by $N_m^d$.

\begin{theorem}[\cite{ChuXM1994networksforlocApprox}, Thm.~4.2]\label{thm:ApproxWithSplines}
Let $d,m,k\in \N$, and take $\rho$ to be a sigmoidal function of order $k\geq 2$. Further, let $L:= \lceil\log_2(md - d)/ \log_2(k) \rceil+1$. Then, there is $M\in \N$, possibly dependent on $d,m,k$, such that for all $D,\varepsilon>0$, there exists a neural network $\Phi_{D,\varepsilon} \in \cNN_{L,M,d,\rho}$ with
\begin{align*}
    \|N^d_{m} - \Phi_{D,\varepsilon}\|_{L^2([-D,D]^d)} \leq \varepsilon.
\end{align*}
\end{theorem}
Additionally, we will need to control the weights in the approximating networks $\Phi_{D,\varepsilon}$. We next show that this is, indeed, possible for strongly sigmoidal activation functions.
\begin{theorem}\label{thm:EffApproxWithSplines}
Let $d,m,k\in \N$, and $\rho$ strongly sigmoidal of order $k\geq 2$. Further, let $L:= \lceil\log_2(md - d)/ \log_2(k) \rceil+1$.
Then, there is $M\in \N$, and a two-dimensional polynomial $\pi$ possibly dependent on $d ,m,k$, such that for all $D,\varepsilon>0$, there exists a neural network $\Phi_{D,\varepsilon} \in \cNN_{L,M,d,\rho}$ with
\begin{align*}
\|N^d_{m} - \Phi_{D,\varepsilon}\|_{L^2([-D,D]^d)} \leq \varepsilon.
\end{align*}
Moreover, the weights of $\Phi_{D,\varepsilon}$ are polynomially bounded in $D, \varepsilon^{-1}$.
\end{theorem}
\begin{proof}
 The neural network $\Phi_{D,\varepsilon}$ in  Theorem~\ref{thm:ApproxWithSplines} is explicitly
 constructed in \cite{ChuXM1994networksforlocApprox}. Carefully following the
 steps in that construction and making explicit use of the strong sigmoidality of $\rho$, as opposed to plain sigmoidality as in \cite{ChuXM1994networksforlocApprox}, yields the desired result.
\end{proof}

\begin{remark}
We observe that the number of edges of the approximating network in Theorem~\ref{thm:EffApproxWithSplines} does not depend on the approximation error $\varepsilon$. 
%In other words, a representation system containing only a single $B$-spline is effectively representable in the sense of Definition~\ref{def:wellrep} by neural networks with strongly
% sigmoidal activation functions.
\end{remark}

While Theorem~\ref{thm:ApproxWithSplines} demonstrates that a $B$-spline of order $m$ can be approximated to arbitrary accuracy by a neural network based on a sigmoidal activation function and of depth depending on $m,d$, and the order of sigmoidality of the activation function, we next establish that for admissible smooth activation functions, exact
representation of a general class of bump functions is possible with a network of $3$ layers only. Before proceeding, we define for $f\in L^1(\R^d)$, $d\in \N$, the \emph{Fourier transform} of $f$ by 
$$
\hat{f}(\xi) := \int_{\R^d} f(x) e^{-2\pi i \langle x, \xi \rangle } dx, \text{ for } \xi\in \R^d.
$$
\begin{theorem}\label{thm:Approxsmoothrec}
Let $\rho$ be an admissible smooth activation function. Then, for all $d\in \mathbb{N}$, there exist $M\in \N$ and a neural network $\Phi_\rho
\in \cNN_{3, M, d, \rho}$
such that
\begin{itemize}
\item[(i)] $\Phi_\rho$ is compactly supported,
\item[(ii)] $\Phi_\rho \in C^\infty(\R)$, and
\item[(iii)] $\widehat\Phi_\rho(\xi)\neq 0$, for all $\xi \in [-3,3]^d$.
\end{itemize}
\end{theorem}
\begin{proof} We start by constructing an auxiliary function as follows. For $0<p_1\leq p_2\leq p_3$ such that $p_1 + p_2  = p_3$, define $t : \R \to \R$ as
\begin{align}\label{eq:FirstLayer}
t(x) := \rho(x) - \rho(x - p_1) - \rho(x - p_2) + \rho(x - p_3), \,\, x\in \R.
\end{align}
Then, $t \in C^\infty$ is compactly supported. Letting $q = \|t\|_{L^\infty(\mathbb{R})}$, we define $g : \R^d \to \R$ according to
\begin{align} \label{eq:SecondLayer}
g(x):= \rho\left(\sum_{i = 1}^d t(x_i) - (d-1) \cdot q\right), \,\, x\in \R^d.
\end{align}
By construction, $g \in C^\infty$ is compactly supported. Moreover, $g$ can be realized through a three-layer neural network thanks to its
two-step design per \eqref{eq:FirstLayer} and \eqref{eq:SecondLayer}.
Since $g \geq 0$ and $g\neq 0$, it follows that $|\hat{g}(0)| > 0$. By continuity of $\hat{g}$ there exists a $\delta>0$ such that $|\hat{g}(\xi)|>0$ for all
$\xi \in [-\delta,\delta]^d$. We now set
\[
\varphi := g\left(3 \left(\frac{\cdot}{\delta}\right)\right),
\]
and note that $\varphi$ can be realized through a three-layer neural network $\Phi_\rho \in \cNN_{3, M, d, \rho}$, for some $M\in \N$. As 
$|\hat{\varphi}(\xi)| >0$, for all $\xi \in [-3,3]^d$, $\Phi_\rho$ satisfies the desired assumptions.
\end{proof}

\subsection{Invariance to Affine Transformations}
We next leverage Theorems \ref{thm:EffApproxWithSplines} and \ref{thm:Approxsmoothrec} to demonstrate that a wide class of representation systems built through affine transformations of $B$-splines and bump functions as constructed in Theorem~\ref{thm:Approxsmoothrec} is effectively representable by neural networks. As a first step towards this general result, we show that representability---in the sense of Definition~\ref{def:wellrep}---of a single function $f$ by neural networks  is invariant to the operation of taking finite linear combinations of affine transformations of $f$. 
\begin{proposition}\label{prop:affscalinv}
    Let $d\in \mathbb{N}$, $\rho:\mathbb{R}\to \mathbb{R}$, and $f\in L^2(\R^d)$. Assume
    that there exist $M,L\in\mathbb{N}$ such that for all $D,\varepsilon>0$, there is
    $\Phi_{D,\varepsilon}\in \cNN_{L,M,d,\rho}$ with
    \begin{equation}\label{eq:nodilapp}
        \|f - \Phi_{D,\varepsilon}\|_{L^2([-D,D]^d)} \leq \varepsilon.
    \end{equation}

    Let $A\in \mathbb{R}^{d \times d}$ be full-rank and $b\in \mathbb{R}^d$. Then, there exists $M'\in \N$, depending on $M$ and $d$ only, such that for all $E,\eta>0$, there is $\Psi_{E,\eta}\in \cNN_{L,M',d,\rho}$ with
    \begin{equation*}%\label{eq:dilapp}
        \left\||\!\det(A)|^{\frac12}f(A\cdot-\,b) - \Psi_{E,\eta}\right\|_{L^2([-E,E]^d)} \leq \eta.
    \end{equation*}
    Moreover, if the weights of $\Phi_{D,\varepsilon}$ are polynomially bounded in $D, \varepsilon^{-1}$, then
    the weights of $\Psi_{E,\eta}$ are polynomially bounded in $\|A\|_\infty, E, \|b\|_\infty,  \eta^{-1}$, where $\|A\|_\infty$ and $\|b\|_\infty$ denote the max-norm of $A$ and $b$, respectively.
\end{proposition}
\begin{proof}
    By a change of variables, we have for every $\Phi \in \cNN_{L,M,d,\rho}$ that
    \begin{align}\label{eq:weApplyThisInTheFirstEstimate}
        \left\| |\!\det(A)|^{\frac12}f(A\cdot- \,  b) - |\!\det(A)|^{\frac12}\Phi(A\cdot - \, b)\right\|_{L^2([-E,E]^d)}
        =\|f-\Phi\|_{L^2(A\cdot [-E,E]^d\, - \, b)},
    \end{align}
    and there exists $M'$ depending  on $M$ and $d$ only such that $|\!\det(A)|^{1/2}\Phi(A\cdot- \, b)\in \cNN_{L,M',d,\rho}$. We furthermore have that
   \begin{align}\label{eq:weApplyThisInTheSecondEstimate}
        A\cdot [-E,E]^d-\,b\subset \left[-(d E \|A\|_\infty +\|b\|_\infty), (d E \|A\|_\infty +\|b\|_\infty)\right]^d.
    \end{align}
    We now set $F= d E \|A\|_\infty +\|b\|_\infty$ and $\Psi_{E,\eta}:= |\!\det(A)|^{1/2}\Phi_{F,\eta}(A\cdot -\,b)$ and observe that 
    \begin{align*}
             &\left\||\!\det(A)|^{\frac12}f(A\cdot-\,b) - \Psi_{E,\eta}\right\|_{L^2([-E,E]^d)}
              =  \  \left\|f - \Phi_{F,\eta} \right\|_{L^2(A\cdot[-E,E]^d -\,b)} \leq \ \left\|f - \Phi_{F,\eta} \right\|_{L^2(\left[-F, F)\right]^d)} \leq   \eta,
    \end{align*}    
    where we applied the same reasoning as in \eqref{eq:weApplyThisInTheFirstEstimate} in the first equality, \eqref{eq:weApplyThisInTheSecondEstimate} in the first inequality, and \eqref{eq:nodilapp} in the second inequality. Moreover, we see that if the weights of $\Phi_{D,\varepsilon}$ are polynomially bounded in $D, \varepsilon^{-1}$, then the weights of $\Psi_{E,\eta}$ are polynomially bounded in $\|A\|_\infty, |\!\det(A)| , E, \|b\|_\infty,  \eta^{-1}$.
%P6: im Satz darueber hatten wir unseren Begriff von polynomially bounded nicht verwendet, brauchen wir denn das Polynom pi^* irgendwo explizit?    
%    Moreover, we see that if there exists a \pp{two-dimensional polynomial $\pi$ such that the absolute value of the weights of $\Phi_{D,\varepsilon}$ are bounded by $\pi(D, 
%\varepsilon^{-1})$, then there exists 
%    a four-dimensional polynomial $\pi^*$ such that the absolute value of the weights of $\Psi_{E,\eta}$ are bounded by 
  %  $\pi^*(\|A\|_\infty, |\!\det(A)| , E, \|b\|_\infty,  \eta^{-1})$}.
    Since $|\!\det(A)|$ is polynomially bounded in $\|A\|_\infty$, it follows that the weights of $\Psi_{E,\eta}$ are polynomially bounded in $|\|A\|_\infty, E, \|b\|_\infty,  \eta^{-1}$. This yields the claim.
\end{proof}
Next, we show that representability by neural networks is preserved under finite linear combinations
of translates.
\begin{proposition}\label{prop:transapp}
Let $d\in \mathbb{N}$, $\rho:\mathbb{R}\to \mathbb{R}$, and $f\in L^2(\R^d)$. Assume that there exist $M,L\in \mathbb{N}$ such that for all $D,\varepsilon>0$, there is
    $\Phi_{D,\varepsilon}\in \cNN_{L,M,d,\rho}$ with
    \begin{equation}\label{eq:notransapp}
        \|f - \Phi_{D,\varepsilon}\|_{L^2([-D,D]^d)} \leq \varepsilon.
    \end{equation}

    Let $r\in \mathbb{N}$, $(c_i)_{i=1}^r\subset \mathbb{R}$, and $(d_i)_{i=1}^r\subset \mathbb{R}^d$. Then, there exists $M'\in \N$, depending on $M,d$, and $r$ only, such that for all $E,\eta>0$, there is $\Psi_{E,\eta}\in \cNN_{L,M',d,\rho}$ with
    \begin{equation}\label{eq:transapp}
        \left\|\sum_{i=1}^rc_if(\cdot - d_i) - \Psi_{E,\eta}\right\|_{L^2([-E,E]^d)} \leq \eta.
    \end{equation}
    Moreover, if the weights of $\Phi_{D,\varepsilon}$ are polynomially bounded in $D,\varepsilon^{-1}$, 
    then the weights of $\Psi_{E,\eta}$ are polynomially bounded in 
    $$
      \sum_{i=1}^r|c_i|, E, \max_{i=1, \dots, r} \|d_i\|_\infty, \eta^{-1}.
    $$
\end{proposition}
\begin{proof}
Let $E, \eta>0$. We start by noting that, for all $D,\varepsilon >0$,
$$
    \left\|\sum_{i=1}^rc_if(\cdot - d_i) - \sum_{i=1}^rc_i \Phi_{D,\varepsilon}(\cdot - d_i) \right\|_{L^2([-E,E]^d)}
    \le \left(\sum_{i=1}^r|c_i|\right)\cdot \|f-\Phi_{D,\varepsilon}\|_{L^2([-(E + d^*),(E + d^*)]^d)},
$$
where $d^* = \max_{i=1, \dots, r} \|d_i\|_\infty$. Setting $D = E + d^*$ and $\varepsilon = \eta/\max\{1 , \sum_{i=1}^r|c_i|\}$, 
and noting that for every $\Phi\in \cNN_{L,M,d,\rho}$, the function
 $$
     \Psi:=\sum_{i=1}^r c_i \Phi(\cdot - d_i)
 $$
is in $\cNN_{L,M',d,\rho}$ with $M'\in \N$ depending on $d,r$, and $M$ only, it follows that the network
$$
    \Psi_{E,\eta}:=\sum_{i=1}^r c_i \Phi_{D,\varepsilon}(\cdot - d_i)
$$
%Setting $d^* = \max_{i=1}^r \|d_i\|_\infty$, %
%
satisfies \eqref{eq:transapp}. 
%P6: auch den folgenden Satz habe ich umformuliert unter Verwendung unserer Def. von polynomially bounded.
Finally, if the weights of $\Phi_{D,\varepsilon}$ are polynomially bounded in $D,\varepsilon^{-1}$, then the weights of $\Psi_{E,\eta}$ are polynomially bounded 
in $\sum_{i=1}^r|c_i|, E, d^*, \eta^{-1}$.

%\pp{Finally, if there is a two-dimensional polynomial $\pi$ such that the absolute values of the weights of $\Phi_{D,\varepsilon}$ are bounded by $\pi(D,\varepsilon^{-1})$, then we %observe that there  is a four-dimensional polynomial $\widetilde{\pi}$ such that the absolute values of the weights of $\Psi_{E,\eta}$ are bounded 
%by $\widetilde{\pi}( \sum_{i=1}^r|c_i|, E, d^*, \eta^{-1})$}.
\end{proof}
Based on the invariance results in Propositions \ref{prop:affscalinv} and \ref{prop:transapp}, we now construct neural networks which approximate functions with a given number of vanishing moments with arbitrary accuracy. The resulting construction will be crucial in establishing representability of affine representation systems (see Definition~\ref{def:affsys}) by neural networks.
\begin{definition}
    \label{def:vanishingmoments}Let $R,d\in \mathbb{N}$, and $k\in \{1,\dots , d\}$.
    A function $g\in C(\R^d)$ is said to possess \emph{$R$ directional vanishing moments}\/ in $x_k$-direction,
    if
    $$
        \int_{\mathbb{R}}x_k^\ell g(x_1,\dots , x_k,\dots  , x_d)dx_k = 0,\quad \mbox{for all }x_1,\dots , x_{k-1},x_{k+1},\dots , x_d\in \R,\ \ell\in \{0,\dots , R-1\}.
    $$
\end{definition}
The next result establishes that functions with an arbitrary number of vanishing moments in a given coordinate direction can be built from suitable linear combinations of translates of a given continuous function with compact support.
\begin{lemma}\label{lem:generatevanishingmoments}
    Let $R,d\in \mathbb{N}$, $B>0$, $k\in \{1,\dots , d\}$, and $f\in C(\R^d)$ with compact support.
    Then, the function
    \begin{equation}\label{eq:vanmom}
        g(x_1,\dots , x_d):=\sum_{\ell=0}^{R-1}\binom{R-1}{\ell}(-1)^\ell f\left(x_1,\dots , x_k - \frac{\ell}{B},\dots , x_d\right)
    \end{equation}
    has $R$ directional vanishing moments in $x_k$-direction. Moreover, if $\hat f(\xi)\neq 0$ for all $\xi \in [-B,B]^d \setminus \{0\}$, then
    \begin{equation}\label{eq:nozero}
        \hat g(\xi)\neq 0, \quad \mbox{for all }\xi\in [-B,B]^d\mbox{ with } \xi_k \neq 0.
    \end{equation}
\end{lemma}
\begin{proof}For simplicity of exposition, we consider the case $B=1$ only.
    Taking the Fourier transform of (\ref{eq:vanmom}) yields
    \begin{align}
    \label{eq:fouriermoments}
        \hat{g}(\xi)=\sum_{\ell=0}^{R-1}\binom{R-1}{\ell}(-1)^\ell e^{-2\pi i \ell\xi_k}\hat f(\xi)
        = \left(1-e^{-2\pi i \xi_k}\right)^{R-1}\cdot \hat f(\xi)
    \end{align}
    which implies
    $$
        \left(\frac{\partial^{\, \ell}}{\partial \xi_k^\ell}\hat g \right)_{\xi_k = 0} = 0, \quad \mbox{for all } \ell\in \{0,\dots , R-1\}.
    $$
    But by Definition~\ref{def:vanishingmoments}, this says precisely that $g$ possesses
    the desired vanishing moments.
    Statement (\ref{eq:nozero}) follows by inspection of (\ref{eq:fouriermoments}).
\end{proof}

\subsection{Affine Representation Systems}
We are now ready to introduce the general family of representation systems announced earlier in the paper as \emph{affine systems}. This class 
%as In this section, we introduce a family of representation systems, coined \emph{affine systems}, that 
includes all representation systems based on affine transformations of a given ``mother function''. Special cases of affine systems are wavelets, ridgelets, curvelets, shearlets, $\alpha$-shearlets, and more generally $\alpha$-molecules, as well as tensor products thereof. The formal definition of affine systems is as follows.
\begin{definition}\label{def:affsys}
    Let $d, r, S\in \mathbb{N}$, $\Omega\subset \mathbb{R}^d$ be bounded, and $f\in L^2(\R^d)$ compactly supported.
    Let $\delta >0$, $(c_i^s)_{i=1}^r \subset \mathbb{R}$, for $s=1,\dots , S$, and
    $(d_i)_{i=1}^r\subset \mathbb{R}^d$. Further, let $A_j \in \R^{d \times d}, j \in \N$, be full-rank, with the absolute values of the eigenvalues of $A_j$ bounded below by $1$.     
    Consider the compactly supported functions
    $$
        g_s:=\sum_{i=1}^r c_i^s f(\cdot - d_i), \quad s = 1,\dots , S.
    $$
    We define the \emph{affine system} $\mathcal{D}\subset L^2(\Omega)$ corresponding to $(g_s)_{s = 1}^S$ according to
    $$
        \mathcal{D}:=\left\{g_s^{j,b} := \left( | \! \det(A_j)|^{\frac{1}{2}}g_s(A_j\cdot - \, \delta \cdot b)\right)_{|\Omega}:\ s=1,\dots , S,\ b\in \mathbb{Z}^d, \ j\in \mathbb{N}, \mbox{ and } g_s^{j,b} \neq 0\right\},
    $$
    and refer to $f$ as the \emph{generator function of $\mathcal{D}$.}
    \end{definition}
We define the sub-systems $\mathcal{D}_{s,j} := \{g_s^{j,b} \in \mathcal{D}:\ b\in \mathbb{Z}^d \}$. Since every $g_s$, $s = 1, \dots, S,$ has compact support, $|\mathcal{D}_{s,j}|$ is finite for all $s= 1, \dots, S$ and $j\in \N$. 
Indeed, we observe that there exists $c_{\textrm{b}} := c_{\textrm{b}}((g_s)_{s = 1}^S, \delta, d ) > 0$ such that for all $s\in \{1, \dots, S \}$, $j\in \Z$, and $b \in \Z^d$,
\begin{align}\label{eq:boundOnB}
g_s^{j,b} \in \mathcal{D} \implies \|b\|_\infty \leq c_b \|A_j\|_\infty.
\end{align}
As the $\mathcal{D}_{s,j}$ are finite, we can organize the representation system $\mathcal{D}$ according to
\begin{equation}
\label{eq:canonicalordering}
    \mathcal{D} = (\varphi_i)_{i\in \mathbb{N}}=\left(\mathcal{D}_{1,1},\dots , \mathcal{D}_{S,1},\mathcal{D}_{1,2},\dots , \mathcal{D}_{S,2},\dots \right),
\end{equation}
where the elements within each sub-system $\mathcal{D}_{s,j}$ may be ordered arbitrarily.
This ordering of $\mathcal{D}$ is assumed in the remainder of the paper and will be referred to as \emph{canonical ordering}. 

Moreover, we note that if there exists $s_o\in \{1, \dots, S\}$ such that $g_{s_o}$ is nonzero, then there \pp{is} a constant $c_{\textrm{o}} := c_{\textrm{o}}((g_s)_{s = 1}^S, \delta, d) > 0$ such that
\begin{align}\label{eq:WeNeedThisAssumption}
    \sum_{s = 1}^S|\mathcal{D}_{s,j}| \geq c_{\textrm{o}}|\! \det(A_j)|, \text{ for all } j\in \N.
\end{align}
The next result establishes that all affine systems whose generator functions can be approximated to within arbitrary accuracy by neural networks are (effectively) representable by neural networks.
\begin{theorem}\label{thm:affdicopt} Let $d\in \N$, $\rho:\mathbb{R}\to \mathbb{R}$, $\Omega\subset \mathbb{R}^d$ be bounded, and $\mathcal{D}=(\varphi_i)_{i\in \mathbb{N}}\subset L^2(\Omega)$ an affine system with generator function $f$. Suppose that there exist constants $L, R \in \N$ such that for all $D,\varepsilon>0$,  there is
    $\Phi_{D,\varepsilon}\in \cNN_{L,R,d,\rho}$ with
    \begin{equation}\label{eq:bootstrapbound}
        \|f - \Phi_{D,\varepsilon}\|_{L^2([-D,D]^d)} \leq \varepsilon.
    \end{equation}
    Then, $\mathcal{D}$ is representable by neural networks with activation function $\rho$.
    If, in addition, the weights of $\Phi_{D,\varepsilon}$ are polynomially bounded in $D, \varepsilon^{-1}$, and if there exist $a>0$ and $c>0$ such that 
    \begin{equation}\label{eq:detgrowth}
        \sum_{k=1}^{j-1}|\det(A_k)| \ge c \|A_j\|_\infty^a, \,\, \text{ for all } j \in \N, 
    \end{equation}
    then $\mathcal{D}$ is effectively representable by neural networks with activation function $\rho$.
\end{theorem}
\begin{proof}
    Let $(g_s)_{s = 1}^S$ be as in Definition~\ref{def:affsys}.
If $g_s = 0$ for all $s \in \{1, \dots, S\}$, then $\mathcal{D} = \mathlarger{\mathlarger{\varnothing}}$ and the result is trivial. 
Hence, we can assume that there exists at least one $s \in \{1, \dots, S\}$ such that $g_s\neq 0$, implying that \eqref{eq:WeNeedThisAssumption} holds.

    Pick $D$ such that $\Omega\subset [-D,D]^d$. We first show that (\ref{eq:bootstrapbound}) implies representability of $\mathcal{D}$ by neural networks with activation function $\rho$. To this end, we need to establish the existence of constants $L, R \in \N$ such that for all $i\in \mathbb{N}$ and all $\eta>0$, there exist $\Phi_{i,\eta}\in \cNN_{L, R, d, \rho}$
    with
    \begin{align}
    \label{eq:approxOfDictElements}
           \|\varphi_i- \Phi_{i,\eta}\|_{L^2(\Omega)}\le \eta.
    \end{align}
    The elements of $\mathcal{D}$ consist of dilations and translations of $f$ according to
    \begin{equation}  \label{eq:affsysex}
    \varphi_i = |\!\det(A_{j_i})|^{\frac{1}{2}}\left(\sum_{k=1}^r c_k^{s_i} f(A_{j_i}\cdot-\, \delta b_i - d_k)\right)_{|\Omega},
    \end{equation}
for some $r\in \N$ independent of $i$, and $s_i\in \{1,\dots , S\}$, $j_i\in \mathbb{N}$, and $b_i\in \mathbb{Z}^d$. Thus \eqref{eq:approxOfDictElements} follows directly by Propositions \ref{prop:affscalinv} and \ref{prop:transapp}.

    It remains to show that the weights of $\Phi_{D,\varepsilon}$ in (\ref{eq:bootstrapbound}) polynomially bounded in $D, \varepsilon^{-1}$ implies that $\mathcal{D}$ is effectively representable by neural networks with activation function $\rho$, which, by Definition~\ref{def:wellrep}, means that the weights of $\Phi_{i,\eta}$ are polynomially bounded in $i,\eta^{-1}$.
Propositions \ref{prop:affscalinv} and \ref{prop:transapp} state that the weights of $\Phi_{i,\eta}$ are polynomially bounded in
    $$
        \|A_{j_i}\|_\infty, D, \|b_i\|_\infty, \sum_{k=1}^r |c_k|, \max_{k=1, \dots, r} \|d_k\|_\infty,  \eta^{-1}.
    $$
    Thanks to \eqref{eq:boundOnB} we have $\|b_i\|_\infty \in \mathcal{O}(\|A_{j_i}\|_\infty)$.
    Moreover, the quantities $D$, $\sum_{k=1}^r|c_k|$, and $\max_{k=1,\dots,r} \|d_k\|_\infty$ do not depend on $i$. 
    We can thus conclude that the weights of $\Phi_{i,\eta}$ are polynomially bounded in
    \begin{equation}
    \label{eq:polydetbd}
        \|A_{j_i}\|_\infty, \eta^{-1}.
    \end{equation}
To complete the proof, we need to show that the quantities $\|A_{j_i}\|_\infty$ are polynomially bounded in $i$. To this end, we first observe that $\varphi_i$ according to (\ref{eq:affsysex})
    satisfies $\varphi_i\in \mathcal{D}_{s_i,j_i}$ for some $s_i\in \{1,\dots , S\}$. Thanks to (\ref{eq:WeNeedThisAssumption}) and the canonical ordering (\ref{eq:canonicalordering}), there exists a constant $c>0$ such that
    $$
        i \ge c \sum_{k=1}^{j_i-1} |\det(A_k)|.
    $$
    We finally appeal to (\ref{eq:detgrowth}) to conclude that $\|A_{j_i}\|_\infty$ is polynomially bounded in $i$,
    which, together with (\ref{eq:polydetbd}), establishes the desired result.
\end{proof}
We remark that condition (\ref{eq:detgrowth}) is very weak; in fact, we are not aware of an affine system in the literature that would violate it.

We now proceed to what is probably the central result of this paper, namely that neural networks provide optimal approximations for all function classes that are optimally approximated by any affine system with generator function that can be approximated to within arbitrary accuracy by neural networks.
\begin{theorem}
    \label{thm:optitransgeneral} Let $d\in \N$, $\Omega\subset \mathbb{R}^d$ be bounded, $\rho:\mathbb{R}\to \mathbb{R}$, and $\mathcal{D}=(\varphi_i)_{i\in \mathbb{N}}\subset L^2(\Omega)$ an affine system with generator function $f$. Assume that there exist $L, R\in \N$ such that for all $D, \varepsilon>0$, there is
    $\Phi_{D,\varepsilon}\in \cNN_{L,R,d,\rho}$ satisfying $\| f -\Phi_{D,\varepsilon} \|_{L^2([-D,D]^{d})} \leq \varepsilon$.
        Then, for all function classes $\cC\subset L^2(\Omega)$, we have
    $$
        \gamma_{\cNN}^\ast(\cC, \rho)\geq \gamma^\ast(\cC,\mathcal{D}).
    $$
    If, in addition, there is a two-dimensional polynomial $\widetilde{\pi}$ such that the weights of $\Phi_{D,\varepsilon}$ are bounded by $|\widetilde{\pi}(D,\varepsilon^{-1})|$, there exist $a > 0$ and $c > 0$ such that \eqref{eq:detgrowth} holds, and $\mathcal{C}$ is optimally represented by $\mathcal{D}$ (according to Definition~\ref{def:repopti}), then 
    for all $\gamma <\gamma^{\ast}(\mathcal{C})$, there exist a constant $c > 0$, a polynomial $\pi$, and
    % $c',C > 0$ and 
    a map
    $$
    \Learn: \left(\!0,\frac{1}{2}\right) \times L^2(\Omega) \to \cNN_{L, \infty, d, \rho}^{\pi} ,
    $$
    such that for every $f\in \cC$ the weights in $\Learn (\varepsilon,f)$ can be represented by no more than $\lceil c\log_2(\varepsilon^{-1})\rceil$ bits while
   $\|f - \Learn (\varepsilon,f)\|_{L^2(\Omega)}\le \varepsilon$ and $\mathcal{M}(\Learn (\varepsilon,f)) \in \mathcal{O}(\varepsilon^{-1/\gamma}), \varepsilon \to 0$.
\end{theorem}
\begin{proof}
    The proof follows directly by combining Theorem~\ref{thm:affdicopt} with Theorems \ref{theo:ApproxOfNeuralNetworks} and \ref{theo:EncodeOfNeuralNetworks}.
\end{proof}
Theorem~\ref{thm:optitransgeneral} reveals a remarkable universality and optimality property of neural networks: All function classes that can be optimally represented by an affine system with generator $f$ satisfying \eqref{eq:bootstrapbound} are also optimally representable by neural networks. 

%------------------------------------------------------------------------------------------------------------------------------
\section{$\alpha$-Shearlets and Cartoon-Like Functions}\label{sec:alphacart}
%------------------------------------------------------------------------------------------------------------------------------

We next present an explicit pair $(\cC, \mathcal{D})$ of function class and representation system satisfying $\gamma_{\cNN}^\ast(\cC, \rho) =  \gamma^\ast(\cC,\mathcal{D})$.
Specifically, we take $\alpha$-shearlets as representation system $\mathcal{D} \subset L^2(\R^2)$ and $\alpha^{-1}$-cartoon-like functions as function class $\cC$. 
%P6: with two pieces "separated by" oder jedes der beiden pieces hat eine smooth boundary?
Cartoon-like functions are piecewise smooth functions with only two pieces. These pieces are separated by a smooth interface. In a sense, they can be understood as a prototype of a two-dimensional classification function with two homogeneous areas corresponding to two classes. Understanding neural network approximation of this function class is hence relevant to classification tasks in machine learning.
We point out that the definition of $\alpha$-shearlets in this paper differs slightly
from that in \cite{GroKKS2016alphaMolecules}. Concretely, relative to \cite{GroKKS2016alphaMolecules} our definition replaces 
$\alpha^{-1}$ by $\alpha$ so that $\alpha$-shearlets are a special case of $\alpha$-molecules, whereas in \cite{GroKKS2016alphaMolecules} $\alpha$-shearlets are a special case of $\alpha^{-1}$-molecules. We will need dilation and shearing matrices defined as
\begin{align*}
D_{\alpha, a} := \left(\begin{array}{l l}
a & 0 \\
0 & a^\alpha
\end{array} \right),\quad \quad
    J := \left(\begin{array}{l l}
        0 & 1 \\
        1 & 0
    \end{array} \right),\quad \text{and} \quad S_{k} := \left(\begin{array}{l l}
        1 & k \\
        0 & 1
    \end{array} \right).
\end{align*}
This leads us to the following definition which is a slightly modified version of the corresponding definition in \cite{Voi17}.

\begin{definition}[\cite{Voi17}] \label{def:alphashearlet}
For $\delta\in \R^+$, $\alpha \in [0,1]$, and $f,g\in L^2(\mathbb{R}^2)$, the \emph{cone-adapted $\alpha$-shearlet system} $\mathcal{SH}_{\alpha}(f,g, \delta)$
generated by $f,g \in L^2(\R^2)$ is defined as
\begin{align*}
    \mathcal{SH}_{\alpha}(f,g,\delta):= \mathcal{SH}_{\alpha}^0(f,g,\delta)\cup  \mathcal{SH}_{\alpha}^1(f,g,\delta),
\end{align*}
where
\begin{align*}
    \mathcal{SH}_{\alpha}^0(f,g,\delta) :&= \left\{  f(\cdot - \delta t):  t\in \Z^2\right\}\!, \\
    \mathcal{SH}_{\alpha}^1(f,g,\delta) :&= \left\{2^{\ell (1 + \alpha)/2} g(S_k D_{\alpha, 2^{\ell}}J^\tau  \cdot - \, \delta t ): \, \ell\in \mathbb{N}_0,\, |k| \leq \lceil 2^{\ell(1-\alpha)} \rceil,\,
    t \in \Z^2\!,\, k \in \Z , \ \tau \in \{0,1\} \right\}\!.
\end{align*}
\end{definition}
Our interest in $\alpha$-shearlets stems from the fact that they optimally represent $\alpha^{-1}$-cartoon-like functions in the sense of Definition~\ref{def:repopti}.

\begin{definition}
Let $\beta \in [1,2)$, and $\nu >0$. Define
\begin{align*}
    \mathcal{E}^{\beta}(\R^2; \nu) = \{ f\in L^2(\R^2) : f = f_0 + \chi_B f_1\},
\end{align*}
where $f_0, f_1 \in C^\beta(\R^2)$, $\suppp f_0, \suppp f_1 \subset (0,1)^2$, $B \subset [0,1]^2$, $\partial B \in C^\beta$, $\|f_1\|_{C^\beta},\|f_2\|_{C^\beta}, \|\partial B\|_{C^\beta} < \nu$, and $\chi_B$ denotes the characteristic function of $B$. The elements of $\mathcal{E}^{\beta}(\R^2; \nu)$ are called \emph{$\beta$-cartoon-like functions}.
\end{definition}

This function class was originally introduced in \cite{Don2001Sparse} as a model class for functions governed by curvilinear discontinuities of prescribed regularity. In this sense, $\beta$-cartoon-like functions provide a convenient model
for images governed by edges or for the solutions of transport equations which often exibit curvilinear singularities.

The optimal exponent $\gamma^\ast(\mathcal{E}^{\beta}(\R^2; \nu))$ was found in \cite{Don2001Sparse,Grohs2016}:

\begin{theorem}\label{thm:ExponentOfCartoons}
For $\beta \in [1,2]$, and $\nu>0$, we have
\[
\gamma^*(\mathcal{E}^{\beta}(\R^2; \nu)) = \frac{\beta}{2}.
\]
\end{theorem}

\begin{proof}
The proof of \cite[Theorem 2]{Don2001Sparse} demonstrates that a general function class 
$\cC$
has optimal exponent 
$\gamma^*(\cC) = {(2-p)}/{2p}$
if $\cC$ contains a copy of $\ell^p_0$. The result now follows, since by \cite{Grohs2016}, the function class $\mathcal{E}^{\beta}(\R^2; \nu)$ does, indeed, contain a copy of $\ell^p_0$ for $p =  {2}/{(\beta+1)}$.
\end{proof}

Using Proposition \ref{prop:optimalitynoquant}, this result allows to conclude that neural networks achieving uniform approximation error $\varepsilon$ over the class $\cC$ of cartoon-like functions, with weights represented by no more than $\lceil c \log_{2}(\varepsilon^{-1}) \rceil$ bits, for some constant $c > 0$, yield an effective best $M$-edge approximation rate of at most $\beta/2$. 
Theorem~\ref{thm:mainopti} below demonstrates achievability for $\beta=1/\alpha$, with $\alpha \in [1/2,1]$.

The following theorem states that $\alpha$-shearlets yield optimal best $M$-term approximation rates for $\alpha^{-1}$-cartoon-like functions.

\begin{theorem}[\cite{Voi17}, Theorem 6.3 and Remark 6.4]\label{thm:FelixThm}
Let $\alpha \in [1/2, 1]$, $\nu>0$, $f \in C^{12}(\R^2)$, $g \in C^{32}(\R^2)$, both compactly supported and such that
\begin{itemize}
\item[(i)] $\widehat{f}(\xi) \neq 0$, \, for all $|\xi| \leq 1$,
\item[(ii)] $\widehat{g(\xi)} \neq 0$, \, for all $\xi=(\xi_1,\xi_2)^T\in \R^2$ such that $1/3\leq |\xi_1| \leq 3$ and $|\xi_2|\le |\xi_1|$,
\item[(iii)] $g$ has at least $7$ vanishing moments in $x_1$-direction, i.e.,
$$
    \int_\mathbb{R} x_1^\ell g(x_1,x_2)dx_1 = 0, \quad \mbox{for all }x_2\in \R,\ \ell\in \{0,\dots , 6\}.
$$
\end{itemize}
Then, there exists $\delta^\ast>0$ such that for all $\delta<\delta^\ast$, the function class $\mathcal{E}^{1/\alpha}(\R^2; \nu)$ is optimally represented by $\mathcal{SH}_{\alpha}(f,g,\delta)$.
\end{theorem}

\begin{remark}
The assumptions on the smoothness and the number of vanishing moments of $f$ and $g$ in Theorem~\ref{thm:FelixThm} follow from \cite[Eq. 4.9]{Voi17} with $s_1= 3/2, s_0 = 0, p_0=q_0 = 2/3,$ and $|\beta| \leq 4$. While these particular choices allow the statement of the theorem to be independent of $\alpha$, it is possible to weaken the assumptions, if a fixed $\alpha$ is considered. For example, for $\alpha = 1/2$ the smoothness assumptions on $f$ and $g$ reduce to $f \in C^{11}, g\in C^{28}$.
\end{remark}

As our approximation results for neural networks pertain to bounded domains, we require a definition of cartoon-like functions on bounded domains.
\begin{definition}
Let $(0,1)^2 \subset \Omega \subset \R^2$, $\alpha \in [1/2, 1]$, and $\nu>0$. We define the set of \emph{$\alpha^{-1}$-cartoon-like functions on $\Omega$} by
$$
\mathcal{E}^{\frac{1}{\alpha}}(\Omega; \nu): =  \left\{ f_{|\Omega}: f\in \mathcal{E}^{\frac{1}{\alpha}}(\R^2; \nu)\right\}.
$$
Additionally, for $\delta>0$, $f,g \in L^2(\R^2)$, we define an \emph{$\alpha$-shearlet system on $\Omega$} according to 
$$
\mathcal{SH}_{\alpha}(f,g,\delta;\Omega):= \left\{ \phi_{|\Omega}: \phi\in \mathcal{SH}_{\alpha}(f,g,\delta)\right\}.
$$
\end{definition}
\begin{remark}\label{rem:BdDomain}
It is straightforward to check, that if $\mathcal{E}^{1/\alpha}(\R^2; \nu)$ is optimally represented by $\mathcal{SH}_{\alpha}(f,g,\delta)$, then $\mathcal{E}^{1/\alpha}(\Omega; \nu)$ is optimally represented by $\mathcal{SH}_{\alpha}(f,g,\delta;\Omega)$.
\end{remark}

We proceed to the main statement of this section.
\begin{theorem}\label{thm:mainopti}Suppose that $(0,1)^2 \subset \Omega \subset \R^2$ is bounded and $\rho:\R\to \R$ is either strongly sigmoidal of order $k\geq 2$ (see Definition~\ref{def:sigmoidal}) or an admissible smooth activation function (see Definition~\ref{def:admissible}). Then, for every $\alpha\in [1/2,1]$, the function class  $\mathcal{E}^{1/\alpha}(\Omega; \nu)$ is optimally representable by a neural network with activation function $\rho$.
\end{theorem}
\begin{proof}
Let $\alpha \in [1/2,1]$ and $\nu >0$. We first consider the case of $\rho$ strongly sigmoidal of order $k\geq 2$. 
Since the two-dimensional cardinal $B$-spline of order $34$, denoted by $N^{2}_{34}$, is $32$ times continuously differentiable and $\widehat{N^{2}_{34}}(0) \neq 0$ by construction, 
we conclude that there exists $c>0$ such that $f := N^{2}_{34}(c \cdot)$ satisfies $f\in C^{32}(\R^2)$ and $\hat{f} \neq 0$ for all $\xi\in [-3,3]^2$.
Application of Lemma \ref{lem:generatevanishingmoments} then yields the existence of $(c_i)_{i = 1}^7\subset \R$, $(d_i)_{i = 1}^7 \subset \R^2$ such that $g := \sum_{i=1}^7 c_i f(\cdot - d_i)$ is compactly supported, has $7$ vanishing moments in $x_1$-direction, and $\hat{g}(\xi) \neq 0$ for all $\xi \in [-3,3]^2$ such that $\xi_1 \neq 0$. 
Then, by Theorem~\ref{thm:FelixThm} and Remark \ref{rem:BdDomain} there exists $\delta>0$ such that $\mathcal{SH}_{\alpha}(f,g,\delta;\Omega)$ is optimal 
for $\mathcal{E}^{1/\alpha}(\Omega; \nu)$. 
We define 
$$
\{A_{j}: j\in \N\} := \left\{S_{k} D_{\alpha, 2^{\ell}}J^{\tau}: \ell \in \mathbb{N}_0, |k| \leq \lceil 2^{\ell(1-\alpha)} \rceil,  \tau \in \{0,1\} \right\},
$$
where we order $(A_{j})_{j\in \N}$ such that $|\! \det(A_j)| \leq | \! \det(A_{j+1})|$, for all $j\in \N$. This construction implies that the $\alpha$-shearlet system $\mathcal{SH}_{\alpha}(f,g,\delta;\Omega)$ is an affine system with generator function $f$. Thanks to Theorem~\ref{thm:EffApproxWithSplines}, there exist $L, R\in \N$ such that for all $D,\varepsilon>0$, there is a network $\Phi_{D,\varepsilon} \in \cNN_{L,R,d,\rho}$ with
\begin{align*}
\|f - \Phi_{D,\varepsilon}\|_{L^2([-D,D]^d)} \leq \varepsilon.
\end{align*}
Moreover, the weights of $\Phi_{D,\varepsilon}$ are polynomially bounded in $D, \varepsilon^{-1}$.
It is not difficult to verify that \eqref{eq:detgrowth} holds and hence Theorem~\ref{thm:affdicopt} yields that $\mathcal{SH}_{\alpha}(f,g,\delta;\Omega)$ is effectively representable by neural networks with activation function $\rho$. 
Finally, since $\mathcal{E}^{1/\alpha}(\Omega; \nu)$ is optimally representable by $\mathcal{SH}_{\alpha}(f,g,\delta;\Omega)$, we conclude with Theorem~\ref{theo:EncodeOfNeuralNetworks} that $\mathcal{E}^{1/\alpha}(\Omega; \nu)$ is optimally representable by neural networks with activation function $\rho$.

It remains to establish the statement for admissible smooth $\rho$.  In this case, by Theorem~\ref{thm:Approxsmoothrec} there exist $M\in \N$ and a 
neural network in $\cNN_{3, M, d, \rho}$ which realizes a compactly supported $f \in C^\infty(\R)$ satisfying $\hat{f}(\xi) \neq 0$, for all $\xi \in [-3,3]^2$.
Lemma \ref{lem:generatevanishingmoments} applied to this $f$ then yields a function $g$ that can be realized by a neural network in $\cNN_{3, M', d, \rho}$, for some $M'\in \N$, has $7$ vanishing moments in $x_1$-direction, is compactly supported, and satisfies $g \in C^\infty(\R)$, and $\hat{g}(\xi) \neq 0$, for all $\xi \in [-3,3]^2$ such that $\xi_1\neq 0 $.
By Theorem~\ref{thm:FelixThm} and Remark \ref{rem:BdDomain}, there exists $\delta>0$ such that $\mathcal{E}^{1/\alpha}(\Omega; \nu)$ is optimally representable by $\mathcal{SH}_{\alpha}(f,g,\delta;\Omega)$. 
Note that $\mathcal{SH}_{\alpha}(f,g,\delta;\Omega)$ is an affine system with generator function $f$. 
Since $f$ can be implemented with zero error by a neural network, Theorem~\ref{thm:affdicopt} yields that $\mathcal{SH}_{\alpha}(f,g,\delta;\Omega)$ is effectively representable by neural networks with admissible smooth activation function $\rho$. Optimality of $\mathcal{SH}_{\alpha}(f,g,\delta;\Omega)$ for $\mathcal{E}^{1/\alpha}(\Omega; \nu)$ implies, with Theorem~\ref{theo:EncodeOfNeuralNetworks}, that $\mathcal{E}^{1/\alpha}(\Omega; \nu)$ is optimally representable by neural networks with admissible smooth activation function  
$\rho$.
\end{proof}

\begin{remark}
Theorem~\ref{thm:FelixThm} requires the generators of the shearlet system guaranteeing optimal representability of 
$\mathcal{E}^{1/\alpha}(\Omega; \nu)$, for $1/2\leq\alpha\leq 1, \nu>0$, $\Omega \subset \R^2$ to be very smooth. 
On the other hand, Theorem~\ref{thm:mainopti} demonstrates that 
optimally-approximating neural networks are not required to be particularly smooth. Indeed, Theorem~\ref{thm:mainopti} holds for networks with differentiable but not necessarily twice differentiable activation functions. As the proof of Theorem~\ref{thm:mainopti} reveals, such weak assumptions suffice thanks to Theorem~\ref{thm:EffApproxWithSplines}, which 
demonstrates that it is possible to approximate arbitrarily smooth $B$-splines (in the $L^2$-norm) to within error $\varepsilon$ by neural networks with a number of weights that
does not depend on $\varepsilon$ as long as the activation function is strongly sigmoidal.
\end{remark}

\begin{remark}
We observe from the proof of Theorem \ref{thm:mainopti} that the depth of the networks required to achieve optimal approximation depends on the activation function only. 
Indeed, for an admissible smooth activation function, inspection of Theorem \ref{thm:Approxsmoothrec} reveals that networks with three layers can produce optimal approximations in Theorem \ref{thm:mainopti}. On the other hand, if a sigmoidal activation function is employed, Theorem \ref{thm:EffApproxWithSplines} shows that the construction in Theorem \ref{thm:mainopti} requires a certain minimum depth depending on the order of sigmoidality.
\end{remark}

\section{Generalization to Manifolds}\label{sec:manifold}

Frequently, a function $f$ to be approximated by a neural network models phenomena on (possibly low-dimensional) immersed submanifolds $\Gamma \subset \R^d$ of dimension $m<d$. We next briefly outline how our main results can be extended to this situation. Since analogous results, for the case of wavelets as representation systems,
appear already in \cite{ShaCC2015provableAppDNN}, we will allow ourselves to be somewhat informal.

Suppose that $f:\Gamma \to \R$ is compactly supported. Let $(U_i)_{i\in \N}\subset \Gamma$ be an open cover of $\Gamma$ such that for each $i\in \N$ the manifold patch $U_i$ can be parametrized as the graph of a function over a subset of the Euclidean coordinates, i.e.,
there exist coordinates $x_{d_1},\dots , x_{d_m}$, open sets $V_i\subset \R^m$, and smooth mappings
$$
\gamma_\ell:\R^m\to \R ,\quad \ell\in \{1,\dots , d\}\setminus \{d_1,\dots, d_m\}
$$
such that
$$
    U_i = \left\{\Xi_i(x_{d_1},\dots , x_{d_m}) := (\gamma_1(x_{d_1},\dots , x_{d_m}),\dots , x_{d_1},\dots , \gamma_d(x_{d_1},\dots , x_{d_m}))  :\ (x_{d_1},\dots , x_{d_m})\in V_i \right\}.
$$
Take a smooth partition of unity $(h_i)_{i\in \N}$, where $h_i:\Gamma\to \R$ is smooth with $\supp(h_i)\subset \overline{U_i}$ and $\sum_{i\in \N}h_i = 1$.
Define the localization of $f$ to $U_i$ by $f_i:=f h_i$ such that
\begin{equation}
\label{eq:patchsum}
     f = \sum_{i\in \N} f_i.
\end{equation}
Every $f_i:U_i\to \R$ can be reparametrized to
$$
    \tilde{f_i}:\left\{\begin{array}{ccc}\R^m&\to & \R\\
    (x_{d_1},\dots , x_{d_m})  & \mapsto & f_i\circ\Xi_i(x_{d_1},\dots , x_{d_m}).
    \end{array}
    \right.
$$
Suppose that there exist $L, M\in \N$ and neural networks $\tilde \Phi_i\in \cNN_{L,M,m,\rho}$ such that
\begin{equation}  \label{eq:patchapprox}
    \|\tilde{f}_i-\tilde\Phi_i\|_{L^2(V_i)}\le \varepsilon.
\end{equation}
Then, we can construct a neural network $\Phi_i \in \cNN_{L,M + m d,d,\rho}$ according to
$$
    \Phi_i(x):=\widetilde{\Phi}_i (P_i x),
$$
where $P_i$ denotes the orthogonal projection of $x$ onto the coordinates $(x_{d_1},\dots , x_{d_m})$. Since $P_i$ is linear, $\Phi_i$ is a neural network. Moreover, since $P_i$ is the inverse of the diffeomorphism $\Xi_i$, we get
$$
    \|\Phi_i - f_i\|_{L^2(U_i)}\le C \varepsilon,
$$
with $C>0$ depending on the curvature of $\Gamma|_{U_i}$ only.
Now we may build a neural network $\Phi$ by setting $\Phi:=\sum_{i\in \N}\Phi_i$. Combining (\ref{eq:patchapprox})
with the observation that, owing to the compact support of $f$, only a finite number of summands appears in the definition of $f$, we have constructed a neural network $\Phi$ which approximates $f$ on $\Gamma$. In summary, we observe the following.

\medskip
    \emph{
 Whenever a function class $\cC$ is invariant with respect to diffeomorphisms (in our construction the functions $\Xi_i$) and multiplication by smooth functions (in our construction the functions $h_i$), then approximation results on $\R^m$ can be lifted to approximation results on $m$-dimensional submanifolds $\Gamma\subset \R^d$.}
\medskip

Such invariances are, in particular, satisfied for all function classes characterized by a particular smoothness behavior, for example, the class of cartoon-functions as studied in Section \ref{sec:alphacart}. 

%------------------------------------------------------------------------------------------------------------------------------
\section{Numerical Results}\label{sec:numerics}
%------------------------------------------------------------------------------------------------------------------------------

Our theoretical results show that neural networks realizing uniform approximation error $\varepsilon$ over a function class $\cC \subset L^2(\R^d)$, $d\in \N$, must obey a fundamental lower bound on the growth rate (as $\varepsilon \rightarrow 0$) of the number of edges of nonzero weight. One of the most widely used learning algorithms is stochastic gradient descent with the gradient computed via backpropagation \cite{Rumelhart1988Backpropagation}. The purpose of this section is to investigate how this algorithm fares relative to our lower bound.

Interestingly, our numerical experiments below indicate that for a fixed, sparsely connected, network topology inspired by the construction of bump functions according to
\eqref{eq:FirstLayer} and \eqref{eq:SecondLayer}, and with the ReLU as activation function, the stochastic gradient descent algorithm generates neural networks that achieve 
$M$-edge approximation rates quite close to the fundamental limit.

The network topology we prescribe is depicted in Figure \ref{fig:network}. The rationale for choosing this topology is as follows. As mentioned before, admissible smooth activation functions consist of smooth functions which equal a ReLU outside a compact interval. For this class of activation functions, the 
associated $\alpha$-shearlet generators were constructed from a function $g$ as specified in \eqref{eq:SecondLayer}. Choosing $p_1 = p_2 = 1$ and $p_3 = 2$ in \eqref{eq:FirstLayer} yields hat
functions $t$. This construction implies that six nodes are required in the first layer in each subnetwork. 
\begin{figure}[htb]
    \centering
    \includegraphics[width = 0.7\textwidth]{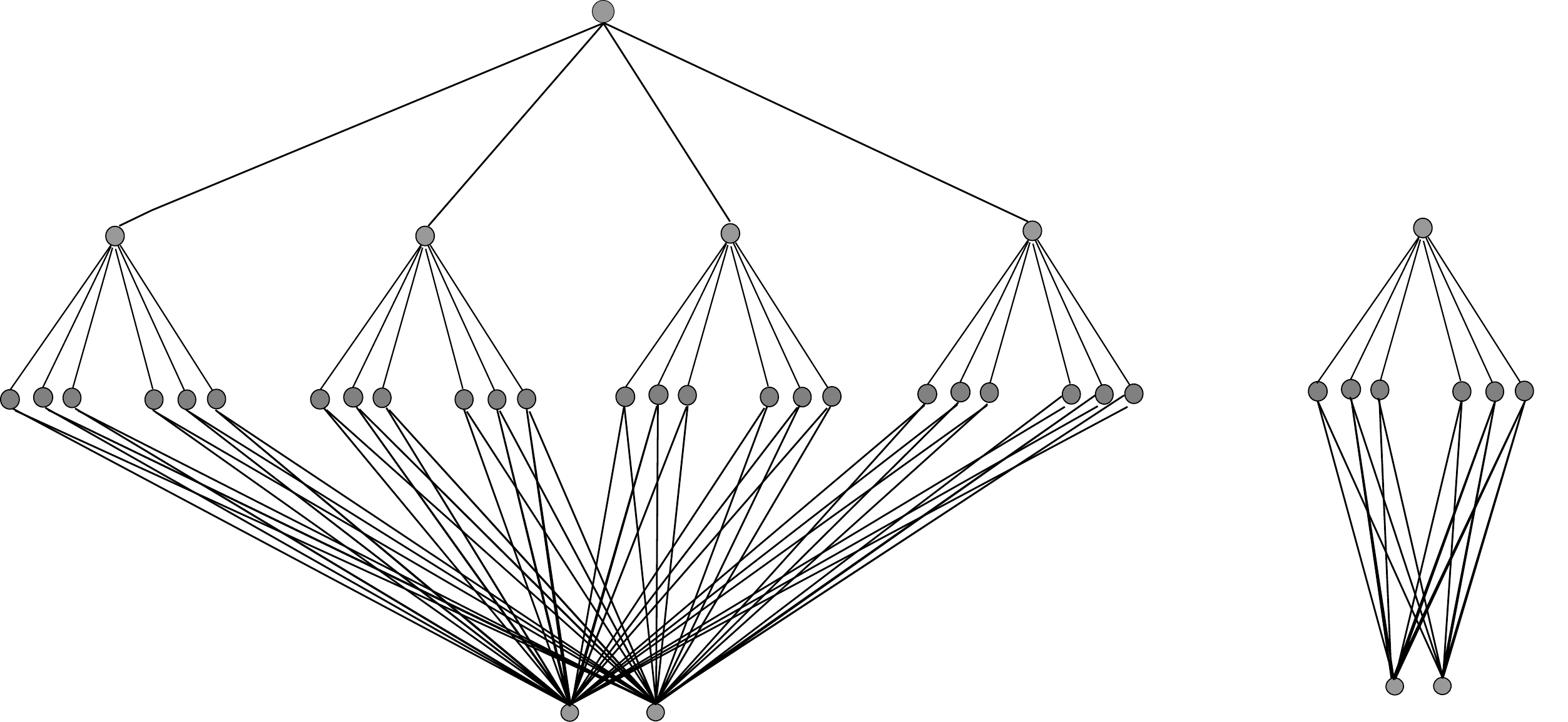}
    \caption{\textit{Left:} Topology of the neural network trained using stochastic gradient descent. The network consists of a weighted sum of four subnetworks. \textit{Right:} A single subnetwork.}
    \label{fig:network}
\end{figure}
In Figure \ref{fig:network}, we see four network realizations of $g$ in parallel. The output layer realizes a linear combination of the subnetworks.

We now train the network using the stochastic gradient descent algorithm. Following \eqref{eq:SecondLayer} the weights of the second layer remain fixed, and the weights in the first and the third layer only are trained. Training is performed for two different
functions, where one is a function with a line singularity (Figure \ref{fig:Ridgelets}(a)), and the other one is a cartoon-like function (Figure \ref{fig:Shearlets}(a)). Specifically, we train the network by drawing samples $(x_1,x_2)$ from an equispaced grid in $[-1,1]^2$. 
The resulting error is then backpropagated through the network. We repeat this procedure for different network sizes, i.e., for different numbers of subnetworks.
\begin{figure}[htb]
    \centering
    \includegraphics[width = 0.35\textwidth]{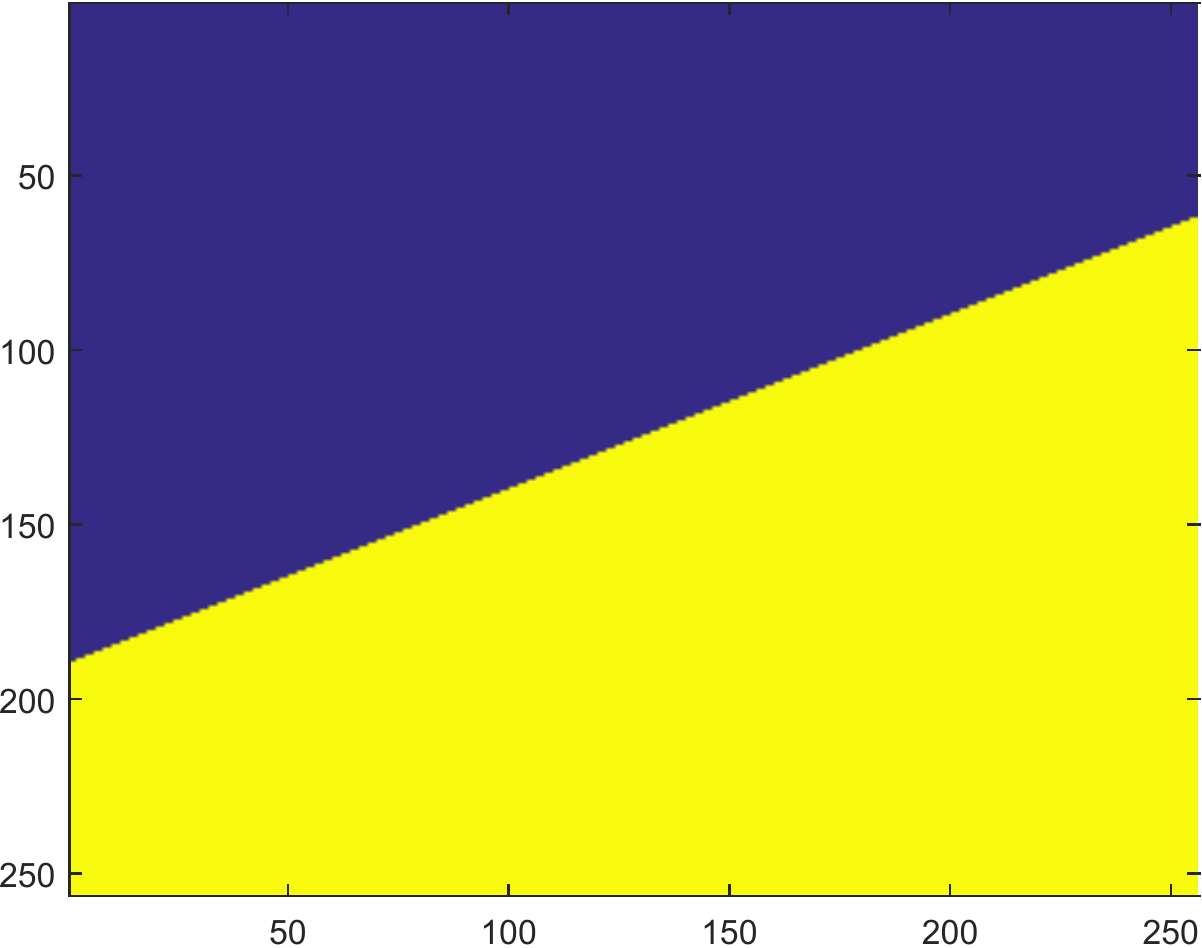}  \put(-80,-13){(a)}  \qquad  \includegraphics[width = 0.35\textwidth]{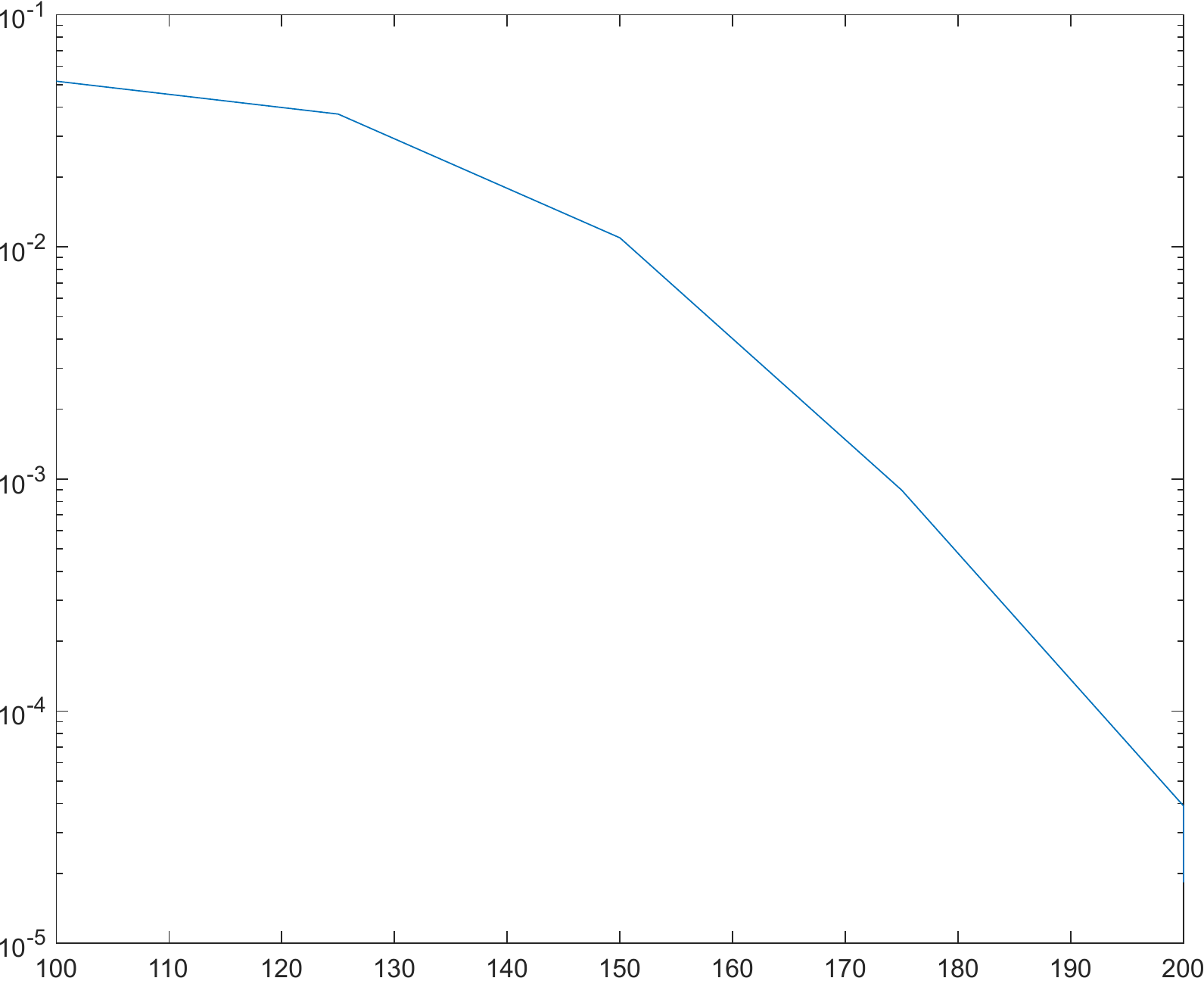} \put(-80,-13){(b)}\\[3ex]
        \includegraphics[width = 0.3\textwidth]{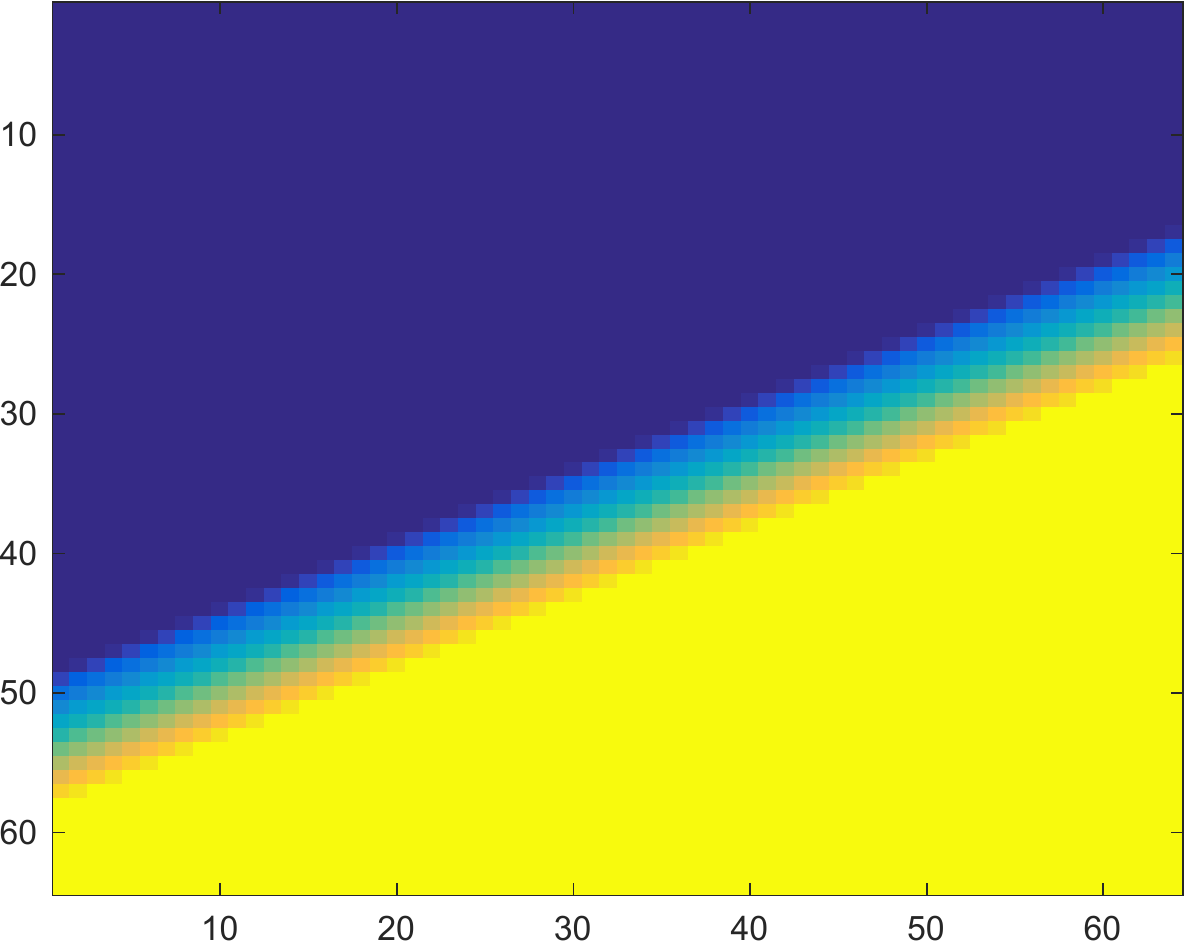} \put(-70,-13){(c)}  \quad  \includegraphics[width = 0.3\textwidth]{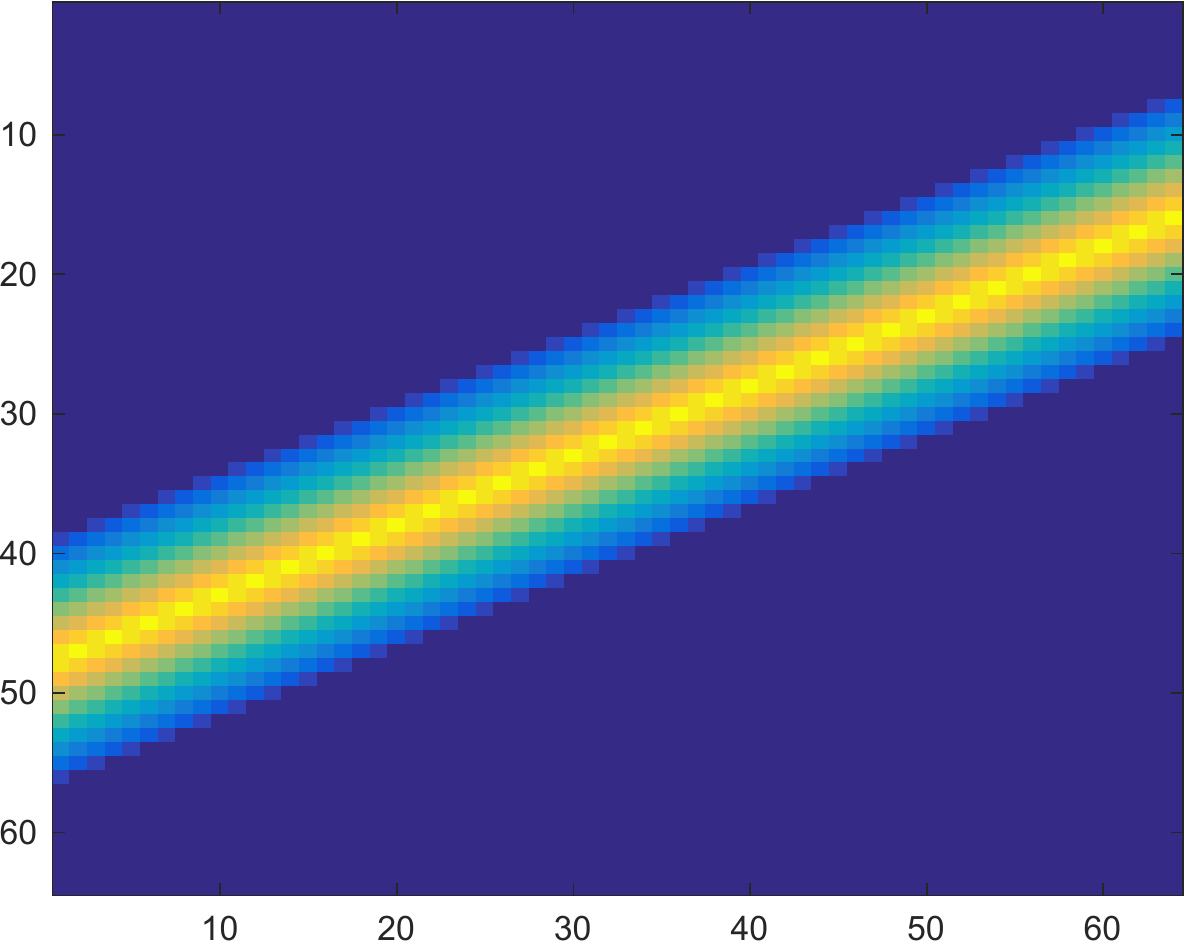} \put(-70,-13){(d)}  \quad    \includegraphics[width = 0.3\textwidth]{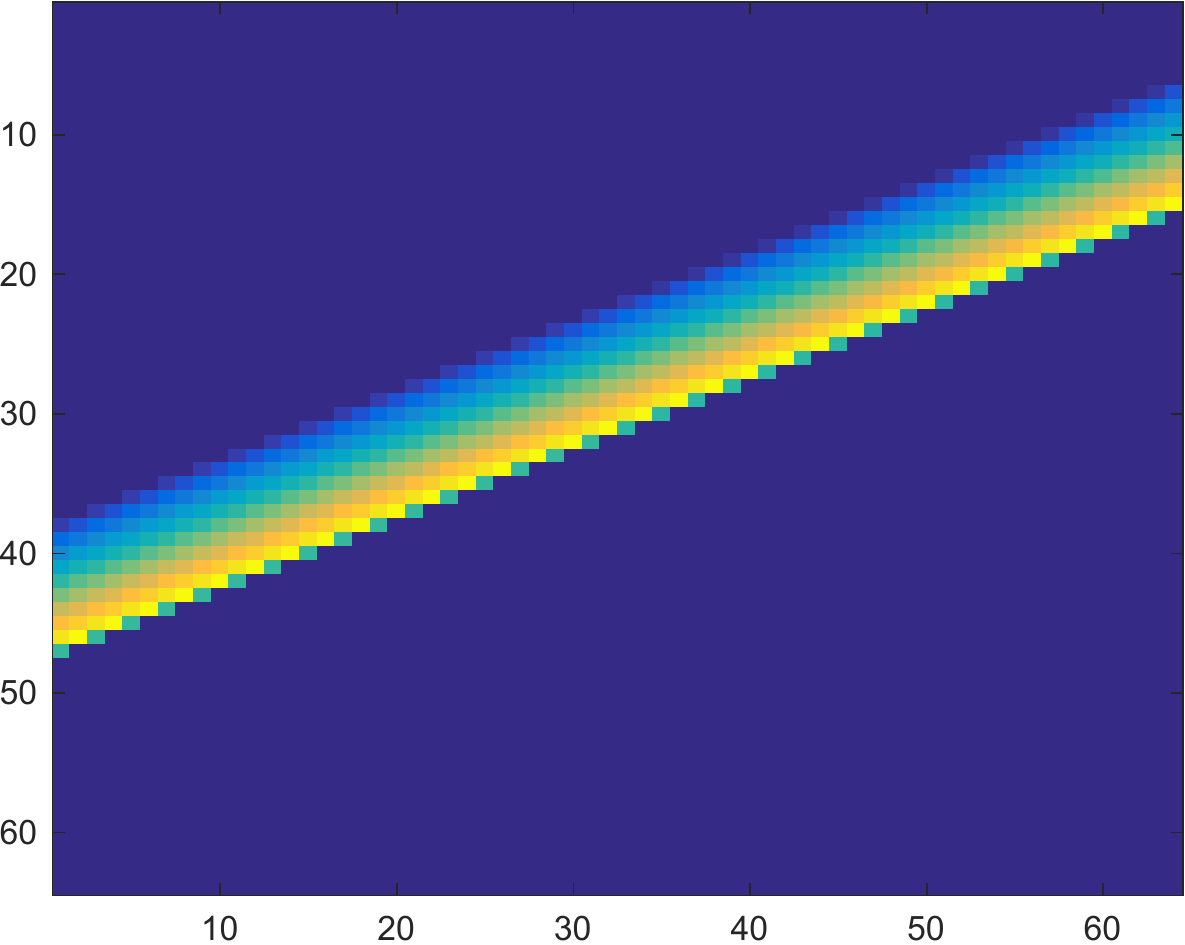} \put(-70,-13){(e)}
    \caption{(a): Function with a line singularity. (b): Approximation error (vertical axis) as a function of the number of edges (horizontal axis). (c)-(e): The functions obtained by restricting to the subnetworks with the largest weights in modulus in the final layer.}
    \label{fig:Ridgelets}
\end{figure}
We start by discussing the results for the function with a line singularity depicted in Figure \ref{fig:Ridgelets}(a). The approximation error corresponding to the trained neural network
is shown in Figure \ref{fig:Ridgelets}(b). The faster than linear decay of the approximation error in the semi-logarithmic scale indicates faster than exponential decay with respect to the number of edges. 
This is consistent with the best $M$-term approximation rate that ridgelets yield for piecewise constant functions with line singularities, see \cite{candes2001ridgelets}.

\begin{figure}[b!]
    \centering
    \includegraphics[width = 0.3\textwidth]{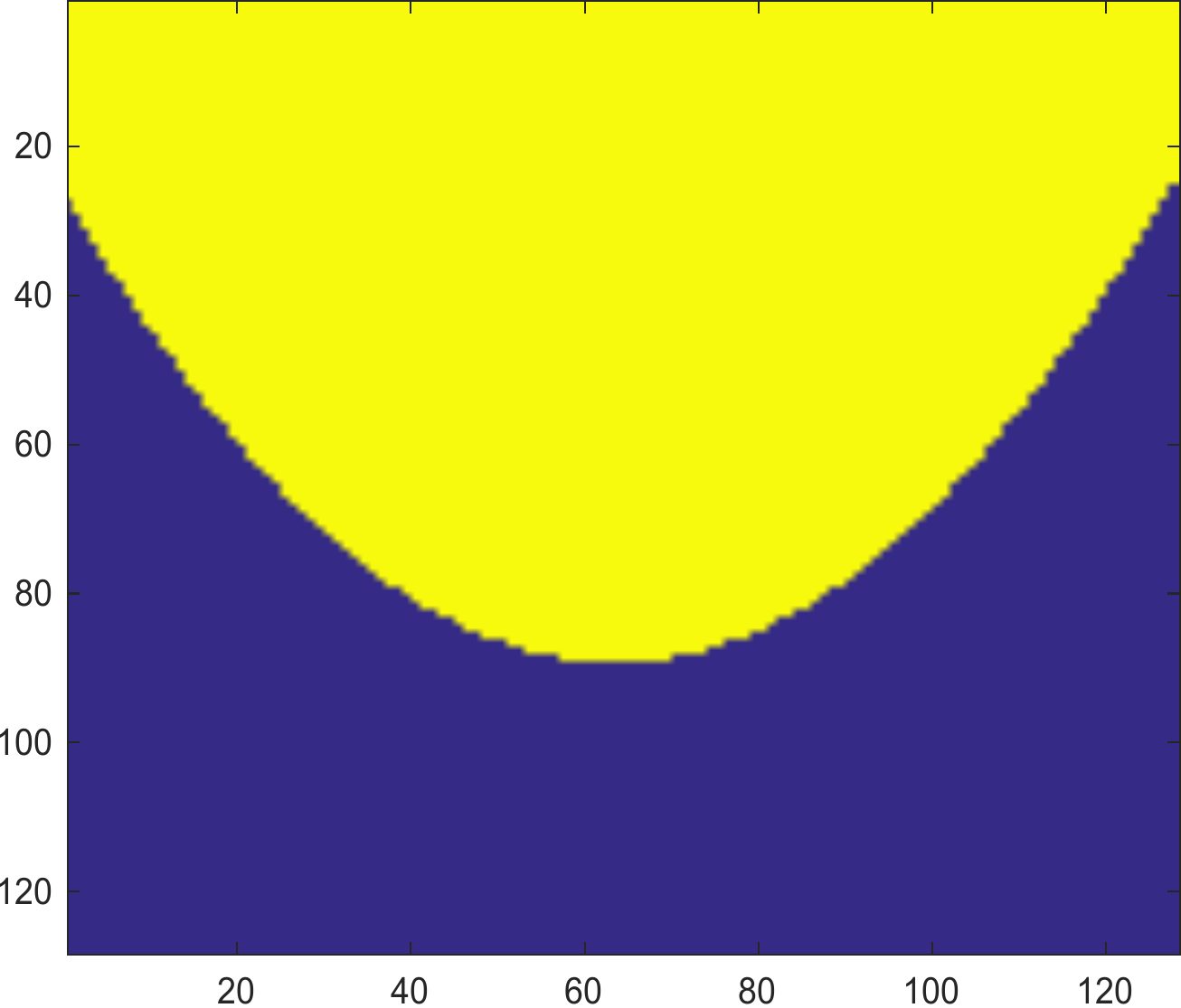}  \put(-70,-13){(a)}  \qquad  \includegraphics[width = 0.3\textwidth]{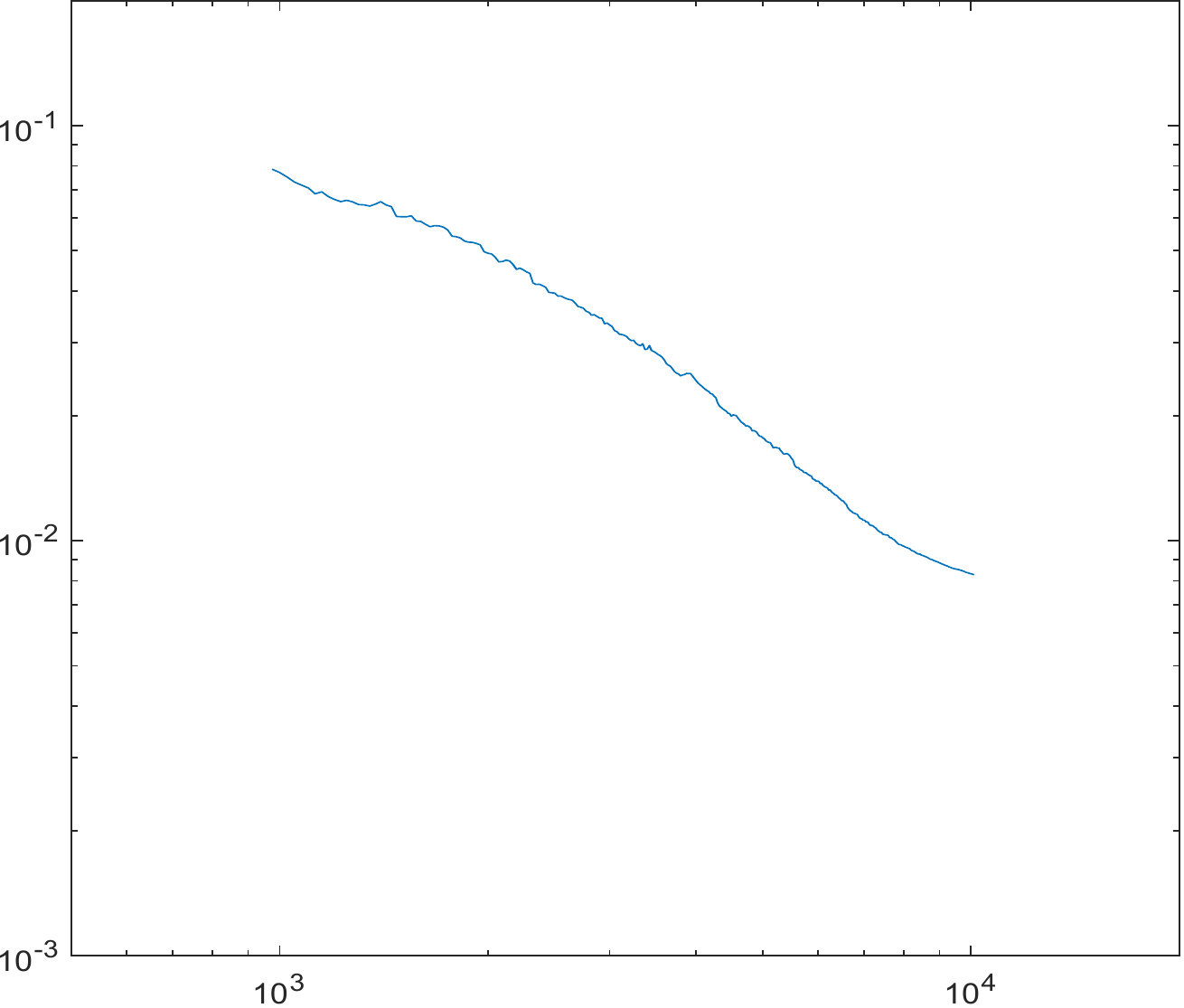} \put(-70,-13){(b)}\\[3ex]
        \includegraphics[width = 0.22\textwidth]{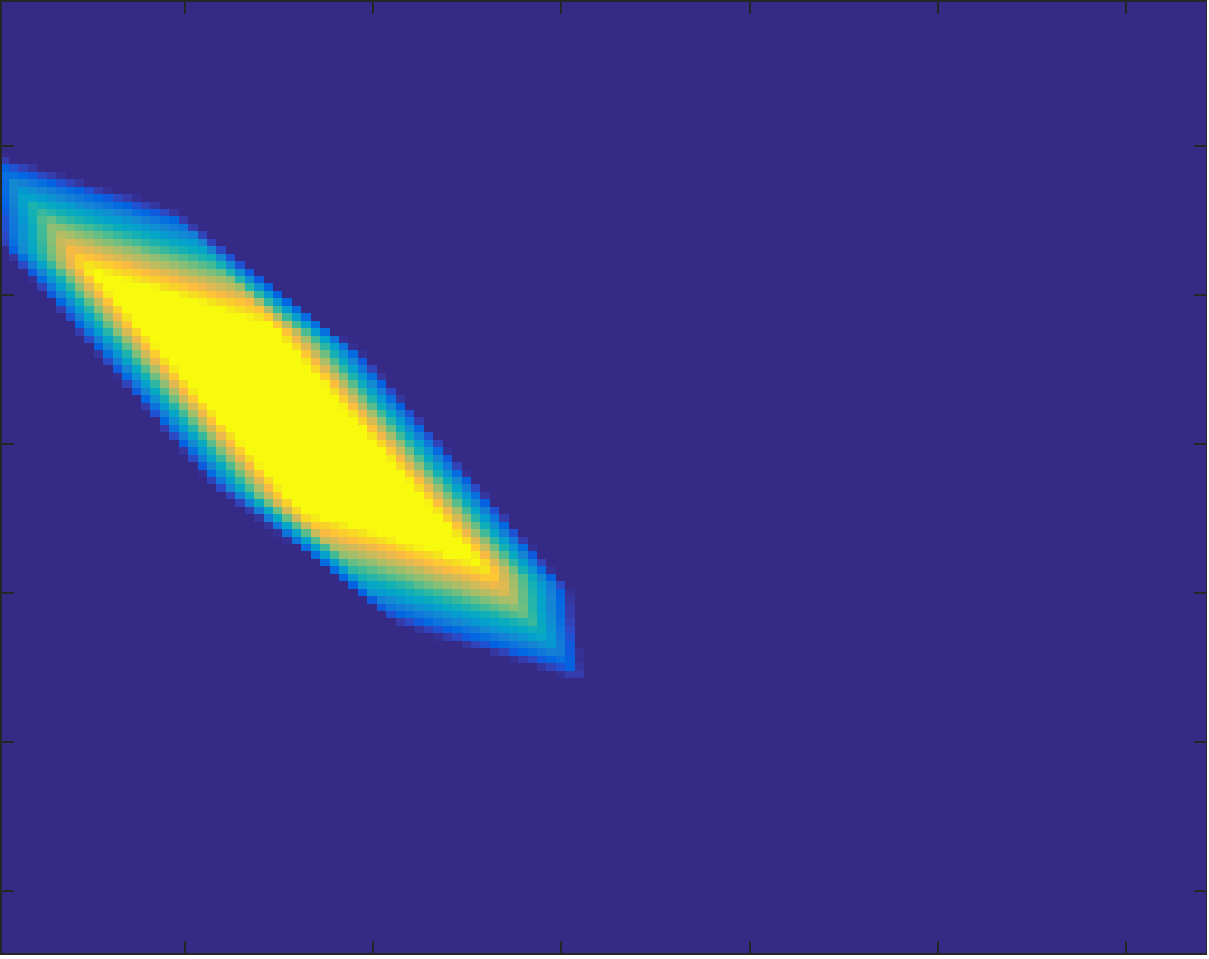} \put(-55,-13){(c)}  \  \includegraphics[width = 0.22\textwidth]{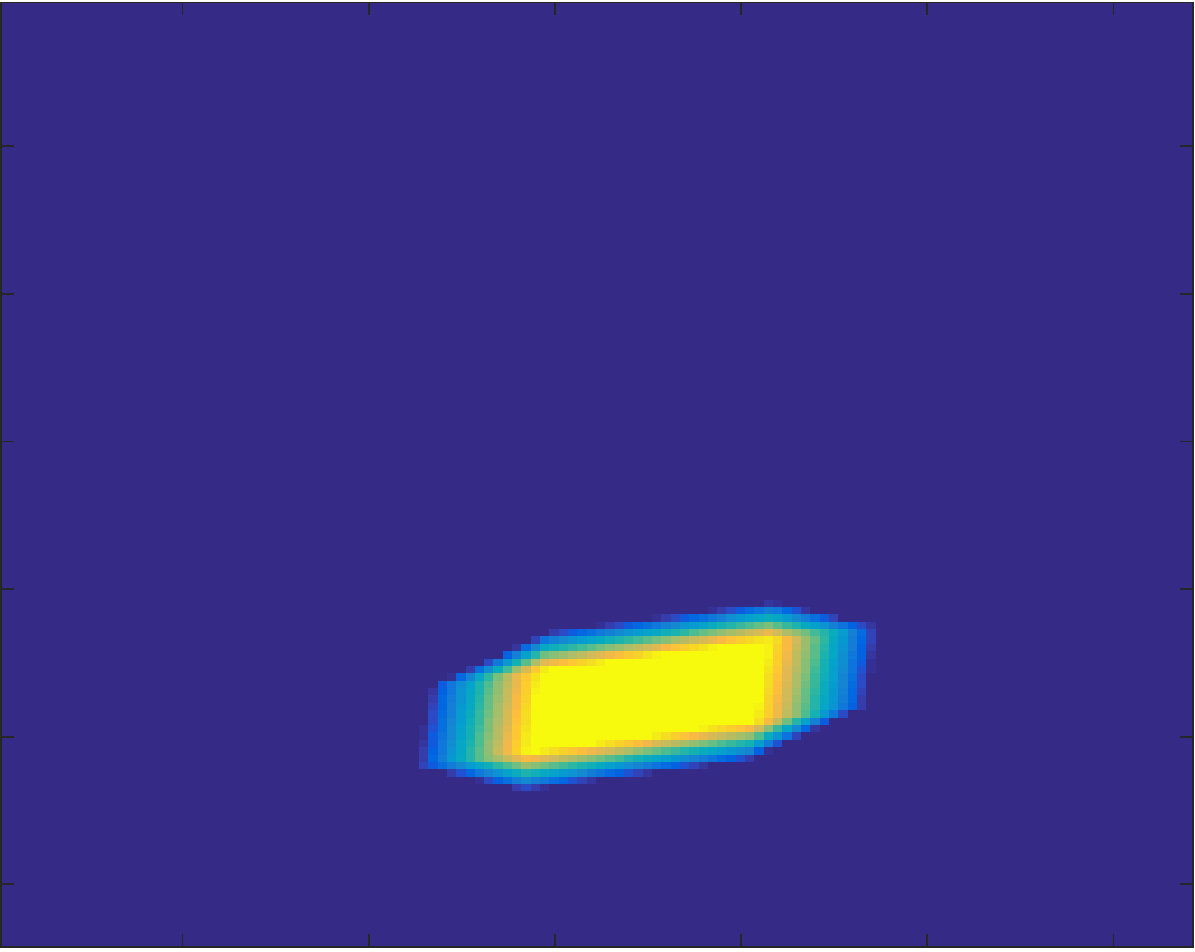} \put(-55,-13){(d)}  \    \includegraphics[width = 0.22\textwidth]{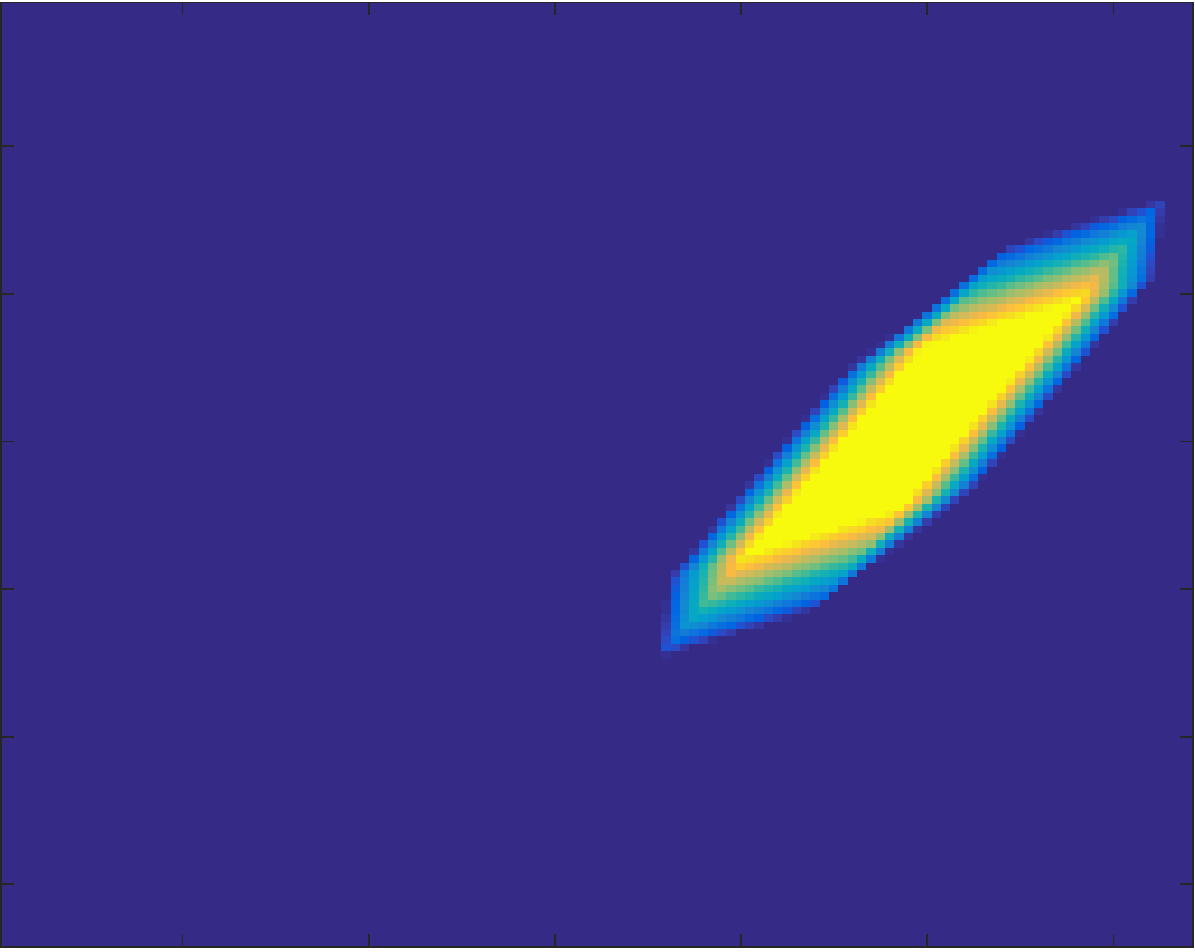} \put(-55,-13){(e)} \ \includegraphics[width = 0.22\textwidth]{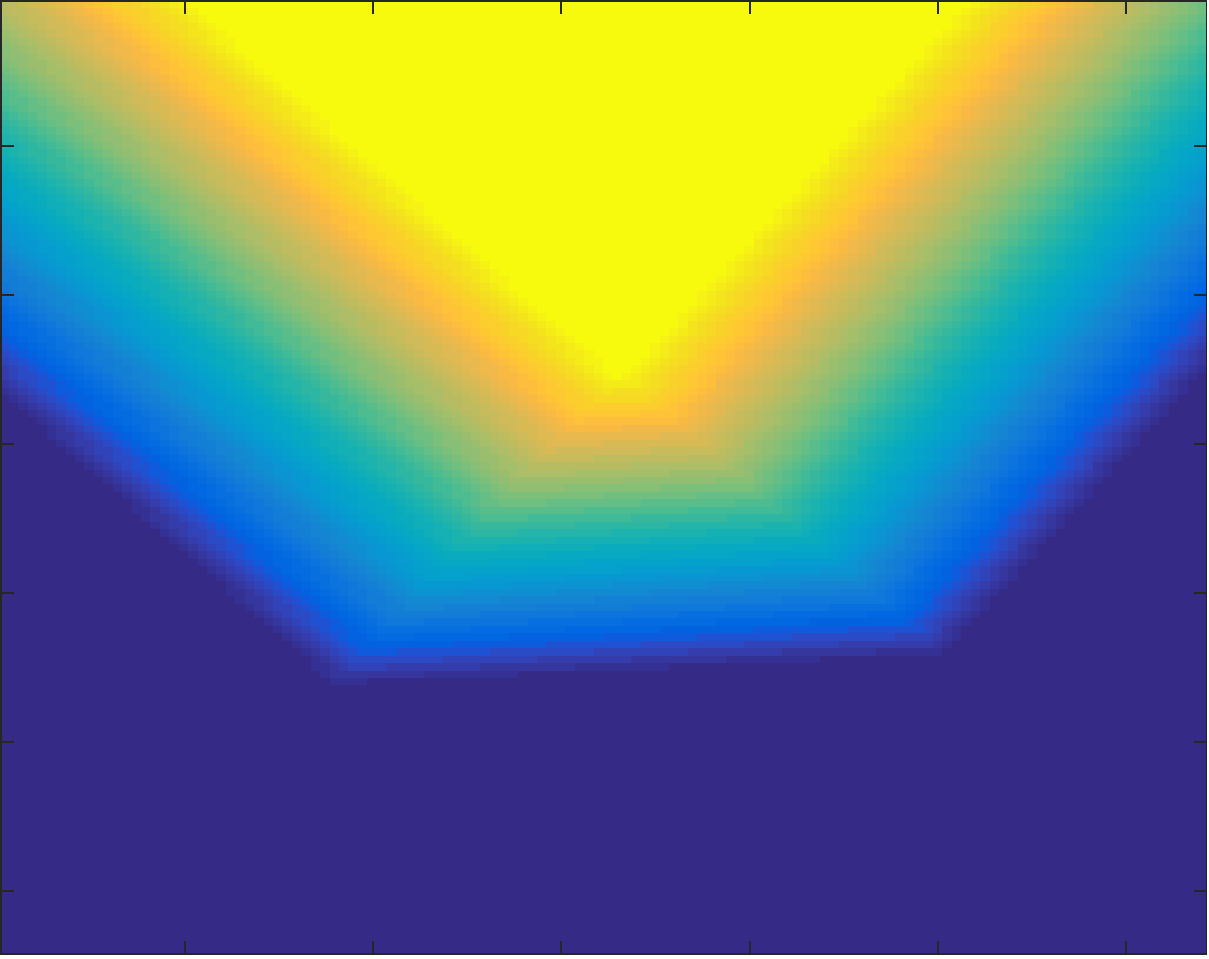} \put(-55,-13){(f)}\\[3ex]
        \includegraphics[width = 0.3\textwidth]{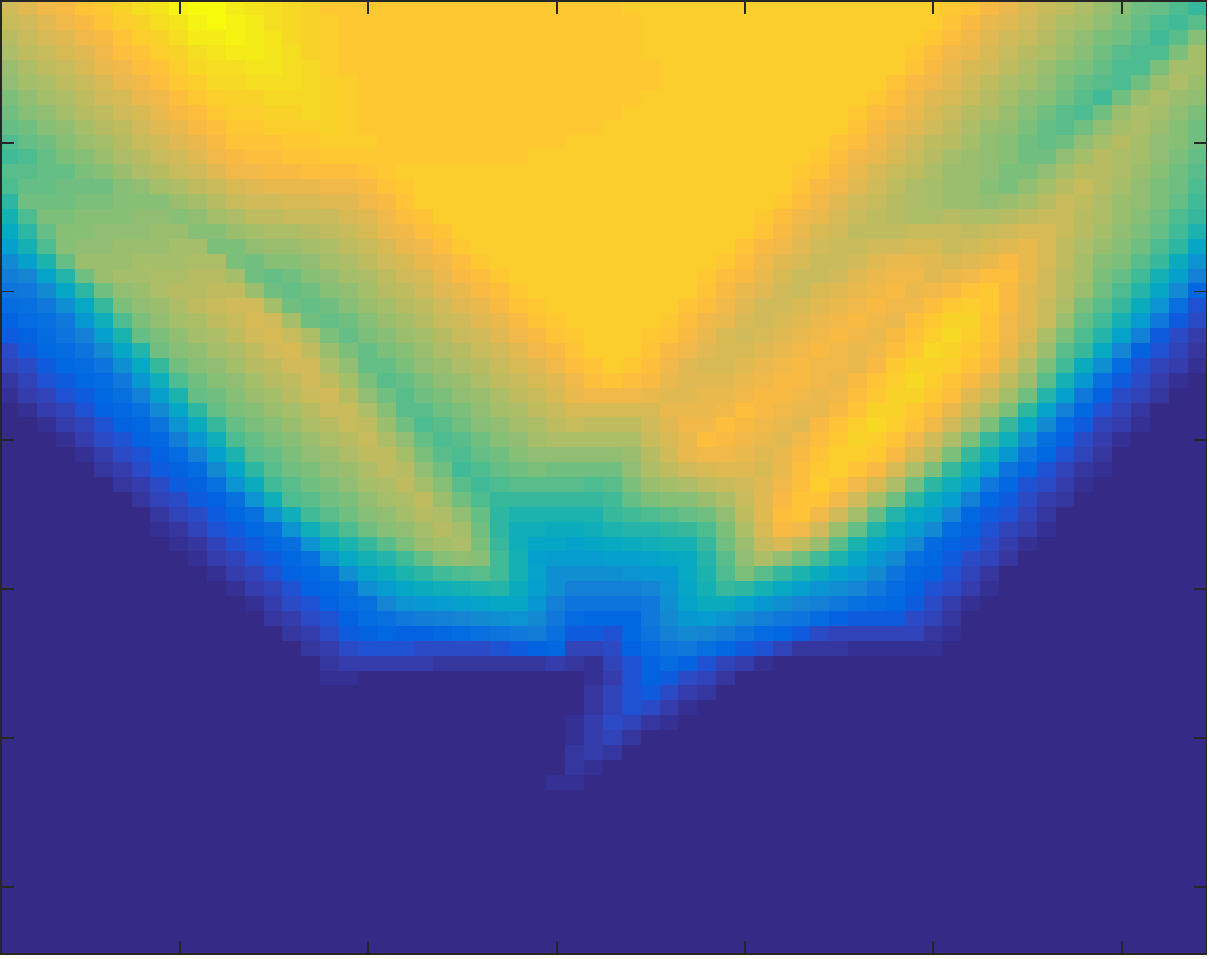} \put(-70,-13){(g)}  \  \includegraphics[width = 0.3\textwidth]{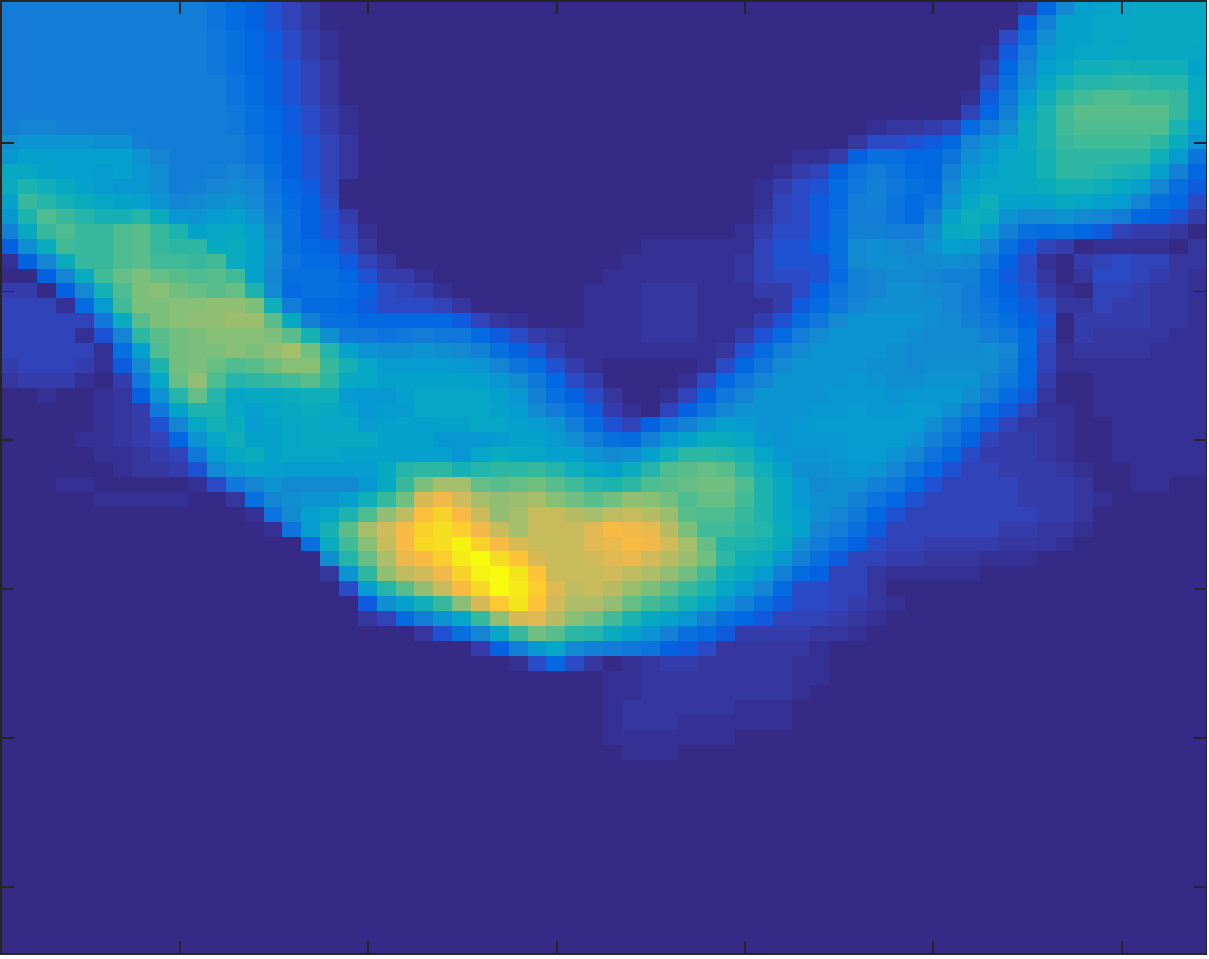} \put(-70,-13){(h)}   \   \includegraphics[width = 0.3\textwidth]{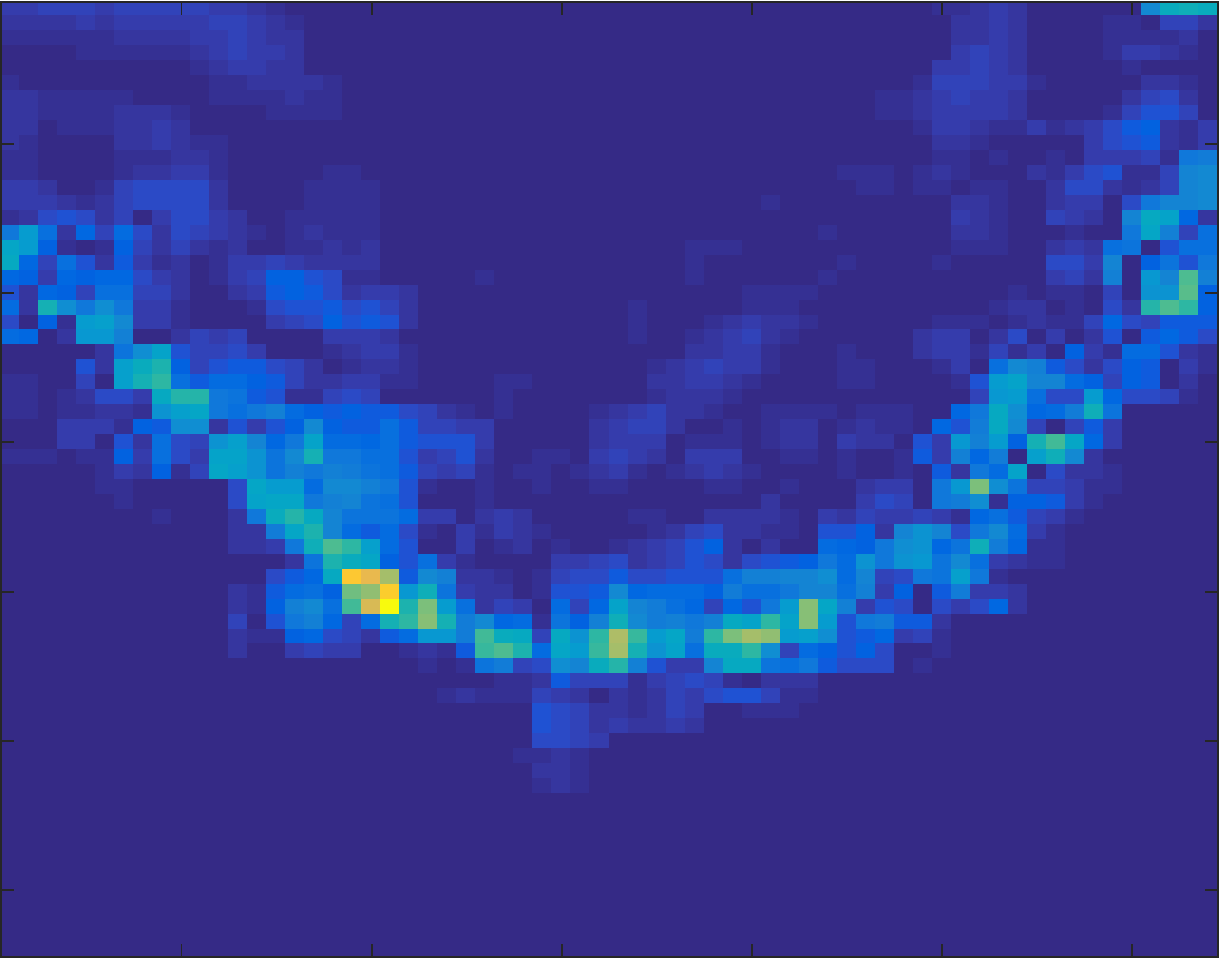} \put(-70,-13){(i)}
    \caption{(a): Function with curvilinear singularity to be approximated by the neural network. (b): Approximation error (vertical axis) as a function of the number of edges (horizontal axis). (c)-(f): Shearlet-like subnetworks. (g): Reconstruction using only the 10 subnetworks whose associated functions have the largest supports. (h): Reconstruction using only subnetworks whose associated functions have medium-sized support. (i): Reconstruction using only subnetworks with associated functions of very small support.}
    \label{fig:Shearlets}
\end{figure}
It is interesting to observe that the trained subnetworks yield $\alpha$-molecules
for $\alpha = 0$ (see Figures \ref{fig:Ridgelets}(c)-(e)). These functions are constant along one direction and vary along another,
hence can be considered part of a ridgelet system, which is, in fact, an optimally sparsifying representation system for line singularities.
Moreover, the orientation of the three learned ridge functions matches that of the original function.

In the second experiment, we draw samples from the function depicted in Figure \ref{fig:Shearlets}(a) below, which exhibits a curvilinear singularity. Figures \ref{fig:Shearlets}(c)-(e) show that the corresponding trained subnetworks resemble 
anisotropic molecules with different scales and of different orientations. We report, without showing the results, that the decay rate of the corresponding approximation error obtained when simply training with different network sizes did not come close to the rate of $M^{-1}$ predicted by our theory.
However, with a slight adaptation one obtains the result of Figure \ref{fig:Shearlets}(b), which demonstrates a decay of 
roughly $M^{-1}$. The specifics of this adaptation are as follows: We first train a large network with $\sim 10000$ edges, again by stochastic gradient descent. Then, the weights in the last layer are optimized using the Lasso \cite{Tib1996Lasso} to obtain a sparse weight vector $c^*$. 
We then pick the $M$ largest coefficients of $c^*$ and compute
the corresponding weighted sum of the associated subnetworks. The resulting approximation error is shown in Figure \ref{fig:Shearlets}(b).
Finally, we investigate whether the approximation characteristics delivered by this procedure are similar to what would be obtained by best $M$-term approximation with standard shearlet systems. 
Recall that shearlet elements at high scales tend to cluster around singularities \cite{GKL06, Kutyniok20111564}. Figures \ref{fig:Shearlets}(g)-(i) depict the corresponding results. Specifically, Figure \ref{fig:Shearlets}(g) shows the weighted sum of those subnetworks that have the largest support. In Figure \ref{fig:Shearlets}(h), we show weighted sums of subnetworks with medium-sized support, and in Figure \ref{fig:Shearlets}(i) we sum up only the subnetworks with the smallest supports. We observe that, indeed, subnetworks of large support approximate the smooth part of the underlying function, whereas the subnetworks associated to small supports resolve the jump singularity.

%------------------------------------------------------------------------------------------------------------------------------
\section*{Acknowledgments}
%------------------------------------------------------------------------------------------------------------------------------

The authors would like to thank J. Bruna, E. Cand\`{e}s, M. Genzel, S. G\"unt\"urk, Y. LeCun, K.-R. M\"uller, H. Rauhut, and F. Voigtl\"ander for interesting discussions, and D. Perekrestenko for very detailed and insightful comments on the manuscript.
G. K. and P. P. are grateful to the Faculty of Mathematics at the University of Vienna for the hospitality and support during their visits. Moreover, G. K. thanks the Department of Mathematics at Stanford University whose support allowed for completion of a portion of this work. G. K. acknowledges  partial  support  by  the  Einstein  Foundation  Berlin,  the  Einstein Center for Mathematics
Berlin (ECMath), the European Commission-Project DEDALE (contract no. 665044) within the H2020 Framework Program, DFG Grant KU 1446/18,
DFG-SPP 1798 Grants KU 1446/21 and KU 1446/23,
and by the DFG Research Center {\sc Matheon} ``Mathematics for Key Technologies''. G. K. and P. P acknowledge support by the DFG Collaborative Research Center TRR 109 ``Discretization in Geometry and Dynamics".

\bibliographystyle{abbrv}
\bibliography{references}

\end{document}